\documentclass[11pt,twoside]{article}
\def\jmlr{1}

\usepackage{hyperref}
\usepackage{amsmath}
\usepackage{amssymb}
\usepackage{amsbsy}
\usepackage{amsfonts}
\usepackage{xifthen}

\ifdefined\jmlr
\usepackage{jmlr2e}
\newcommand{\qed}{\hspace{\stretch{1}}
  $\blacksquare$
}

\ShortHeadings{Communication-Efficient Algorithms
  for Statistical Optimization}{Zhang, Duchi,
  and Wainwright}
\firstpageno{1}
\jmlrheading{14}{2013}{}{3/13}{10/13}{Yuchen Zhang, John Duchi, and
Martin Wainwright}

\else  
\usepackage{fullpage}
\usepackage{amsthm}
\renewenvironment{proof}{\noindent{\bf Proof}\hspace*{1em}}{\qed\\}
\usepackage[numbers]{natbib}

\fi  

\usepackage{color}
\usepackage{psfrag}
\usepackage{bootstrap}

\ifdefined\jmlr
\newenvironment{exa}[1]{
  \textbf{Example}
  \ifthenelse{\isempty{#1}}{%
    \textbf{.}
  }{%
    ({#1})\textbf{.}
  }
}{}
\else
\theoremstyle{definition}

\fi

\newcommand{\real}{\mathbb{R}}

\newcommand{\MYMIN}[1]{\ensuremath{\Big\{#1 \Big \}}}

\ifdefined\jmlr
\else
\makeatletter
\long\def\@makecaption#1#2{
  \vskip 0.8ex
  \setbox\@tempboxa\hbox{\small {\bf #1:} #2}
  \parindent 1.5em  
  \dimen0=\hsize
  \advance\dimen0 by -3em
  \ifdim \wd\@tempboxa >\dimen0
  \hbox to \hsize{
    \parindent 0em
    \hfil
    \parbox{\dimen0}{\def\baselinestretch{0.96}\small
      {\bf #1.} #2
    }
    \hfil}
  \else \hbox to \hsize{\hfil \box\@tempboxa \hfil}
  \fi
}
\makeatother
\fi

\ifdefined\jmlr
\title{Communication-Efficient Algorithms for \\ Statistical Optimization}
\author{\name Yuchen Zhang$^1$ \email yuczhang@eecs.berkeley.edu \\
  \name
  John C.\ Duchi$^1$ \email jduchi@eecs.berkeley.edu \\
  \name Martin J.\ Wainwright$^{1,2}$ \email wainwrig@berkeley.edu \\
  \addr $^1$Department of Electrical Engineering and
  Computer Science \\ $^2$Department of Statistics \\
  University of California, Berkeley \\
  Berkeley, CA 94720-1776 USA
}
\editor{Francis Bach}
\else 
\title{{{\LARGE{\bf{Communication-Efficient Algorithms for \\
          Statistical Optimization}}}}}

\author{Yuchen Zhang\footnotemark[1]
  \footnotemark[3] \and John
  C.\ Duchi\footnotemark[1]\ \footnotemark[3] \and Martin
  J.\ Wainwright\footnotemark[1]\ \footnotemark[2]\ \footnotemark[3]}

\fi  

\begin{document}

\maketitle

\ifdefined\jmlr \else \renewcommand{\thefootnote}{\fnsymbol{footnote}}
\footnotetext[1]{Department of Electrical Engineering and Computer
  Sciences, University of California, Berkeley}
\footnotetext[2]{Department of Statistics, University of California,
  Berkeley, Berkeley, CA 94720. Email:
  \texttt{\{yuczhang,jduchi,wainwrig\}@eecs.berkeley.edu}}
\footnotetext[3]{ YZ, JCD, and MJW were partially supported by MURI
  grant N00014-11-1-0688 from the Office of Naval Research.  JCD was
  additionally supported by an NDSEG fellowship and a Facebook
  Fellowship.}  \renewcommand{\thefootnote}{\arabic{footnote}} \fi

\renewcommand{\labelenumi}{(\roman{enumi})}

\begin{abstract}
We analyze two communication-efficient algorithms for distributed
optimization in statistical settings involving large-scale data
sets. The first algorithm is a standard averaging method that
distributes the $\totalnumobs$ data samples evenly to $\nummac$
machines, performs separate minimization on each subset, and then
averages the estimates.  We provide a sharp analysis of this average
mixture algorithm, showing that under a reasonable set of conditions,
the combined parameter achieves mean-squared error (MSE) that decays
as $\order(\totalnumobs^{-1}+(\totalnumobs/\nummac)^{-2})$.  Whenever
$\nummac \le \sqrt{\totalnumobs}$, this guarantee matches the best
possible rate achievable by a centralized algorithm having access to
all $\totalnumobs$ samples.  The second algorithm is a novel method,
based on an appropriate form of bootstrap subsampling. Requiring only
a single round of communication, it has mean-squared error that decays
as $\order(\totalnumobs^{-1} + (\totalnumobs/\nummac)^{-3})$, and so
is more robust to the amount of parallelization. In addition, we show
that a stochastic gradient-based method attains mean-squared error
decaying as $\order(\totalnumobs^{-1} + (\totalnumobs /
\nummac)^{-3/2})$, easing computation at the expense of a potentially
slower MSE rate.  We also provide an experimental evaluation of our
methods, investigating their performance both on simulated data and on
a large-scale regression problem from the internet search domain. In
particular, we show that our methods can be used to efficiently solve
an advertisement prediction problem from the Chinese SoSo Search
Engine, which involves logistic regression with $\totalnumobs \approx
2.4 \times 10^8$ samples and $d \approx 740,\!000$ covariates.
\end{abstract}

\ifdefined\jmlr
\else
\fi

\section{Introduction}

Many procedures for statistical estimation are based on a form of
(regularized) empirical risk minimization, meaning that a parameter of
interest is estimated by minimizing an objective function defined by
the average of a loss function over the data.  Given the current
explosion in the size and amount of data available in statistical
studies, a central challenge is to design efficient algorithms for
solving large-scale problem instances.  In a centralized setting,
there are many procedures for solving empirical risk minimization
problems, among them standard convex programming approaches
\citep[e.g.][]{BoydVa04} as well as stochastic approximation and
optimization algorithms~\citep{RobbinsMo51, HazanKaKaAg06,
  NemirovskiJuLaSh09}.  When the size of the dataset becomes extremely
large, however, it may be infeasible to store all of the data on a
single computer, or at least to keep the data in memory.  Accordingly,
the focus of this paper is the study of some distributed and
communication-efficient procedures for empirical risk minimization.

Recent years have witnessed a flurry of research on distributed
approaches to solving very large-scale statistical optimization
problems.  Although we cannot survey the literature adequately---the
papers~\citet{NedicOz09, RamNeVe10, JohanssonRaJo09,
  DuchiAgWa12, DekelGiShXi12, AgarwalDu11, RechtReWrNi11, DuchiBaWa12}
and references therein contain a sample of relevant work---we touch on
a few important themes here. It can be difficult within a purely
optimization-theoretic setting to show explicit benefits arising from
distributed computation. In statistical settings, however, distributed
computation can lead to gains in computational efficiency, as shown by a
number of authors~\citep{AgarwalDu11, DekelGiShXi12, RechtReWrNi11,
  DuchiBaWa12}.  Within the family of distributed algorithms, there
can be significant differences in communication complexity: different
computers must be synchronized, and when the dimensionality of the
data is high, communication can be prohibitively expensive. It is thus
interesting to study distributed estimation algorithms that require
fairly limited synchronization and communication while still enjoying
the greater statistical accuracy that is usually associated
with a larger dataset.

With this context, perhaps the simplest algorithm for distributed
statistical estimation is what we term the \emph{average mixture}
(\avgm) algorithm.  It is an appealingly simple method: given
$\nummac$ different machines and a dataset of size $\totalnumobs$,
first assign to each machine a (distinct) dataset of size $\numobs =
\totalnumobs / \nummac$, then have each machine $i$ compute the
empirical minimizer $\optvar_i$ on its fraction of the data, and
finally average all the parameter estimates $\optvar_i$ across the
machines.  This approach has been studied for some classification and
estimation problems by~\citet{MannMcMoSiWa09}
and~\citet{McDonaldHaMa2010}, as well as for certain stochastic
approximation methods by~\citet{ZinkevichSmWeLi10}. Given an empirical
risk minimization algorithm that works on one machine, the procedure
is straightforward to implement and is extremely communication
efficient, requiring only a single round of communication. It is also
relatively robust to possible failures in a subset of machines and/or
differences in speeds, since there is no repeated synchronization.
When the local estimators are all unbiased, it is clear that the the
\avgm\ procedure will yield an estimate that is essentially as good as
that of an estimator based on all $\totalnumobs$ samples.  However,
many estimators used in practice are biased, and so it is natural to
ask whether the method has any guarantees in a more general setting.
To the best of our knowledge, however, no work has shown rigorously
that the \avgm\ procedure generally has greater
efficiency than the naive approach of using $\numobs = \totalnumobs /
\nummac$ samples on a single machine.

This paper makes three main contributions.  First, in
Section~\ref{sec:avg_bound}, we provide a sharp analysis of the
\avgm\ algorithm, showing that under a reasonable set of conditions on
the population risk, it can indeed achieve substantially better rates
than the naive approach.  More concretely, we provide bounds on the
mean-squared error (MSE) that decay as \mbox{$\order((\numobs
  \nummac)^{-1} + \numobs^{-2})$.}  Whenever the number of machines
$\nummac$ is less than the number of samples $\numobs$ per machine,
this guarantee matches the best possible rate achievable by a
centralized algorithm having access to all $\totalnumobs = \numobs
\nummac$ samples.  In the special case of optimizing log likelihoods,
the pre-factor in our bound involves the trace of the Fisher
information, a quantity well-known to control the fundamental limits
of statistical estimation. We also show how the result extends to
stochastic programming approaches, exhibiting a stochastic
gradient-descent based procedure that also attains convergence rates
scaling as $\order((\numobs \nummac)^{-1})$, but with slightly worse
dependence on different problem-specific parameters. 

Our second contribution is to develop a novel extension of simple
averaging.  It is based on an appropriate form of
resampling~\citep{EfronTi93, Hall92, PolitisRoWo99}, which we refer to
as the \emph{subsampled average mixture} (\savgm) approach. At a high
level, the \savgm\ algorithm distributes samples evenly among
$\nummac$ processors or computers as before, but instead of simply
returning the empirical minimizer, each processor further subsamples
its own dataset in order to estimate the bias of its own estimate, and
returns a subsample-corrected estimate.  We establish that the
\savgm\ algorithm has mean-squared error decaying as
$\order(\nummac^{-1} \numobs^{-1} + \numobs^{-3})$.  As long as
$\nummac < \numobs^2$, the subsampled method again matches the
centralized gold standard in the first-order term, and has a
second-order term smaller than the standard averaging approach.

Our third contribution is to perform a detailed empirical evaluation
of both the \avgm\ and \savgm\ procedures, which we present in
Sections~\ref{sec:simulations} and~\ref{sec:real-experiment}. Using
simulated data from normal and non-normal regression models, we
explore the conditions under which the \savgm\ algorithm yields better
performance than the \avgm\ algorithm; in addition, we study the
performance of both methods relative to an oracle baseline that uses
all $\totalnumobs$ samples. We also study the sensitivity of the
algorithms to the number of splits $\nummac$ of the data, and in the
\savgm\ case, we investigate the sensitivity of the method to the
amount of resampling. These simulations show that both \avgm\ and
\savgm\ have favorable performance, even when compared to the
unattainable ``gold standard'' procedure that has access to all
$\totalnumobs$ samples. In Section~\ref{sec:real-experiment}, we
complement our simulation experiments with a large logistic regression
experiment that arises from the problem of predicting whether a user
of a search engine will click on an advertisement. This experiment is
large enough---involving $\totalnumobs \approx 2.4 \times 10^8$
samples in $d \approx 740,000$ dimensions with a storage size of
approximately 55 gigabytes---that it is difficult to solve efficiently
on one machine.  Consequently, a distributed approach is essential to
take full advantage of this data set.  Our experiments on this
problem show that \savgm---with the resampling and correction
it provides---gives substantial performance benefits over naive
solutions as well as the averaging algorithm \avgm.


\section{Background and problem set-up}

We begin by setting up our decision-theoretic framework for empirical
risk minimization, after which we describe our algorithms and the
assumptions we require for our main theoretical results.

\subsection{Empirical risk minimization}

Let $\{f(\cdot; \statsample), \; \statsample \in \statsamplespace\}$
be a collection of real-valued and convex loss functions, each defined
on a set containing the convex set $\optvarspace \subseteq \R^d$. Let
$\statprob$ be a probability distribution over the sample space
$\statsamplespace$.  Assuming that each function $x \mapsto f(\optvar;
\, x)$ is $\statprob$-integrable, the \emph{population risk} $\popfun
: \optvarspace \rightarrow \R$ is given by
\begin{equation}
  \label{eqn:objective}
  \popfun(\optvar) \defeq \E_\statprob[f(\optvar; \statrv)] =
  \int_{\statsamplespace} f(\optvar; \statsample) d
  \statprob(\statsample).
\end{equation}
Our goal is to estimate the parameter vector minimizing the population
risk, namely the quantity
\begin{equation}
  \label{eqn:optimal-point}
  \optvar^* \defeq
  \argmin_{\optvar \in \optvarspace} \popfun(\optvar)
  = \argmin_{\optvar \in \optvarspace}
  \int_{\statsamplespace} f(\optvar; \statsample) d\statprob(\statsample),
\end{equation}
which we assume to be unique. In practice, the population distribution
$\statprob$ is unknown to us, but we have access to a collection
$\samples$ of samples from the distribution $\statprob$.  Empirical
risk minimization is based on estimating $\optvar^*$ by solving the
optimization problem
\begin{equation}
  \label{eqn:erm}
  \what{\optvar} \in \argmin_{\optvar \in \optvarspace} \MYMIN{
    \frac{1}{|\samples|} \sum_{\statsample \in \samples} f(\optvar;
    \statsample)}.
\end{equation}

Throughout the paper, we impose some regularity conditions on the
parameter space, the risk function $\popfun$, and the instantaneous
loss functions $f(\cdot; \statsample) : \optvarspace \rightarrow \R$.
These conditions are standard in classical statistical analysis of
$M$-estimators~\citep[e.g.][]{LehmannCa98,Keener10}.  Our first
assumption deals with the relationship of the parameter space to the
optimal parameter $\optvar^*$.
\begin{assumption}[Parameters]
  \label{assumption:parameter-space}
  The parameter space $\optvarspace \subset \R^d$ is a compact convex
  set, with \mbox{$\optvar^* \in \interior \optvarspace$} and
  $\ell_2$-radius $\radius = \max \limits_{\optvar \in \optvarspace}
  \ltwo{\optvar - \optvar^*}$.
\end{assumption}
\noindent
In addition, the risk function is required to have some amount of
curvature.  We formalize this notion in terms of the Hessian of
$\popfun$:
\begin{assumption}[Local strong convexity]
  \label{assumption:strong-convexity}
  The population risk is twice differentiable, and there exists a
  parameter $\strongparam > 0$ such that \mbox{$\nabla^2
    \popfun(\optvar^*) \succeq \strongparam I_{d \times d}$.}
\end{assumption}
\noindent
Here $\nabla^2 \popfun(\optvar)$ denotes the $\usedim \times \usedim$
Hessian matrix of the population objective $\popfun$ evaluated at
$\optvar$, and we use $\succeq$ to denote the positive semidefinite
ordering (i.e., $A \succeq B$ means that $A - B$ is positive
semidefinite.) This local condition is milder than a global strong
convexity condition and is required to hold only for the population
risk $\popfun$ evaluated at $\optvar^*$.  It is worth observing that
some type of curvature of the risk is required for any method to
consistently estimate the parameters $\optvar^*$.

\subsection{Averaging methods}

Consider a data set consisting of $\totalnumobs = \nummac \numobs$
samples, drawn i.i.d.\ according to the distribution
$\statprob$.  In the distributed setting, we divide this
$\totalnumobs$-sample data set evenly and uniformly at random among a
total of $\nummac$ processors.  (For simplicity, we have assumed the
total number of samples is a multiple of $\nummac$.)  For $i = 1,
\ldots, \nummac$, we let $\samples_{1,i}$ denote the data set assigned
to processor $i$; by construction, it is a collection of $\numobs$
samples drawn i.i.d.\ according to $\statprob$, and the samples in
subsets $\samples_{1,i}$ and $\samples_{1,j}$ are independent for $i
\neq j$.  In addition, for each processor $i$ we define the (local)
empirical distribution $\statprob_{1,i}$ and empirical objective
$F_{1,i}$ via
\begin{equation}
\label{EqnDefnLocalQuants}
  \statprob_{1,i} \defeq \frac{1}{|\samples_1|} \sum_{\statsample \in
    \samples_{1,i}} \delta_{\statsample}, ~~~ \mbox{and} ~~~
  F_{1,i}(\optvar) \defeq \frac{1}{|\samples_{1,i}|} \sum_{\statsample
    \in \samples_{1,i}} f(\optvar; \statsample).
\end{equation}

\noindent With this notation, the \avgm\ algorithm is very simple to
describe:
\paragraph{Average mixture algorithm}
\begin{enumerate}
\item[(1)] For each $i \in \{1, \ldots, \nummac\}$, processor $i$ uses
  its local dataset $\samples_{1, i}$ to compute the local empirical
  minimizer
  \begin{align}
    \label{EqnDefnLocalEmpMin}
    \optvar_{1, i} & \in \argmin_{\optvar \in \optvarspace}
    \MYMIN{\underbrace{\frac{1}{|S_{1,i}|} \sum_{x \in S_{1,i}} f(\optvar;
      x)}_{F_{1,i}(\optvar)}}.
  \end{align}
\item[(2)] These $\nummac$ local estimates are then averaged
  together---that is, we compute
  \begin{equation}
    \label{eqn:avgm-alg-def}
   \optavg_1 = \frac{1}{\nummac} \sum_{i = 1}^\nummac \optvar_{1, i}.
  \end{equation}
\end{enumerate}

The subsampled average mixture (\savgm) algorithm is based on an
additional level of sampling on top of the first, involving a fixed
subsampling rate $\ratio \in [0, 1]$.  It consists of
the following additional steps:
\paragraph{Subsampled average mixture algorithm}
\begin{enumerate}
\item[(1)] Each processor $i$ draws a subset $\samples_{2,i}$ of size
  $\ceil{r \numobs}$ by sampling uniformly at random without
  replacement from its local data set $\samples_{1,i}$.

\item[(2)] Each processor $i$ computes both the local empirical
  minimizers $\optvar_{1,i}$ from equation~\eqref{EqnDefnLocalEmpMin}
  and the empirical minimizer
  \begin{align}
    \optvar_{2,i} & \in \argmin_{\optvar \in \Theta }
    \MYMIN{\underbrace{\frac{1}{|\samples_{2,i}|} \sum_{\statsample \in
          \samples_{2,i}} f(\optvar; \statsample)}_{F_{2,i}(\optvar)}}.
  \end{align}

\item[(3)] In addition to the previous
  average~\eqref{eqn:avgm-alg-def}, the \savgm\ algorithm computes the
  bootstrap average $\optavg_2 \defeq \frac{1}{\nummac}
  \sum_{i=1}^\nummac \optvar_{2,i}$, and then returns the weighted
  combination
  \begin{equation}
    \label{eqn:bootstrap-avg-def}
    \savgvec \defeq \frac{\optavg_1 - \ratio \optavg_2}{1 - \ratio}.
  \end{equation}
\end{enumerate}

The intuition for the weighted estimator~\eqref{eqn:bootstrap-avg-def} is
similar to that for standard bias correction procedures using the bootstrap or
subsampling~\citep{EfronTi93,Hall92, PolitisRoWo99}.  Roughly speaking, if $b_0
= \optvar^* - \optvar_1$ is the bias of the first estimator, then we may
approximate $b_0$ by the subsampled estimate of bias $b_1 = \optvar^* -
\optvar_2$. Then, we use the fact that
$b_1 \approx b_0 / \ratio$ to argue that $\optvar^* \approx
(\optvar_1 - \ratio \optvar_2) / (1-\ratio)$. The re-normalization
enforces that the
relative ``weights'' of $\optavg_1$ and $\optavg_2$ sum to 1. \\

The goal of this paper is to understand under what
conditions---and in what sense---the
estimators~\eqref{eqn:avgm-alg-def} and~\eqref{eqn:bootstrap-avg-def}
approach the \emph{oracle performance}, by which we mean the error of
a centralized risk minimization procedure that is given access to all
$\totalnumobs = \numobs\nummac$ samples.

\paragraph{Notation:}

Before continuing, we define the remainder of our notation. We use
$\ell_2$ to denote the usual Euclidean norm $\ltwo{\theta} =
(\sum_{j=1}^d \theta_j^2)^{\half}$.  The $\ell_2$-operator norm of a
matrix $A \in \R^{d_1 \times d_2}$ is its maximum singular value,
defined by
\begin{equation*}
  \matrixnorm{A}_2 \defeq \sup_{v \in \R^{d_2}, \ltwo{v} \le 1} \|A
  v\|_2.
\end{equation*}
A convex function $F$ is $\strongparam$-strongly convex on a set $U
\subseteq \R^d$ if for arbitrary $u, v \in U$ we have
\begin{equation*}
  F(u) \ge F(v) + \<\nabla F(v), u - v\> + \frac{\strongparam}{2}
  \ltwo{u - v}^2.
\end{equation*}
(If $F$ is not differentiable, we may replace $\nabla F$ with any
subgradient of $F$.) We let $\otimes$ denote the Kronecker product,
and for a pair of vectors $u, v$, we define the outer product $u
\otimes v = u v^\top$. For a three-times differentiable function $F$,
we denote the third derivative tensor by $\nabla^3 F$, so that for
each $u \in \dom F$ the operator $\nabla^3 F(u) : \R^{d \times d}
\rightarrow \R^d$ is linear and satisfies the relation
\begin{equation*}
  \left[\nabla^3 F(u) (v \otimes v)\right]_i = \sum_{j, k = 1}^d
  \left(\frac{\partial^3}{\partial
    u_i \partial u_j \partial u_k} F(u)\right) v_j v_k.
\end{equation*}
We denote the indicator function of an event $\event$ by $\indic{\event}$,
which is 1 if $\event$ is true and 0 otherwise.


\section{Theoretical results}
\label{sec:avg_bound}

Having described the \avgm\ and \savgm\ algorithms, we now turn to
statements of our main theorems on their statistical properties, along
with some consequences and comparison to past work.

\subsection{Smoothness conditions}

In addition to our previously stated assumptions on the population
risk, we require regularity conditions on the empirical risk
functions.  It is simplest to state these in terms of the functions
$\theta \mapsto f(\theta; x)$, and we note that, as with
Assumption~\ref{assumption:strong-convexity}, we require these to hold
only locally around the optimal point $\optvar^*$, in particular
within some Euclidean ball \mbox{$U = \{\optvar \in \R^d \mid
  \ltwo{\optvar^* - \optvar} \le \rho\} \subseteq \optvarspace$} of
radius $\rho > 0$.

\begin{assumption}[Smoothness]
\label{assumption:smoothness}
There are finite constants $\lipobj, \lipgrad$ such that the first and
the second partial derivatives of $f$ exist and satisfy the bounds
  \begin{equation*}
    \E[\ltwo{\nabla f(\optvar; \statrv)}^8] \leq \lipobj^8
    ~~~\mbox{and} ~~~ \E[\matrixnorm{\nabla^2 f(\optvar; \statrv) -
        \nabla^2 \popfun(\optvar)}_2^8] \le \lipgrad^8 \quad \mbox{for
      all $\optvar \in U$.}
  \end{equation*}
 In addition, for any $\statsample \in \statsamplespace$, the Hessian
 matrix $\nabla^2 f(\optvar; \statsample)$ is
 $\liphessian(\statsample)$-Lipschitz continuous, meaning that
\begin{align}
\label{EqnHessLip}
\matrixnorm{\nabla^2 f(\optvar'; \statsample)-\nabla^2 f(\optvar;
  \statsample)}_2 \leq \liphessian(\statsample) \ltwo{\optvar' -
  \optvar} \quad \mbox{ for all $\optvar, \optvar' \in U$.}
\end{align}
We require that $\E[\liphessian(\statrv)^8] \le \liphessian^8$ and
$\E[(\liphessian(\statrv) - \E[\liphessian(\statrv)])^8] \le
\liphessian^8$ for some finite constant $\liphessian$.
\end{assumption}

It is an important insight of our analysis that some type of
smoothness condition on the Hessian matrix, as in the Lipschitz
condition~\eqref{EqnHessLip}, is \emph{essential} in order for simple
averaging methods to work.  This necessity is illustrated by the
following example:
\begin{exam}[Necessity of Hessian conditions]
Let $X$ be a Bernoulli variable with parameter $\frac{1}{2}$, and
consider the loss function
  \begin{equation}
    \label{eqn:smoothness-loss}
    f(\optvar; \statsample) =
    \begin{cases} \optvar^2 - \optvar &
      \mbox{if}~ \statsample = 0 \\ \optvar^2 \indic{\optvar \leq 0} +
      \optvar & \mbox{if~} \statsample = 1,
    \end{cases}
  \end{equation}
  where $\indic{\optvar \leq 0}$ is the indicator of the
  event $\{ \optvar \leq 0\}$.  The associated population risk is
  $\popfun(\optvar)=\frac{1}{2}(\optvar^2 + \optvar^2 \indic{\optvar
    \leq 0})$.  Since $|\popfun'(w) - \popfun'(v)| \le 2 |w - v|$, the
  population risk is strongly convex and smooth, but it has
  discontinuous second derivative.  The unique minimizer of the
  population risk is $\optvar^* = 0$, and by an asymptotic expansion
  given in Appendix~\ref{appendix:smoothness-necessary}, it can be
  shown that $\E[\optvar_{1,i}] = \Omega(\numobs^{-\half})$.
  Consequently, the bias of $\optavg_1$ is $\Omega(\numobs^{-\half})$,
  and the \avgm\ algorithm using $\totalnumobs = \nummac \numobs$
  observations must suffer mean squared error $\E[(\optavg_1 -
    \optvar^*)^2] = \Omega(\numobs^{-1})$.
\end{exam}

The previous example establishes the necessity of a smoothness
condition.  However, in a certain sense, it is a pathological case:
both the smoothness condition given in
Assumption~\ref{assumption:smoothness} and the local strong convexity
condition given in Assumption~\ref{assumption:strong-convexity} are
relatively innocuous for practical problems.  For instance, both
conditions will hold for standard forms of regression, such as linear
and logistic, as long as the \emph{population} data covariance matrix
is not rank deficient and the data has suitable moments. Moreover, in
the linear regression case, one has $\liphessian = 0$.



\subsection{Bounds for simple averaging}

We now turn to our first theorem that provides guarantees on the
statistical error associated with the \avgm\ procedure.  We recall
that $\optvar^*$ denotes the minimizer of the population objective
function $F_0$, and that for each $i \in \{1, \ldots, \nummac\}$, we
use $\samples_i$ to denote a dataset of $\numobs$ independent samples.
For each $i$, we use \mbox{$\optvar_{i} \in \argmin_{\optvar
    \in \optvarspace} \{ \frac{1}{\numobs} \sum_{\statsample
    \in \samples_i} f(\optvar; \statsample)\}$} to denote a minimizer
of the empirical risk for the dataset $S_i$, and we define the
averaged vector $\optavg = \frac{1}{\nummac} \sum_{i=1}^\nummac
\optvar_{i}$.  The following result bounds the mean-squared error
between this averaged estimate and the minimizer $\optvar^*$ of the
population risk.

\begin{theorem}
  \label{theorem:no-bootstrap}
  Under Assumptions~\ref{assumption:parameter-space}
  through~\ref{assumption:smoothness}, the mean-squared error is upper
  bounded as
  \begin{align}
    \lefteqn{\E\left[\ltwo{\optavg - \optvar^*}^2\right] \le
      \frac{2}{\numobs \nummac} \E\left[\ltwo{\nabla^2
          F_0(\optvar^*)^{-1} \nabla f(\optvar^*; \statrv)}^2\right] }
    \label{eqn:no-bootstrap} \\
    & \qquad\qquad ~ + \frac{c}{\strongparam^2 \numobs^2}
    \left(\lipgrad^2 \log d + \frac{\liphessian^2 \lipobj^2}{\strongparam^2} \right)
    \E\left[\ltwo{\nabla^2 F_0(\optvar^*)^{-1} \nabla f(\optvar^*;
        \statrv)}^2\right] \nonumber \\ & \qquad\qquad ~ +
    \order(\nummac^{-1} \numobs^{-2}) + \order(\numobs^{-3}),
    \nonumber
  \end{align}
  where $c$ is a numerical constant.
\end{theorem}


\noindent A slightly weaker corollary of
Theorem~\ref{theorem:no-bootstrap} makes it easier to parse.  In
particular, note that
\begin{equation}
\label{EqnWeak}
  \ltwo{\nabla^2 F_0(\optvar^*)^{-1} \nabla f(\optvar^*; \statsample)}
  \; \stackrel{(i)}{\leq} \; \matrixnorm{\nabla^2
    F_0(\optvar^*)^{-1}}_2 \ltwo{\nabla f(\optvar^*; \statsample)} \;
  \stackrel{(ii)}{\leq} \;  \frac{1}{\strongparam} \ltwo{\nabla
    f(\optvar^*; \statsample)},
\end{equation}
where step (i) follows from the inequality $\matrixnorm{A x}_2 \le
\matrixnorm{A} \ltwo{x}$, valid for any matrix $A$ and vector $x$; and
step (ii) follows from Assumption~\ref{assumption:strong-convexity}.
In addition, Assumption~\ref{assumption:smoothness} implies
$\E[\ltwo{\nabla f(\optvar^*; \statrv)}^2] \le \lipobj^2$, and
putting together the pieces, we have established the following.
\begin{corollary}
  \label{corollary:no-bootstrap}
  Under the same conditions as Theorem~\ref{theorem:no-bootstrap},
  \begin{align}
    \label{EqnWeakBound}
    \E\left[\ltwo{\optavg - \optvar^*}^2\right] & \le
    \frac{2\lipobj^2}{\strongparam^2 \numobs \nummac} + \frac{ c
      \lipobj^2}{\strongparam^4 \numobs^2} \left(\lipgrad^2 \log d +
    \frac{\liphessian^2 \lipobj^2}{\strongparam^2}\right) 
    + \order(\nummac^{-1} \numobs^{-2}) +
    \order(\numobs^{-3}).
  \end{align}
\end{corollary}
This upper bound shows that the leading term decays proportionally to
$(\numobs \nummac)^{-1}$, with the pre-factor depending inversely on
the strong convexity constant $\strongparam$ and growing
proportionally with the bound $\lipobj$ on the loss gradient.
Although easily interpretable, the upper bound~\eqref{EqnWeakBound}
can be loose, since it is based on the relatively weak series of
bounds~\eqref{EqnWeak}. \\

The leading term in our original upper
bound~\eqref{eqn:no-bootstrap} involves the product of the gradient
$\nabla f(\optvar^*; X)$ with the inverse Hessian.  In many
statistical settings, including the problem of linear regression, the
effect of this matrix-vector multiplication is to perform some type of
standardization.  When the loss $f(\cdot; \statsample) : \optvarspace
\rightarrow \R$ is actually the negative log-likelihood
$\loglike(\statsample \mid \optvar)$ for a parametric family of models
$\{ \statprob_\optvar \}$, we can make this intuition precise.  In
particular, under suitable regularity conditions~\citep[e.g.][Chapter
  6]{LehmannCa98}, we can define the Fisher information matrix
\begin{equation*}
  \information(\optvar^*) \defeq \E\left[\nabla \loglike(\statrv \mid
    \optvar^*) \nabla \loglike(\statrv \mid \optvar^*)^\top\right] =
  \E[\nabla^2 \loglike(\statrv\mid \optvar^*)].
\end{equation*}
Recalling that $\totalnumobs = \nummac \numobs$ is the total number of samples
available, let us define the neighborhood $B_2(\optvar, t) \defeq \{\optvar'
\in \R^d : \ltwo{\optvar' - \optvar} \le t\}$. Then under our assumptions, the
H\'ajek-Le~Cam minimax theorem~\cite[Theorem 8.11]{VanDerVaart98} guarantees
for \emph{any estimator} $\what{\optvar}_{\totalnumobs}$ based on
$\totalnumobs$ samples that
\begin{equation*}
  \lim_{c \rightarrow \infty}
  \liminf_{\totalnumobs \rightarrow \infty}
  \sup_{\optvar \in B_2(\optvar^*, c / \sqrt{\totalnumobs})}
  \totalnumobs \,
  \E_\optvar\left[\ltwobig{\what{\optvar}_{\totalnumobs} - \optvar}^2
    \right]
  \ge \tr(\information(\optvar^*)^{-1}).
\end{equation*}
In connection with
Theorem~\ref{theorem:no-bootstrap}, we obtain:
\begin{corollary}
  \label{corollary:maximum-likelihood}
  In addition to the conditions of Theorem~\ref{theorem:no-bootstrap},
  suppose that the loss functions $f(\cdot; \statsample)$ are the
  negative log-likelihood $\loglike(\statsample \mid \optvar)$ for a
  parametric family $\{\statprob_\optvar, \: \optvar \in \optvarspace
  \}$.  Then the mean-squared error is upper bounded as
  \begin{equation*}
    \E \left [\ltwo{\optavg_1 - \optvar^*}^2\right] \le
    \frac{2}{\totalnumobs} \tr(\information(\optvar^*)^{-1}) + \frac{c
      \nummac^2 \tr(\information(\optvar^*)^{-1})}{\strongparam^2
      \totalnumobs^2} \left(\lipgrad^2 \log d + \frac{\liphessian^2
    \lipobj^2}{\strongparam^2}\right) + \order(\nummac \totalnumobs^{-2}),
  \end{equation*}
  where $c$ is a numerical constant.
\end{corollary}
\begin{proof}
  Rewriting the log-likelihood in the notation of
  Theorem~\ref{theorem:no-bootstrap}, we have $\nabla \loglike(\statsample
  \mid \optvar^*) = \nabla f(\optvar^*; \statsample)$ and
  all we need to note is that
  \begin{align*}
    \information(\optvar^*)^{-1}
    & = \E\left[\information(\optvar^*)^{-1}
      \nabla \loglike(\statrv \mid \optvar^*)
      \nabla \loglike(\statrv \mid \optvar^*)^\top \information(\optvar^*)^{-1}
      \right] \\
    & = \E\Big[\left(\nabla^2 F_0(\optvar^*)^{-1} \nabla f(\optvar^*;
      \statrv)\right)
      \left(\nabla^2 F_0(\optvar^*)^{-1} \nabla f(\optvar^*;
      \statrv)\right)^\top \Big].
  \end{align*}
  Now apply the linearity of the trace and use the fact that
  $\tr(uu^\top) = \ltwo{u}^2$.
\end{proof}

Except for the factor of two in the bound,
Corollary~\ref{corollary:maximum-likelihood} shows that
Theorem~\ref{theorem:no-bootstrap} essentially achieves the best
possible result. The important aspect of our bound, however, is that
we obtain this convergence rate without calculating an estimate on all
$\totalnumobs = \nummac \numobs$ samples: instead, we calculate
$\nummac$ independent estimators, and then average them to attain the
convergence guarantee. We remark that an inspection of our proof shows
that, at the expense of worse constants on higher order terms, we can
reduce the factor of $2 / \nummac \numobs$ on the leading term in
Theorem~\ref{theorem:no-bootstrap} to $(1 + c) / \nummac \numobs$ for
any constant $c > 0$; as made clear by
Corollary~\ref{corollary:maximum-likelihood}, this is
unimprovable, even by constant factors.

As noted in the introduction, our bounds are certainly to be expected
for unbiased estimators, since in such cases averaging $\nummac$
independent solutions reduces the variance by $1/\nummac$. In this
sense, our results are similar to classical distributional convergence
results in $M$-estimation: for smooth enough problems, $M$-estimators
behave asymptotically like averages~\citep{VanDerVaart98,LehmannCa98},
and averaging multiple independent realizations reduces their
variance.  However, it is often desirable to use biased estimators,
and such bias introduces difficulty in the analysis, which we explore
more in the next section.  We also note that in contrast to classical
asymptotic results, our results are applicable to finite samples and
give explicit upper bounds on the mean-squared error.  Lastly, our
results are not tied to a specific model, which allows for fairly
general sampling distributions.

\subsection{Bounds for subsampled mixture averaging}

When the number of machines $\nummac$ is relatively small,
Theorem~\ref{theorem:no-bootstrap} and
Corollary~\ref{corollary:no-bootstrap} show that the convergence rate
of the \avgm\ algorithm is mainly determined by the first term in the
bound~\eqref{eqn:no-bootstrap}, which is at most
$\frac{\lipobj^2}{\strongparam^2 \nummac \numobs}$.  In contrast, when
the number of processors $\nummac$ grows, the second term in the
bound~\eqref{eqn:no-bootstrap}, in spite of being
$\order(\numobs^{-2})$, may have non-negligible effect.  This issue is
exacerbated when the local strong convexity parameter $\strongparam$
of the risk $\popfun$ is close to zero or the Lipschitz continuity
constant $\lipgrad$ of $\nabla f$ is large. This concern motivated our
development of the subsampled average mixture (\savgm) algorithm, to
which we now return. \\

Due to the additional randomness introduced by the subsampling in
\savgm, its analysis requires an additional smoothness condition. In
particular, recalling the Euclidean \mbox{$\rho$-neighborhood $U$} of
the optimum $\optvar^*$, we require that the loss function $f$ is
(locally) smooth through its third derivatives.
\begin{assumption}[Strong smoothness]
  \label{assumption:strong-smoothness}
  For each $\statsample \in \statsamplespace$, the third derivatives
  of $f$ are \mbox{$\lipthird(\statsample)$-Lipschitz} continuous,
  meaning that
  \begin{equation*}
    \ltwo{\left(\nabla^3 f(\optvar; \statsample) - \nabla^3
      f(\optvar'; \statsample)\right)( u \otimes u)}
    \leq \lipthird(\statsample) \ltwo{\optvar - \optvar'} \ltwo{u}^2 \quad
    \mbox{for all $\optvar, \optvar' \in U$, and $u \in \R^d$,}
  \end{equation*}
  where $\E[\lipthird^8(X)] \le \lipthird^8$ for some constant
  $\lipthird < \infty$.
\end{assumption}
\noindent
It is easy to verify that
Assumption~\ref{assumption:strong-smoothness} holds for least-squares
regression with $\lipthird = 0$.  It also holds for
various types of non-linear regression problems (e.g., logistic,
multinomial etc.)  as long as the covariates have finite eighth
moments.

With this set-up, our second theorem establishes that
bootstrap sampling yields improved performance:
\begin{theorem}
  \label{theorem:bootstrap}
  Under Assumptions~\ref{assumption:parameter-space}
  through~\ref{assumption:strong-smoothness}, the output $\savgvec =
  (\optavg_1 - \ratio\optavg_2) / (1 - \ratio)$ of the bootstrap
  \savgm\ algorithm has mean-squared error bounded as
  \begin{align}
\label{EqnBootUpper}
    \E\left[\ltwo{\savgvec - \optvar^*}^2\right] & \le \frac{2 + 3
      \ratio}{(1 - \ratio)^2} \cdot \frac{1}{\numobs \nummac }
    \E\left[\ltwo{\nabla^2 F_0(\optvar^*)^{-1} \nabla f(\optvar^*;
        \statrv)}^2\right] \\
    &  + c\left(\frac{\lipthird^2 \lipobj^6}{\strongparam^6} +
    \frac{\lipobj^4 \liphessian^2 d \log d}{\strongparam^4}\right)
    \left(\frac{1}{\ratio(1 - \ratio)^2}\right) \numobs^{-3} +
    \order\left(\frac{1}{(1 - \ratio)^2}\nummac^{-1}
    \numobs^{-2}\right) \nonumber
  \end{align}
  for a numerical constant $c$.
\end{theorem}

Comparing the conclusions of Theorem~\ref{theorem:bootstrap} to those
of Theorem~\ref{theorem:no-bootstrap}, we see that the the
$\order(\numobs^{-2})$ term in the bound~\eqref{eqn:no-bootstrap} has
been eliminated.  The reason for this elimination is that subsampling
at a rate $\ratio$ reduces the bias of the \savgm\ algorithm to
$\order(\numobs^{-3})$, whereas in contrast, the bias of the
\avgm\ algorithm induces terms of order $\numobs^{-2}$.
Theorem~\ref{theorem:bootstrap} suggests that the performance of the
\savgm\ algorithm is affected by the subsampling rate $\ratio$; in
order to minimize the upper bound~\eqref{EqnBootUpper} in the regime
$\nummac < \totalnumobs^{2/3}$, the optimal choice is of the form
$\ratio \propto C \sqrt{\nummac} / \numobs = C \nummac^{3/2} /
\totalnumobs$ where \mbox{$C \approx (\lipobj^2 / \strongparam^2)
  \max\{\lipthird \lipobj / \strongparam, \liphessian \sqrt{d \log
    d}\}$.}  Roughly, as the number of machines $\nummac$ becomes
larger, we may increase $\ratio$, since we enjoy averaging affects
from the \savgm\ algorithm.

\newcommand{\stddev}{\sigma}

Let us consider the relative effects of having larger numbers of
machines $\nummac$ for both the \avgm\ and \savgm\ algorithms, which
provides some guidance to selecting $\nummac$ in practice. We
define $\stddev^2 = \E[\ltwo{\nabla^2 F_0(\optvar^*)^{-1} \nabla
    f(\optvar^*; \statrv)}^2]$ to be the asymptotic variance. Then to
obtain the optimal convergence rate of $\stddev^2 / \totalnumobs$, we
must have
\begin{equation}
  \frac{1}{\strongparam^2}
  \max\left\{\lipgrad^2 \log d, \liphessian^2 \lipobj^2\right\}
  \frac{\nummac^2}{\totalnumobs^2} \stddev^2
  \le \frac{\stddev^2}{\totalnumobs}
  ~~~ \mbox{or} ~~~
  \nummac \le \totalnumobs^{\half}
  \sqrt{\frac{\strongparam^2}{\max\{\lipgrad^2 \log d,
      \liphessian^2 \lipobj^2 / \strongparam^2 \}}}
  \label{eqn:bound-m-no-bootstrap}
\end{equation}
in Theorem~\ref{theorem:no-bootstrap}. Applying the bound of
Theorem~\ref{theorem:bootstrap}, we find that to obtain the same rate
we require
\begin{equation*}
  \max\left\{\frac{\lipthird^2 \lipobj^2}{\strongparam^6},
  \frac{\liphessian^2 d \log d}{\strongparam^4}\right\}
  \frac{\lipobj^4 \nummac^3}{\ratio \totalnumobs^3} \le \frac{(1 +
    \ratio) \stddev^2}{\totalnumobs}
  ~ \mbox{or} ~ \nummac \le
  \totalnumobs^{\frac{2}{3}} \left(\frac{\strongparam^4 \ratio(1 +
    \ratio) \stddev^2}{ \max\left\{\lipthird^2 \lipobj^6 /
    \strongparam^2, \lipobj^4 \liphessian^2 d \log
    d\right\}}\right)^{\frac{1}{3}}\!\!\!.
\end{equation*}
Now suppose that we replace $\ratio$ with $C \nummac^{3/2} /
\totalnumobs$ as in the previous paragraph.  Under the conditions
$\stddev^2 \approx \lipobj^2$ and $\ratio = o(1)$, we then find that
\begin{equation}
  \nummac \le \totalnumobs^{\frac{2}{3}} \left(\frac{\strongparam^2
    \stddev^2 \nummac^{3/2}}{ \lipobj^2 \max\left\{\lipthird \lipobj /
    \strongparam, \liphessian \sqrt{d \log
      d}\right\}\totalnumobs}\right)^{\frac{1}{3}} ~\mbox{or}~
  \nummac \le \totalnumobs^{\frac{2}{3}} \left(\frac{\strongparam^2}{
    \max\left\{\lipthird \lipobj / \strongparam, \liphessian \sqrt{d
      \log d}\right\}}\right)^{\frac{2}{3}}.
  \label{eqn:bound-m-bootstrap}
\end{equation}
Comparing inequalities~\eqref{eqn:bound-m-no-bootstrap}
and~\eqref{eqn:bound-m-bootstrap}, we see that in both cases $\nummac$
may grow polynomially with the global sample size $\totalnumobs$ while
still guaranteeing optimal convergence rates. On one hand, this
asymptotic growth is faster in the subsampled
case~\eqref{eqn:bound-m-bootstrap}; on the other hand, the dependence
on the dimension $d$ of the problem is more stringent than the
standard averaging case~\eqref{eqn:bound-m-no-bootstrap}. As the local
strong convexity constant $\strongparam$ of the \emph{population risk}
shrinks, both methods allow less splitting of the data, meaning that
the sample size per machine must be larger.  This limitation is
intuitive, since lower curvature for the population risk means that
the local empirical risks associated with each machine will inherit
lower curvature as well, and this effect will be exacerbated with a
small local sample size per machine.  Averaging methods are, of
course, not a panacea: the allowed number of partitions $\nummac$ does
not grow linearly in either case, so blindly increasing the number of
machines proportionally to the total sample size $\totalnumobs$ will
not lead to a useful estimate.

In practice, an optimal choice of $\ratio$ may not be apparent, which
may necessitate cross validation or another type of model
evaluation. We leave as intriguing open questions whether computing
multiple subsamples at each machine can yield improved performance or
reduce the variance of the \savgm\ procedure, and whether using
estimates based on resampling the data with replacement, as opposed to
without replacement as considered here, can yield improved
performance.


\subsection{Time complexity}
\label{sec:complexity}

In practice, the exact empirical minimizers assumed in
Theorems~\ref{theorem:no-bootstrap} and~\ref{theorem:bootstrap} may be
unavailable.  Instead, we need to use a finite number of iterations of
some optimization algorithm in order to obtain reasonable
approximations to the exact minimizers.  In this section, we sketch an
argument that shows that both the \avgm\ algorithm and the
\savgm\ algorithm can use such approximate empirical minimizers, and
as long as the optimization error is sufficiently small, the resulting
averaged estimate achieves the same order-optimal statistical error.
Here we provide the arguments only for the \avgm\ algorithm; the
arguments for the \savgm\ algorithm are analogous.

More precisely, suppose that each processor runs a finite number of
iterations of some optimization algorithm, thereby obtaining the
vector $\optvar_i'$ as an approximate minimizer of the objective
function $F_{1,i}$.  Thus, the vector $\optvar_i'$ can be viewed as an
approximate form of $\optvar_i$, and we let $\optavg' =
\frac{1}{\nummac} \sum_{i=1}^\nummac \optvar_{i}'$ denote the average
of these approximate minimizers, which corresponds to the output of
the approximate \avgm\ algorithm.  With this notation, we have
\begin{equation}
  \label{eqn:approximate-inequality-no-bootstrap}
  \E\left[\ltwobig{\optavg' - \optvar^*}^2\right] \;
  \stackrel{(i)}{\leq} \; 2 \E[\ltwo{\optavg - \optvar^*}^2] + 2
  \E\left[\ltwobig{\optavg' - \optavg}^2\right] \;
  \stackrel{(ii)}{\leq} \; 2 \E[\ltwo{\optavg - \optvar^*}^2] + 2
  \E[\ltwo{\optvar_1' - \optvar_1}^2],
\end{equation}
where step (i) follows by triangle inequality and the elementary bound
$(a+b)^2 \leq 2 a^2 + 2 b^2$; step (ii) follows by Jensen's
inequality.  Consequently, suppose that \mbox{processor $i$} runs
enough iterations to obtain an approximate minimizer $\optvar'_1$ such
that
\begin{align}
  \label{EqnConness}
  \E[\ltwo{\optvar_i' - \optvar_i}^2] & = \order((\nummac
  \numobs)^{-2}).
\end{align}
When this condition holds, the
bound~\eqref{eqn:approximate-inequality-no-bootstrap} shows that the
average $\optavg'$ of the approximate minimizers shares the same
convergence rates provided by Theorem~\ref{theorem:no-bootstrap}.

But how long does it take to compute an approximate minimizer
$\optvar'_i$ satisfying condition~\eqref{EqnConness}?
Assuming processing one sample requires one unit of time, we claim
that this computation can be performed in time $\order(\numobs
\log(\nummac \numobs))$.  In particular, the following two-stage
strategy, involving a combination of stochastic gradient descent (see
the following subsection for more details) and standard gradient
descent, has this complexity:
\begin{enumerate}
\item[(1)] As shown in the proof of
  Theorem~\ref{theorem:no-bootstrap}, with high probability, the
  empirical risk $F_1$ is strongly convex in a ball
  $B_\rho(\optvar_1)$ of constant radius $\rho > 0$ around
  $\optvar_1$.  Consequently, performing stochastic gradient descent
  on $F_1$ for $\order(\log^2(\nummac \numobs) / \rho^2)$ iterations
  yields an approximate minimizer that falls within
  $B_\rho(\optvar_1)$ with high probability~\citep[e.g.][Proposition
    2.1]{NemirovskiJuLaSh09}. Note that the radius $\rho$ for local
  strong convexity is a property of the population risk $F_0$ we
  use as a prior knowledge.
\item[(2)] This initial estimate can be further improved by a few iterations
  of standard gradient descent.  Under local strong convexity of the objective
  function, gradient descent is known to converge at a geometric
  rate~\citep[see, e.g.][]{NocedalWr06,BoydVa04}, so
  $\order(\log(1/\epsilon))$ iterations will reduce the error to order
  $\epsilon$.  In our case, we have $\epsilon = (\nummac \numobs)^{-2}$, and
  since each iteration of standard gradient descent requires $\order(\numobs)$
  units of time, a total of $\order(\numobs \log(\nummac \numobs))$ time units
  are sufficient to obtain a final estimate $\optvar_1'$ satisfying
  condition~\eqref{EqnConness}.
\end{enumerate}

\noindent Overall, we conclude that the speed-up of the \avgm\,
relative to the naive approach of processing all $\totalnumobs =
\nummac \numobs$ samples on one processor, is at least of order
$\nummac / \log(\totalnumobs)$.

\subsection{Stochastic gradient descent with averaging}

The previous strategy involved a combination of stochastic gradient
descent and standard gradient descent.  In many settings, it may be
appealing to use only a stochastic gradient algorithm, due to their
ease of their implementation and limited computational requirements.
In this section, we describe an extension of
Theorem~\ref{theorem:no-bootstrap} to the case in which each machine
computes an approximate minimizer using only stochastic gradient descent.

Stochastic gradient algorithms have a lengthy history in
statistics, optimization, and machine learning~\citep{RobbinsMo51, PolyakJu92,
  NemirovskiJuLaSh09, RakhlinShSr12}.  Let us begin by briefly
reviewing the basic form of stochastic gradient descent (SGD).
Stochastic gradient descent algorithms iteratively update a parameter
vector $\optvar^t$ over time based on randomly sampled gradient
information. Specifically, at iteration $t$, a sample $\statrv_t$ is
drawn at random from the distribution $\statprob$ (or, in the case of
a finite set of data $\{\statrv_1, \ldots, \statrv_\numobs\}$, a
sample $\statrv_t$ is chosen from the data set). The method then
performs the following two steps:
\begin{equation}
  \optvar^{t + \half} = \optvar^t - \stepsize_t \nabla f(\optvar^t;
  \statrv_t)
  ~~~ \mbox{and} ~~~
  \optvar^{t + 1} = \argmin_{\optvar \in \optvarspace}
  \left\{ \ltwobig{\optvar - \optvar^{t + \half}}^2 \right\}.
  \label{eqn:sgd-iteration}
\end{equation}
Here $\stepsize_t > 0$ is a stepsize, and the first update
in~\eqref{eqn:sgd-iteration} is a gradient descent step with respect to the
random gradient $\nabla f(\optvar^t; \statrv_t)$. The method then projects the
intermediate point $\optvar^{t + \half}$ back onto the constraint set
$\optvarspace$ (if there is a constraint set). The convergence of SGD methods
of the form~\eqref{eqn:sgd-iteration} has been well-studied, and we refer the
reader to the papers by~\citet{PolyakJu92}, \citet{NemirovskiJuLaSh09}, and
\citet{RakhlinShSr12} for deeper investigations.

To prove convergence of our stochastic gradient-based averaging
algorithms, we require the following smoothness and strong convexity
condition, which is an alternative to the
Assumptions~\ref{assumption:strong-convexity}
and~\ref{assumption:smoothness} used previously.
\begin{assumption}[Smoothness and Strong Convexity II]
  \label{assumption:smoothness-sgd}
  There exists a function $\liphessian : \statsamplespace \rightarrow
  \R_+$ such that
  \begin{equation*}
    \matrixnorm{\nabla^2 f(\optvar; \statsample) - \nabla^2
      f(\optvar^*; \statsample)}_2 \le \liphessian(\statsample)
    \ltwo{\optvar - \optvar^*}
    ~~~ \mbox{for~all}~\statsample \in \statsamplespace,
  \end{equation*}
  and $\E[\liphessian^2(\statrv)] \le \liphessian^2 < \infty$.
  There are finite constants $\lipobj$ and $\lipgrad$ such
  that
  \begin{equation*}
    \E[\ltwo{\nabla f(\optvar; \statrv)}^4] \le \lipobj^4, \quad
    \mbox{and} \quad \E[\matrixnorm{\nabla^2 f(\optvar^*; \statrv)}_2^4]
    \le \lipgrad^4 \quad \mbox{for each fixed} ~ \optvar \in \optvarspace.
  \end{equation*}
  In addition, the population function $F_0$ is $\strongparam$-strongly
  convex over the space $\optvarspace$, meaning that
  \begin{align}
    \nabla^2 \popfun(\optvar) \succeq \strongparam I_{d \times d} \quad
    \mbox{for all} ~ \optvar \in \optvarspace.
  \end{align}
\end{assumption}
\noindent
Assumption~\ref{assumption:smoothness-sgd} does not require as many
moments as does Assumption~\ref{assumption:smoothness}, but it does
require each moment bound to hold globally, that is, over the entire
space $\optvarspace$, rather than only in a neighborhood of the
optimal point $\optvar^*$. Similarly, the necessary curvature---in the
form of the lower bound on the Hessian matrix $\nabla^2 \popfun$---is
also required to hold globally, rather than only locally. Nonetheless,
Assumption~\ref{assumption:smoothness-sgd} holds for many common
problems; for instance, it holds for any linear regression problem in
which the covariates have finite fourth moments and the domain
$\optvarspace$ is compact. \\

The averaged stochastic gradient algorithm
(\sgdavgm) is based on the following two steps:
\begin{enumerate}
\item[(1)] Given some constant $c > 1$, each machine performs
  $\numobs$ iterations of stochastic gradient
  descent~\eqref{eqn:sgd-iteration} on its local dataset of $\numobs$
  samples using the stepsize $\stepsize_t = \frac{c}{\strongparam t}$,
  then outputs the resulting local parameter $\optvar'_i$.
\item[(2)] The algorithm computes the average $\optavg^\numobs =
  \frac{1}{\nummac} \sum_{i=1}^\nummac \optvar'_i$.
\end{enumerate}
The following result characterizes the mean-squared error of this
procedure in terms of the constants
\begin{equation*}
  \alpha \defeq 4c^2 \quad
  \mbox{and} \quad
  \beta \defeq \max\left\{\ceil{\frac{c
      \lipgrad}{\strongparam}}, \frac{c \alpha^{3/4}
    \lipobj^{3/2}}{(c - 1)\strongparam^{5/2}} \left(
  \frac{\alpha^{1/4} \liphessian \lipobj^{1/2}}{\strongparam^{1/2}}
  + \frac{4 \lipobj + \lipgrad \radius}{\rho^{3/2}} \right)\right\}.
\end{equation*}
\begin{theorem}
  \label{theorem:sgd}
  Under Assumptions~\ref{assumption:parameter-space}
  and~\ref{assumption:smoothness-sgd}, the output $\optavg^\numobs$ of
  the \savgm\ algorithm has mean-squared error upper bounded as
  \begin{equation}
    \E\left[\ltwobig{\optavg^\numobs - \optvar^*}^2\right]
    \leq
    \frac{\alpha \lipobj^2}{\strongparam^2 \nummac \numobs} +
    \frac{\beta^2}{\numobs^{3/2}}.
    \label{eqn:sgd-bound}
  \end{equation}
\end{theorem}

Theorem~\ref{theorem:sgd} shows that the averaged stochastic gradient descent
procedure attains the optimal convergence rate $\order(\totalnumobs^{-1})$
as a function of the total
number of observations $\totalnumobs = \nummac \numobs$. The constant and
problem-dependent factors
are somewhat worse than those in the earlier results we presented in
Theorems~\ref{theorem:no-bootstrap} and~\ref{theorem:bootstrap}, but the
practical implementability of such a procedure may in some circumstances
outweigh those differences. We also note that the second term of order
$\order(\numobs^{-3/2})$ may be reduced to $\order(\numobs^{(2-2k)/k})$ for
any $k \ge 4$ by assuming the existence of $k$th moments in
Assumption~\ref{assumption:smoothness-sgd}; we show this in passing after our
proof of the theorem in Appendix~\ref{sec:proof-sgd}.  It is not clear whether
a bootstrap correction is possible for the stochastic-gradient based
estimator; such a correction could be significant, because the term $\beta^2 /
\numobs^{3/2}$ arising from the bias in the stochastic gradient estimator may
be non-trivial. We leave this question to future work.


\section{Performance on synthetic data}
\label{sec:simulations}

In this section, we report the results of simulation studies comparing the
\avgm, \savgm, and \sgdavgm\ methods, as well as a trivial method using only a
fraction of the data available on a single machine. For each of our simulated
experiments, we use a fixed total number of samples $\totalnumobs =
100,\!000$, but we vary the number of parallel splits $\nummac$ of the data
(and consequently, the local dataset sizes $\numobs = \totalnumobs / \nummac$)
and the dimensionality $d$ of the problem solved.

For our experiments, we simulate data from one of three regression models:
\begin{subequations}
  \begin{align}
    y & = \<u, x\> + \varepsilon, \label{eqn:standard-regression} \\
    y & = \<u, x\> + \sum_{j = 1}^d v_j x_j^3
    + \varepsilon, ~~~~ \mbox{or} ~
    \label{eqn:under-specified-regression} \\
    y & = \<u, x\> + h(x) |\varepsilon|,
    \label{eqn:heteroskedastic-regression}
  \end{align}
\end{subequations}
where $\varepsilon \sim N(0,1)$, and $h$ is a function to be
specified.  Specifically, the data generation procedure is as follows.
For each individual simulation, we choose fixed vector $u \in \R^d$
with entries $u_i$ distributed uniformly in $[0, 1]$ (and similarly
for $v$), and we set \mbox{$h(x) = \sum_{j=1}^d (x_j/2)^3$.} The
models~\eqref{eqn:standard-regression}
through~\eqref{eqn:heteroskedastic-regression} provide points on a
curve from correctly-specified to grossly mis-specified models, so
models~\eqref{eqn:under-specified-regression}
and~\eqref{eqn:heteroskedastic-regression} help us understand the
effects of subsampling in the \savgm\ algorithm.  (In contrast, the
standard least-squares estimator is unbiased for
model~\eqref{eqn:standard-regression}.)  The noise variable
$\varepsilon$ is always chosen as a standard Gaussian variate
$\normal(0, 1)$, independent from sample to sample.

In our simulation experiments we use the least-squares loss
\begin{equation*}
  f(\optvar; (x, y)) \defeq \half(\<\optvar, x\> - y)^2.
\end{equation*}
The goal in each experiment is to
estimate the vector $\optvar^*$ minimizing $\popfun(\optvar)
\defeq \E[f(\optvar; (X, Y))]$. For each simulation, we generate
$\totalnumobs$ samples according to either the
model~\eqref{eqn:standard-regression}
or~\eqref{eqn:heteroskedastic-regression}. For each $\nummac \in \{2,
4, 8, 16, 32, 64, 128\}$, we estimate $\optvar^* = \arg \min_\optvar
F_0(\optvar)$ using a parallel method with data split into $\nummac$
independent sets of size $\numobs = \totalnumobs / \nummac$,
specifically
\begin{enumerate}[(i)]
\item The \avgm\ method
  \vspace{-.2cm}
\item The \savgm\ method with several settings of the subsampling
  ratio $\ratio$
  \vspace{-.2cm}
\item The \sgdavgm\ method with stepsize $\stepsize_t
  = d / (10(d + t))$, which gave good performance.
\end{enumerate}
In addition to (i)--(iii), we also estimate $\optvar^*$ with
\begin{enumerate}[(i)]
  \setcounter{enumi}{3}
\item The empirical minimizer of a single split of the data of size
  $\numobs = \totalnumobs / \nummac$
  \vspace{-.2cm}
\item The empirical minimizer on the full dataset (the oracle solution).
\end{enumerate}

\begin{figure}[t!]
  \begin{center}
    \begin{tabular}{cc}
      \includegraphics[width=.45\columnwidth]{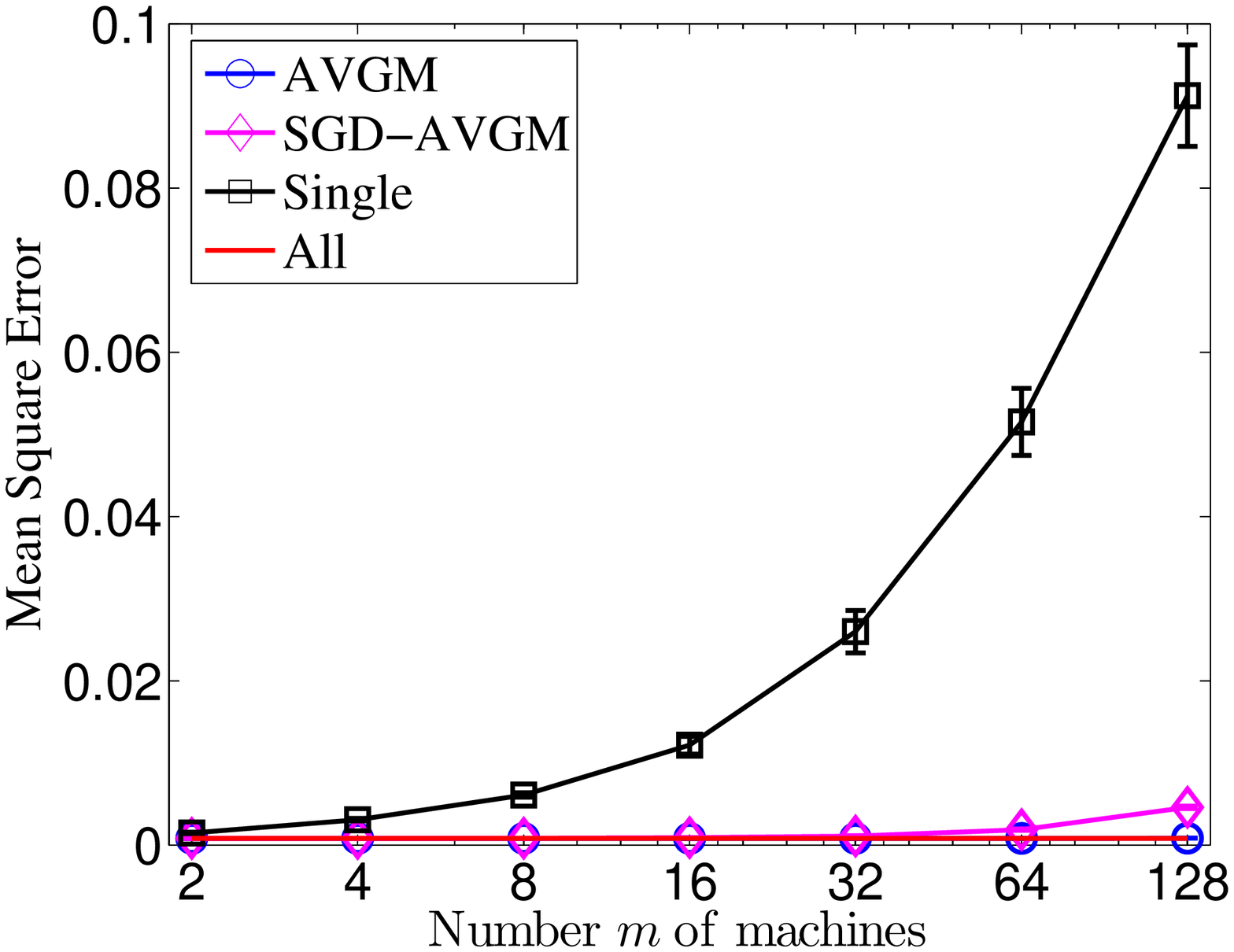} &
      \includegraphics[width=.45\columnwidth]{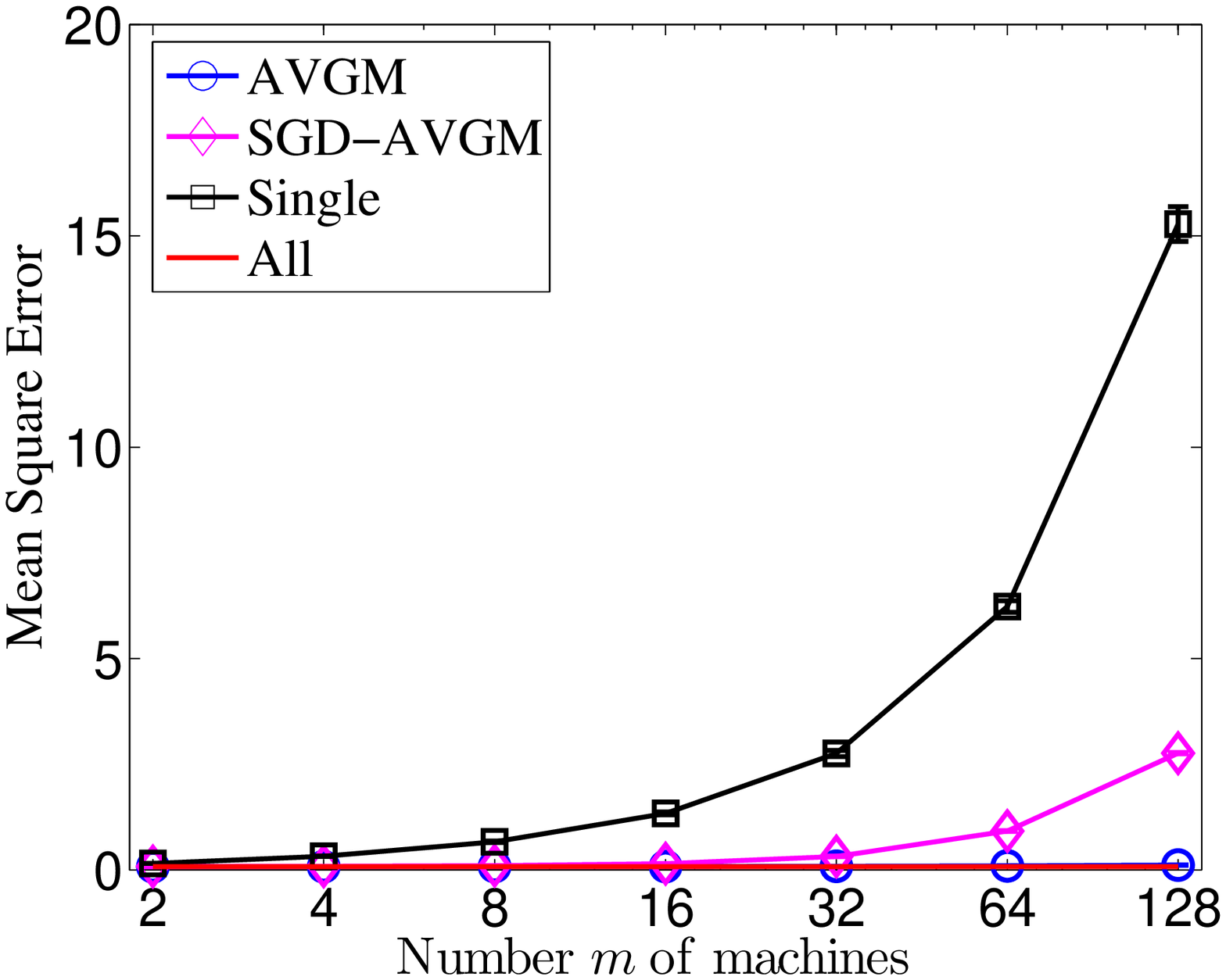} \\
      (a) $d = 20$ & (b) $d = 200$
    \end{tabular}
  \caption{\label{fig:avgm} The error $\ltwos{\what{\optvar} - \optvar^*}^2$
    versus number of machines, with standard errors across twenty simulations,
    for solving least squares with data generated according to the normal
    model~\eqref{eqn:standard-regression}. The oracle
    least-squares estimate using all $\totalnumobs$ samples is given by the
    line ``All,'' while the line ``Single'' gives the performance of the naive
    estimator using only $\numobs = \totalnumobs / \nummac$ samples.}
  \end{center}
\end{figure}

\begin{figure}[t!]
  \begin{center}
    \begin{tabular}{cc}
      \includegraphics[width=.45\columnwidth]{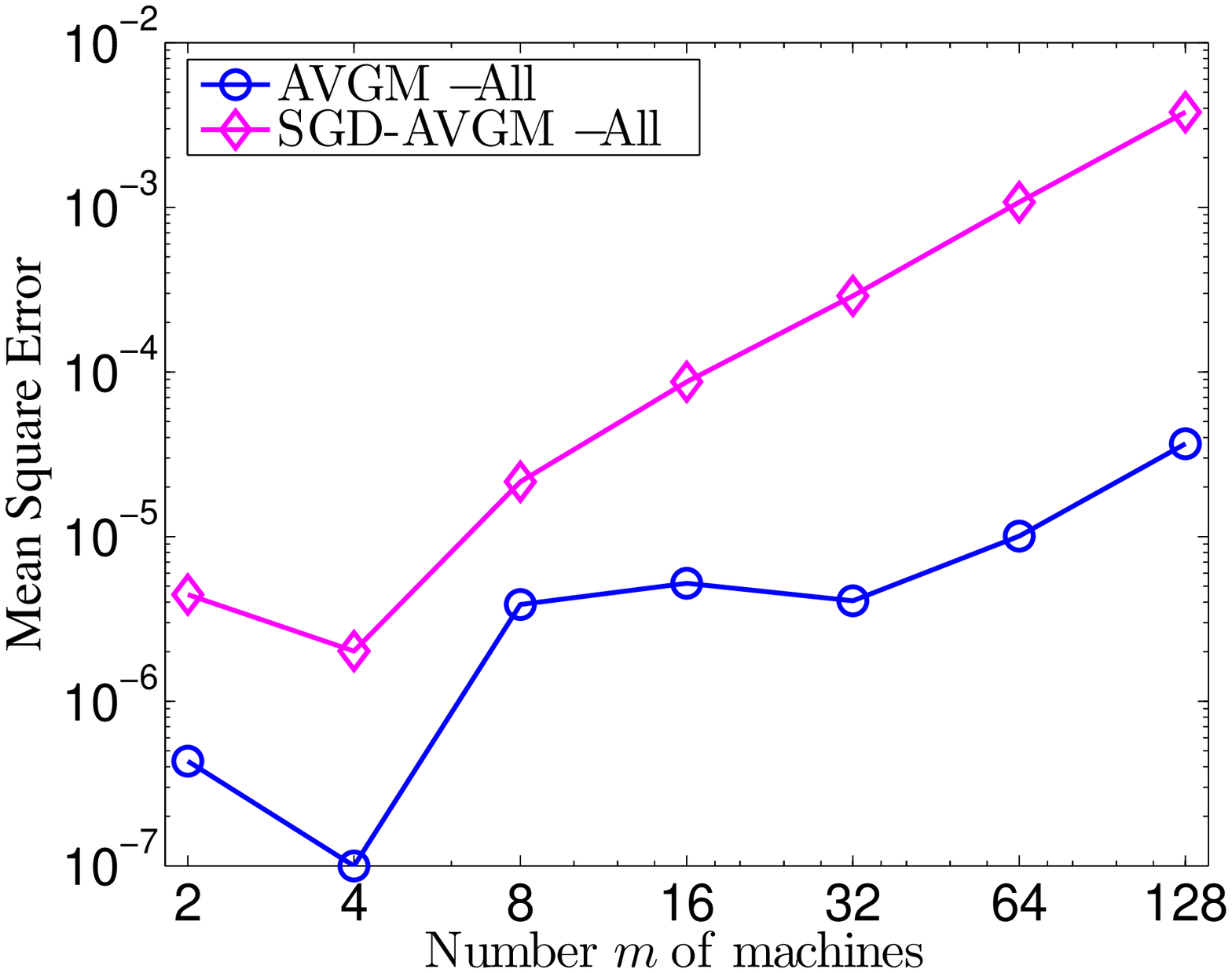} &
      \includegraphics[width=.45\columnwidth]{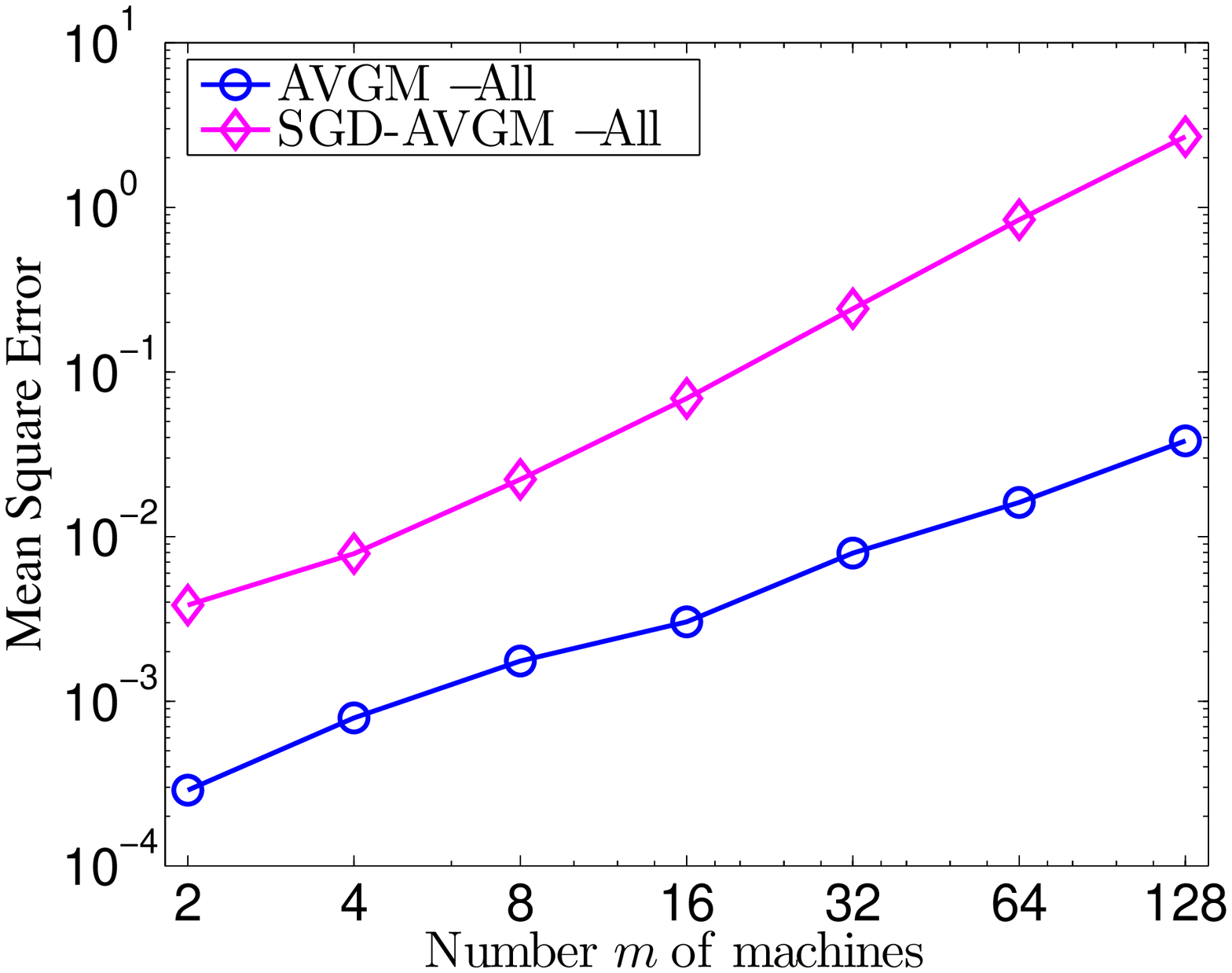} \\
      (a) $d = 20$ & (b) $d = 200$
    \end{tabular}
    \caption{\label{fig:avgm_withsgdonly} Comparison of \avgm\ and
      \sgdavgm\ methods as in Figure~\ref{fig:avgm} plotted on logarithmic
      scale.  The plot shows $\ltwos{\what{\optvar} - \optvar^*}^2 -
      \ltwo{\optvar_\totalnumobs - \optvar^*}^2$, where $\optvar_\totalnumobs$
      is the oracle least-squares estimator using all $\totalnumobs$ data
      samples.}
  \end{center}
\end{figure}

\subsection{Averaging methods}

For our first set of experiments, we study the performance of the averaging
methods (\avgm\ and \savgm), showing their scaling as the number of splits of
data---the number of machines $\nummac$---grows for fixed $\totalnumobs$ and
dimensions $d = 20$ and $d = 200$.  We use the standard regression
model~\eqref{eqn:standard-regression} to generate the data, and throughout we
let $\what{\optvar}$ denote the estimate returned by the method under
consideration (so in the \avgm\ case, for example, this is the vector
$\what{\optvar} \defeq \optavg_1$). The data samples consist of pairs $(x,
y)$, where $x \in \R^d$ and $y \in \R$ is the target value.  To sample each
$x$ vector, we choose five distinct indices in $\{1, \ldots, d\}$ uniformly at
random, and the entries of $x$ at those indices are distributed as $\normal(0,
1)$. For the model~\eqref{eqn:standard-regression}, the population optimal
vector $\optvar^*$ is $u$.

In Figure~\ref{fig:avgm}, we plot the error $\ltwos{\what{\optvar} -
  \optvar^*}^2$ of the inferred parameter vector $\what{\optvar}$ for
the true parameters $\optvar^*$ versus the number of splits $\nummac$,
or equivalently, the number of separate machines available for use.  We
also plot standard errors (across twenty experiments) for each
curve. As a baseline in each plot, we plot as a red line the
squared error $\ltwos{\what{\optvar}_\totalnumobs - \optvar^*}^2$ of
the centralized ``gold standard,'' obtained by applying a batch method
to all $\totalnumobs$ samples.

From the plots in Figure~\ref{fig:avgm}, we can make a few observations.  The
\avgm\ algorithm enjoys excellent performance, as predicted by our theoretical
results, especially compared to the naive solution using only a fraction
$1/\nummac$ of the data. In particular, if $\what{\optvar}$ is obtained by the
batch method, then \avgm\ is almost as good as the full-batch baseline even
for $\nummac$ as large as $128$, though there is some evident degradation in
solution quality. The \sgdavgm\ (stochastic-gradient with averaging) solution
also yields much higher accuracy than the naive solution, but its performance
degrades more quickly than the \avgm\ method's as $\nummac$ grows.  In higher
dimensions, both the \avgm\ and \sgdavgm\ procedures have somewhat worse
performance; again, this is not unexpected since in high dimensions the strong
convexity condition is satisfied with lower probability in local datasets.

We present a comparison between the \avgm\ method and the
\sgdavgm\ method with somewhat more distinguishing power in
Figure~\ref{fig:avgm_withsgdonly}. For these plots, we compute the gap
between the \avgm\ mean-squared-error and the unparallel baseline MSE,
which is the accuracy lost due to parallelization or distributing the
inference procedure across multiple machines.
Figure~\ref{fig:avgm_withsgdonly} shows that the mean-squared error
grows polynomially with the number of machines $\nummac$, which is
consistent with our theoretical
results. From Corollary~\ref{corollary:maximum-likelihood}, we
expect the \avgm\ method to suffer (lower-order) penalties
proportional to $\nummac^2$ as $\nummac$ grows, while
Theorem~\ref{theorem:sgd} suggests the somewhat faster growth we see
for the \sgdavgm\ method in Figure~\ref{fig:avgm_withsgdonly}.  Thus,
we see that the improved run-time performance of the
\sgdavgm\ method---requiring only a single pass through the data on
each machine, touching each datum only once---comes at the expense of
some loss of accuracy, as measured by mean-squared error.

\subsection{Subsampling correction}

\begin{figure}[t]
  \begin{center}
  \begin{tabular}{cc}
    \includegraphics[width=.45\columnwidth]{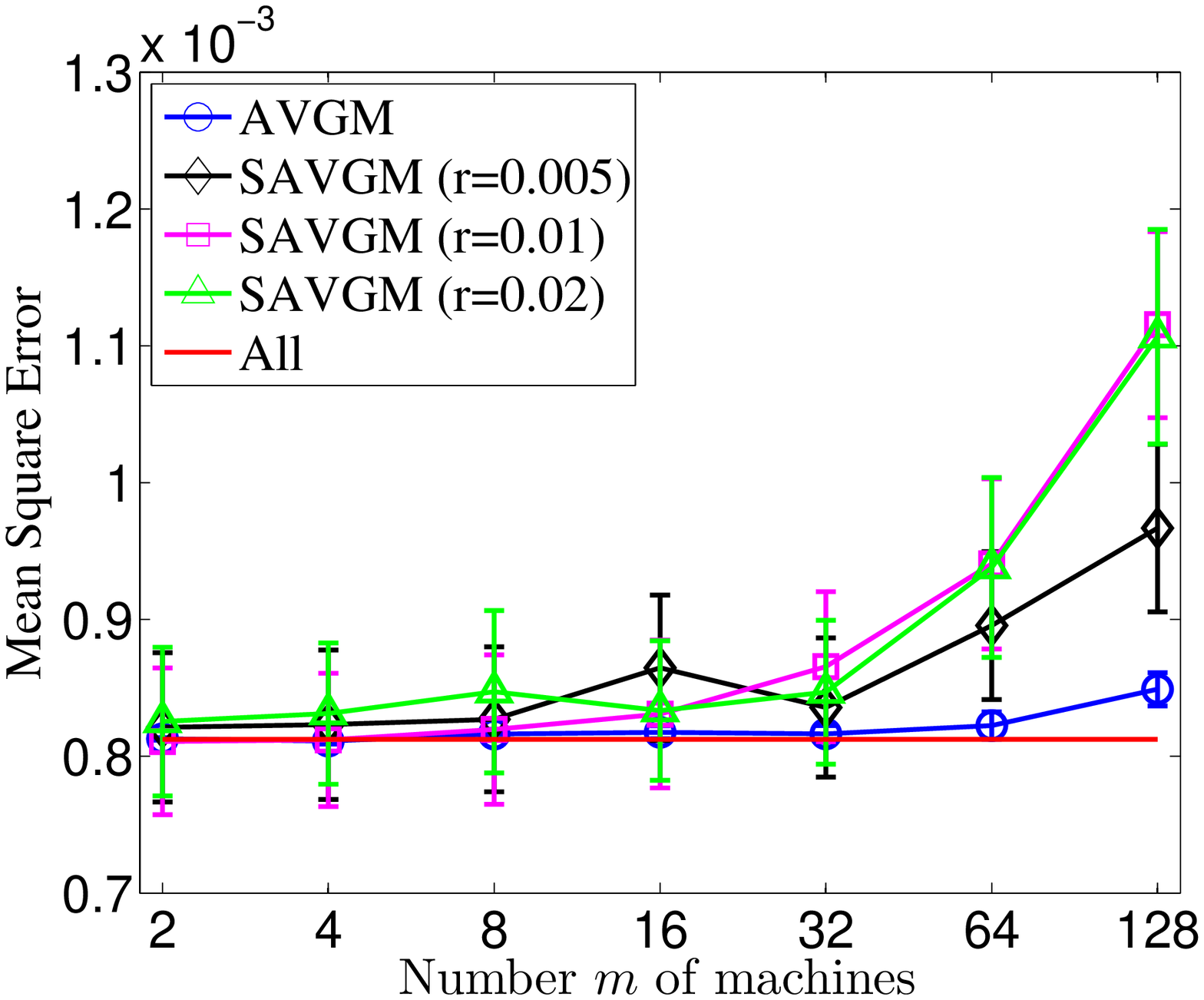} &
    \includegraphics[width=.45\columnwidth]{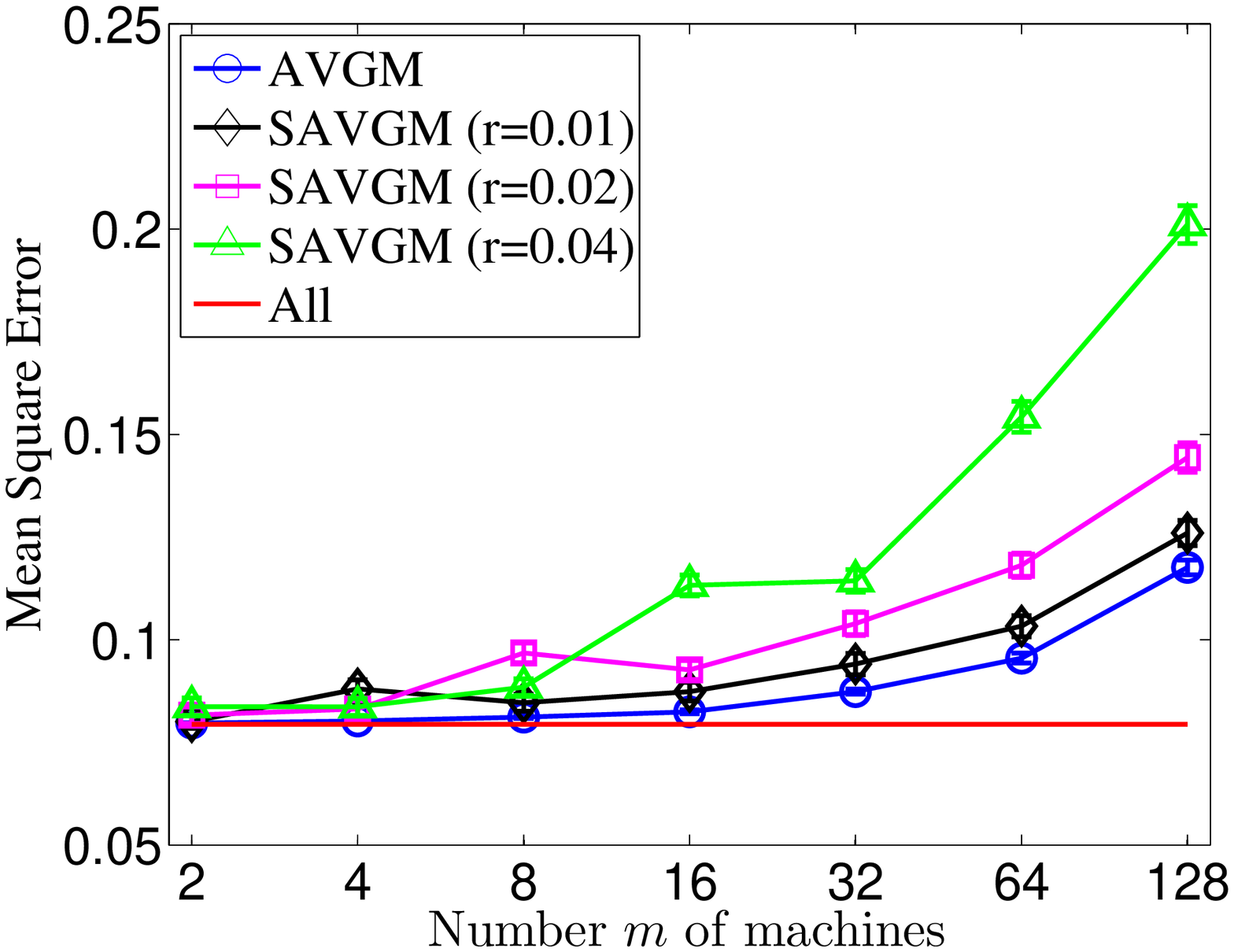} \\
    (a) $d = 20$ & (b) $d = 200$
  \end{tabular}
  \caption{\label{fig:bavgm-ind} The error $\ltwos{\what{\optvar} -
      \optvar^*}^2$ plotted against the number of machines $\nummac$ for the
    \avgm\ and \savgm\ methods, with standard errors across twenty
    simulations, using the normal regression
    model~\eqref{eqn:standard-regression}. The oracle estimator is
    denoted by the line ``All.''}
  \end{center}
  \vspace{-.5cm}
\end{figure}

\begin{figure}[t]
  \begin{center}
  \begin{tabular}{cc}
    \includegraphics[width=.45\columnwidth]{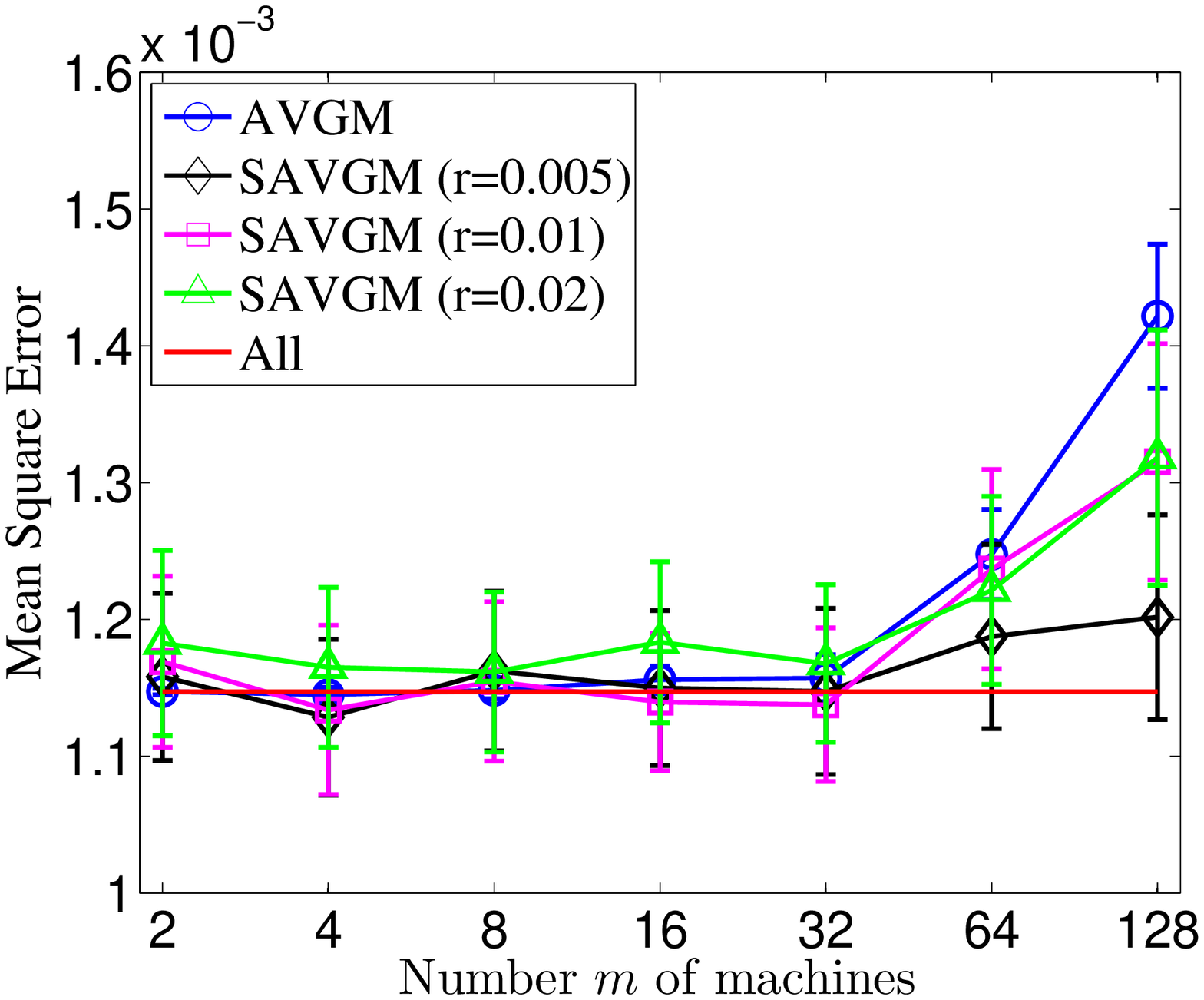} &
    \includegraphics[width=.45\columnwidth]{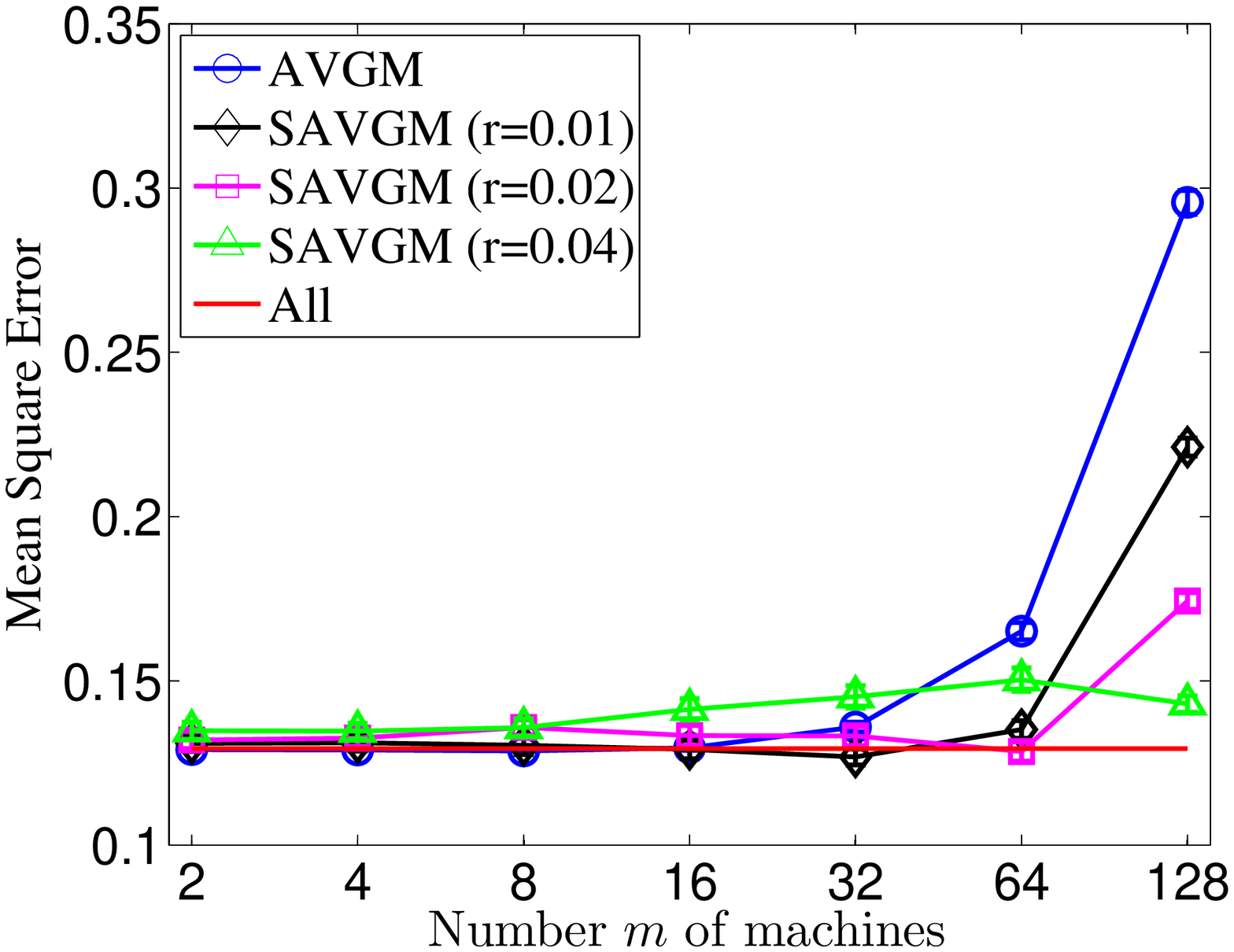} \\
    (a) $d = 20$ & (b) $d = 200$
  \end{tabular}
  \caption{\label{fig:bavgm-cor}  The error $\ltwos{\what{\optvar} -
      \optvar^*}^2$ plotted against the number of machines $\nummac$ for the
    \avgm\ and \savgm\ methods, with standard errors across twenty
    simulations, using the non-normal regression
    model~\eqref{eqn:heteroskedastic-regression}. The oracle estimator is
    denoted by the line ``All.''}
  \end{center}
  \vspace{-.5cm}
\end{figure}

We now turn to developing an understanding of the \savgm\ algorithm in
comparison to the standard average mixture algorithm, developing
intuition for the benefits and drawbacks of the method.  Before
describing the results, we remark that for the standard regression
model~\eqref{eqn:standard-regression}, the least-squares solution is
unbiased for $\optvar^*$, so we expect subsampled averaging to yield
little (if any) improvement. The \savgm\ method is essentially aimed
at correcting the bias of the estimator $\optavg_1$, and de-biasing an
unbiased estimator only increases its variance. However, for the
mis-specified models~\eqref{eqn:under-specified-regression}
and~\eqref{eqn:heteroskedastic-regression} we expect to see some
performance gains. In our experiments, we use multiple sub-sampling
rates to study their effects, choosing
$\ratio\in\{0.005,0.01,0.02,0.04\}$, where we recall that the output
of the \savgm\ algorithm is the vector \mbox{$\what{\optvar} \defeq
  (\optavg_1 - \ratio\optavg_2) / (1 - \ratio)$.}

We begin with experiments in which the data is generated as in the
previous section. That is, to generate a feature vector $x \in
\real^d$, choose five distinct indices in $\{1, \ldots, d\}$ uniformly
at random, and the entries of $x$ at those indices are distributed as
$\normal(0, 1)$.  In Figure~\ref{fig:bavgm-ind}, we plot the results
of simulations comparing \avgm\ and \savgm\ with data generated from
the normal regression model~\eqref{eqn:standard-regression}. Both
algorithms have have low error rates, but the \avgm\ method is
slightly better than the \savgm\ method for both values of the
dimension $d$ and all and sub-sampling rates $\ratio$.  As expected,
in this case the \savgm\ method does not offer improvement over \avgm,
since the estimators are unbiased. (In Figure~\ref{fig:bavgm-ind}(a),
we note that the standard error is in fact very small, since the
mean-squared error is only of order $10^{-3}$.)

To understand settings in which subsampling for bias correction helps, in
Figure~\ref{fig:bavgm-cor}, we plot mean-square error curves for the
least-squares regression problem when the vector $y$ is sampled according to
the non-normal regression model~\eqref{eqn:heteroskedastic-regression}.  In
this case, the least-squares estimator is biased for $\optvar^*$ (which, as
before, we estimate by solving a larger regression problem using $10
\totalnumobs$ data samples).  Figure~\ref{fig:bavgm-cor} shows that both the
\avgm\ and \savgm\ method still enjoy good performance; in some cases, the
\savgm\ method even beats the oracle least-squares estimator for $\optvar^*$
that uses all $\totalnumobs$ samples. Since the \avgm\ estimate is biased in
this case, its error curve increases roughly quadratically with $\nummac$,
which agrees with our theoretical predictions in
Theorem~\ref{theorem:no-bootstrap}.  In contrast, we see that the
\savgm\ algorithm enjoys somewhat more stable performance, with increasing
benefit as the number of machines $\nummac$ increases. For example, in case of
$d=200$, if we choose $\ratio = 0.01$ for $\nummac \leq 32$, choose $\ratio =
0.02$ for $\nummac = 64$ and $\ratio = 0.04$ for $\nummac = 128$, then
\savgm\ has performance comparable with the oracle method that uses all
$\totalnumobs$ samples. Moreover, we see that all the values of $\ratio$---at
least for the reasonably small values we use in the experiment---provide
performance improvements over a non-subsampled distributed estimator.

For our final simulation, we plot results comparing \savgm\ with \avgm\ in
model~\eqref{eqn:under-specified-regression}, which is mis-specified but still
a normal model. We use a simpler data generating mechanism, specifically, we
draw $x \sim \normal(0, I_{d\times d})$ from a standard $d$-dimensional
normal, and $v$ is chosen uniformly in $[0, 1]$; in this case, the population
minimizer has the closed form $\optvar^* = u + 3v$.
Figure~\ref{fig:simple-mis-specified} shows the results for dimensions $d =
20$ and $d = 40$ performed over $100$ experiments (the standard errors are too
small to see). Since the model~\eqref{eqn:under-specified-regression} is not
that badly mis-specified, the performance of the \savgm\ method improves
upon that of the \avgm\ method only for relatively large values of $\nummac$,
however, the performance of the \savgm\ is always at least as good as
that of \avgm.

\begin{figure}
  \begin{center}
    \begin{tabular}{cc}
      \psfrag{Avgm}{\small{\avgm}}
      \psfrag{SAvgm (r=(d/n) 2/3)--}{\small{\savgm\
          $(\ratio = (d/\numobs)^{2/3})$}}
      \includegraphics[width=.45\columnwidth]{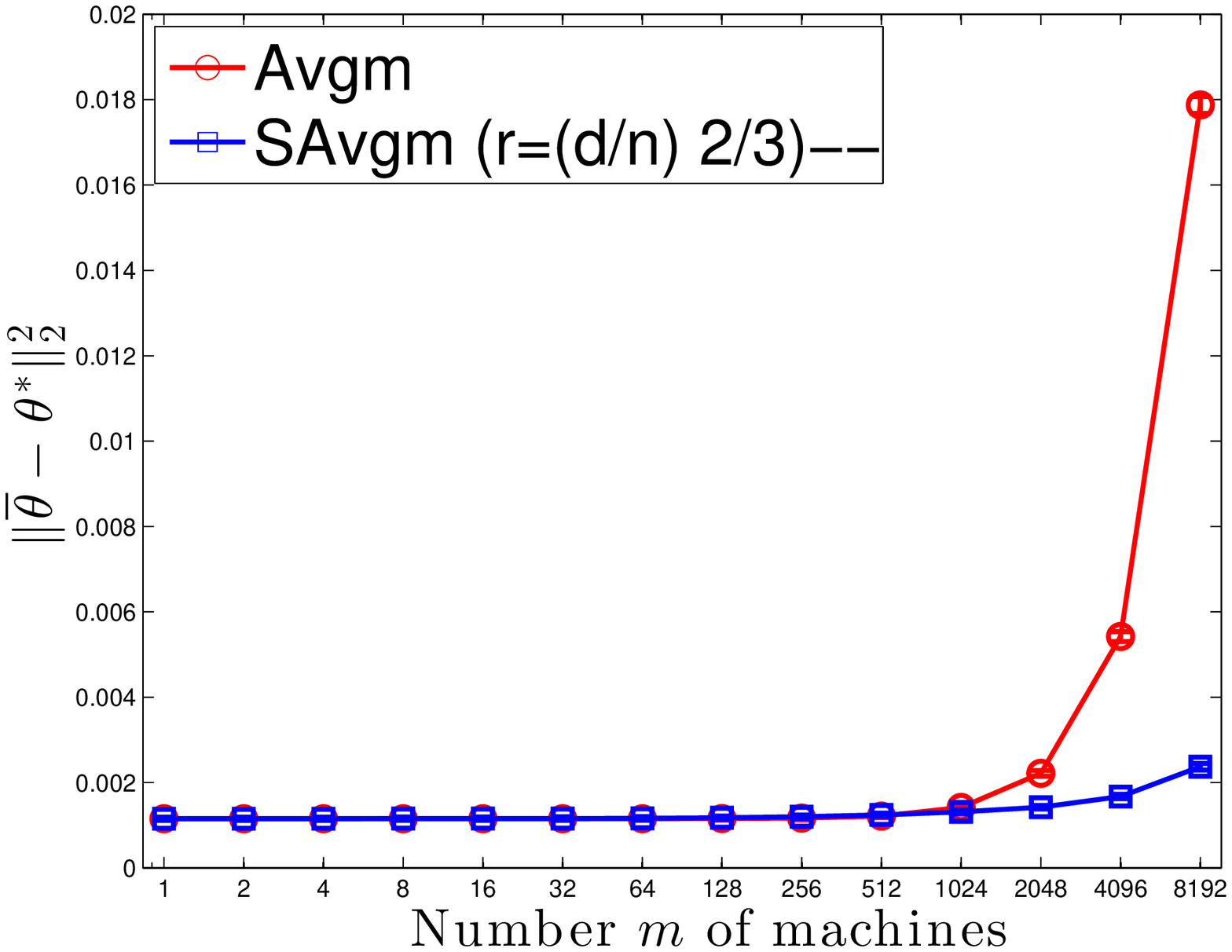} &
      \psfrag{Avgm}{\small{\avgm}}
      \psfrag{SAvgm (r=(d/n) 2/3)--}{\small{\savgm\
          $(\ratio = (d/\numobs)^{2/3})$}}
      \includegraphics[width=.45\columnwidth]{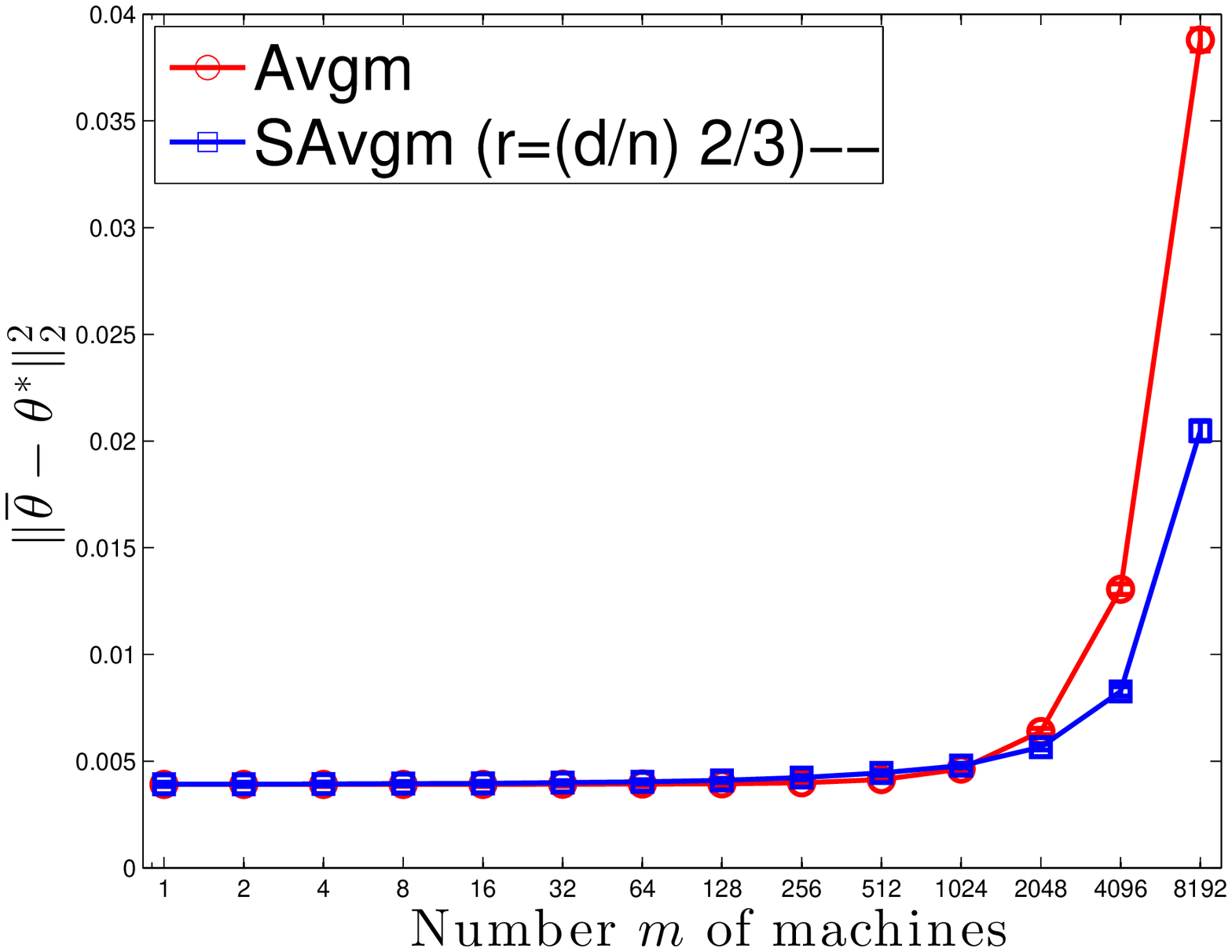} \\
      (a) $d = 20$ & (b) $d = 40$
    \end{tabular}
    \caption{\label{fig:simple-mis-specified} The error $\ltwos{\what{\optvar}
        - \optvar^*}^2$ plotted against the number of machines $\nummac$ for
      the \avgm\ and \savgm\ methods using regression
      model~\eqref{eqn:under-specified-regression}.}
  \end{center}
  \vspace{-1cm}
\end{figure}


\section{Experiments with advertising data}
\label{sec:real-experiment}

\begin{table*}[t]
  \centering
  \caption{Features used in online advertisement prediction problem.}
  \label{tab:feature}
  \begin{tabular}{|l|r|l|}\hline
    Feature Name & Dimension & Description \\
    \hline
    Query & 20000 & Word tokens appearing in the query. \\
    \hline
    Gender & 3 & Gender of the user\\\hline
    Keyword & 20000 & Word tokens appearing in the purchase keywords.\\\hline
    Title & 20000 & Word tokens appearing in the ad title.\\\hline
    Advertiser & 39191 & Advertiser's ID\\\hline
    AdID & 641707 & Advertisement's ID.\\\hline
    Age & 6 & Age of the user\\\hline
    UserFreq & 25 & Number of appearances of the same user.\\\hline
    Position & 3 & Position of advertisement on search page.\\\hline
    Depth & 3 & Number of ads in the session.\\\hline
    QueryFreq & 25 & Number of occurrences of the same query.\\\hline
    AdFreq & 25 &Number of occurrences of the same ad.\\\hline
    QueryLength & 20 & Number of words in the query.\\\hline
    TitleLength & 30 & Number of words in the ad title.\\\hline
    DespLength & 50 & Number of words in the ad description.\\\hline
    QueryCtr & 150 & Average click-through-rate for query.\\\hline
    UserCtr & 150 & Average click-through-rate for user.\\\hline
    AdvrCtr & 150 &Average click-through-rate for advertiser.\\\hline
    WordCtr & 150 & Average click-through-rate for keyword advertised.\\\hline
    UserAdFreq & 20 & Number of times this user sees an ad.\\\hline
    UserQueryFreq & 20 & Number of times this user performs a search.\\\hline
  \end{tabular}
\end{table*}

Predicting whether a user of a search engine will click on an
advertisement presented to him or her is of central importance to the
business of several internet companies, and in this section, we
present experiments studying the performance of the \avgm\ and
\savgm\ methods for this task.  We use a large dataset from the
Tencent search engine, {\tt soso.com}~\citep{Kdd12}, which contains
641,707 distinct advertisement items with $\totalnumobs =
235,\!582,\!879$ data samples.

Each sample consists of a so-called \emph{impression}, which in the
terminology of the information retrieval literature (e.g., see the
book by~\citet{ManningRaSc08}), is a list containing a user-issued
search, the advertisement presented to the user in response to the
search, and a label $y \in \{+1, -1\}$ indicating whether the user
clicked on the advertisement.  The ads in our dataset were presented
to 23,669,283 distinct users.

Transforming an impression into a useable set of regressors $x$ is
non-trivial, but the Tencent dataset provides a standard encoding. We
list the features present in the data in Table~\ref{tab:feature},
along with some description of their meaning. Each text-based
feature---that is, those made up of words, which are Query, Keyword,
and Title---is given a ``bag-of-words''
encoding~\citep{ManningRaSc08}. This encoding assigns each of 20,000
possible words an index, and if the word appears in the query (or
Keyword or Title feature), the corresponding index in the vector $x$
is set to 1. Words that do not appear are encoded with a zero.
Real-valued features, corresponding to the bottom fifteen features in
Table~\ref{tab:feature} beginning with ``Age'', are binned into a
fixed number of intervals $[-\infty, a_1], \openleft{a_1}{a_2},
\ldots, \openleft{a_k}{\infty}$, each of which is assigned an index in
$x$. (Note that the intervals and number thereof vary per feature, and
the dimension of the features listed in Table~\ref{tab:feature}
corresponds to the number of intervals). When a feature falls into a
particular bin, the corresponding entry of $x$ is assigned a 1, and
otherwise the entries of $x$ corresponding to the feature are 0. Each
feature has one additional value for ``unknown.'' The remaining
categorical features---gender, advertiser, and advertisement ID
(AdID)---are also given $\{0, 1\}$ encodings, where only one index of
$x$ corresponding to the feature may be non-zero (which indicates the
particular gender, advertiser, or AdID). This combination of encodings
yields a binary-valued covariate vector $x \in \{0, 1\}^d$ with $d =
741,\!725$ dimensions. Note also that the features incorporate
information about the user, advertisement, and query issued, encoding
information about their interactions into the model.

Our goal is to predict the probability of a user clicking a given
advertisement as a function of the covariates in
Table~\ref{tab:feature}.  To do so, we use a logistic
regression model to estimate the probability of a click response
\begin{equation*}
  P(y=1 \mid x ; \optvar) \defeq \frac{1}{1+\exp(-\<\optvar, x\>)},
\end{equation*}
where $\optvar \in \R^d$ is the unknown regression vector.  We use
the negative logarithm of $P$ as the loss, incorporating a ridge
regularization penalty.  This combination yields instantaneous loss
\begin{equation}
  \label{eqn:logistic-regression}
  f(\optvar; (x,y)) = \log\left(1+\exp(-y \<\optvar, x\>)\right) +
  \frac{\strongparam}{2} \ltwo{\optvar}^2.
\end{equation}
In all our experiments, we assume that the population negative log-likelihood
risk has local strong convexity as suggested by
Assumption~\ref{assumption:strong-convexity}.  In practice, we use a small
regularization parameter $\strongparam = 10^{-6}$ to ensure fast
convergence for the local sub-problems.

\begin{figure}[t]
  \begin{center}
  \begin{tabular}{cc}
    \psfrag{AVGM}{\small{\avgm}}
    \psfrag{BAVGM \(r=0.1\)}{\small{\savgm\ ($\ratio = .1$)}}
    \psfrag{BAVGM \(r=0.25\)}{\small{\savgm\ ($\ratio = .25$)}}
    \includegraphics[width=.45\columnwidth]{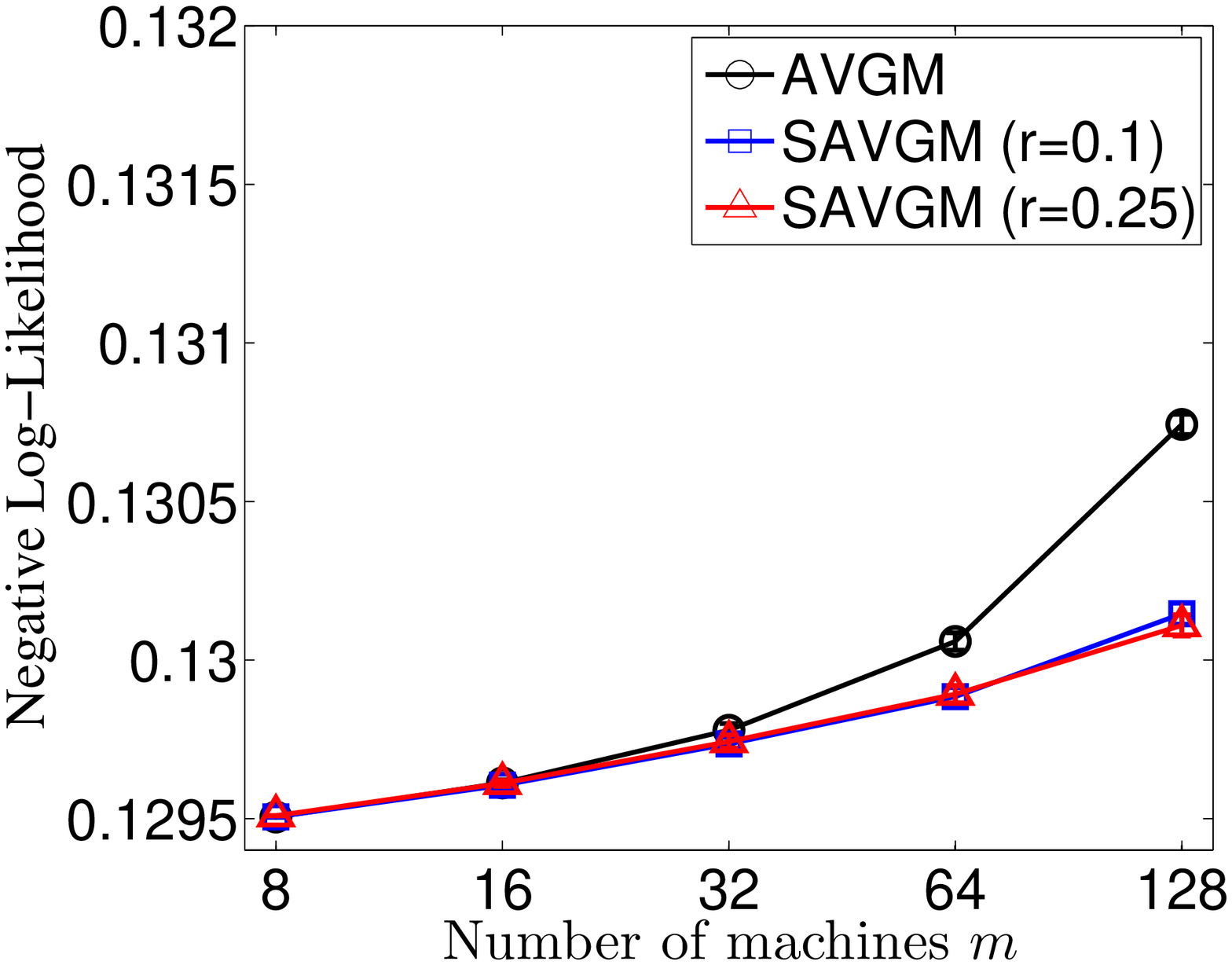} &
    \includegraphics[width=.45\columnwidth]{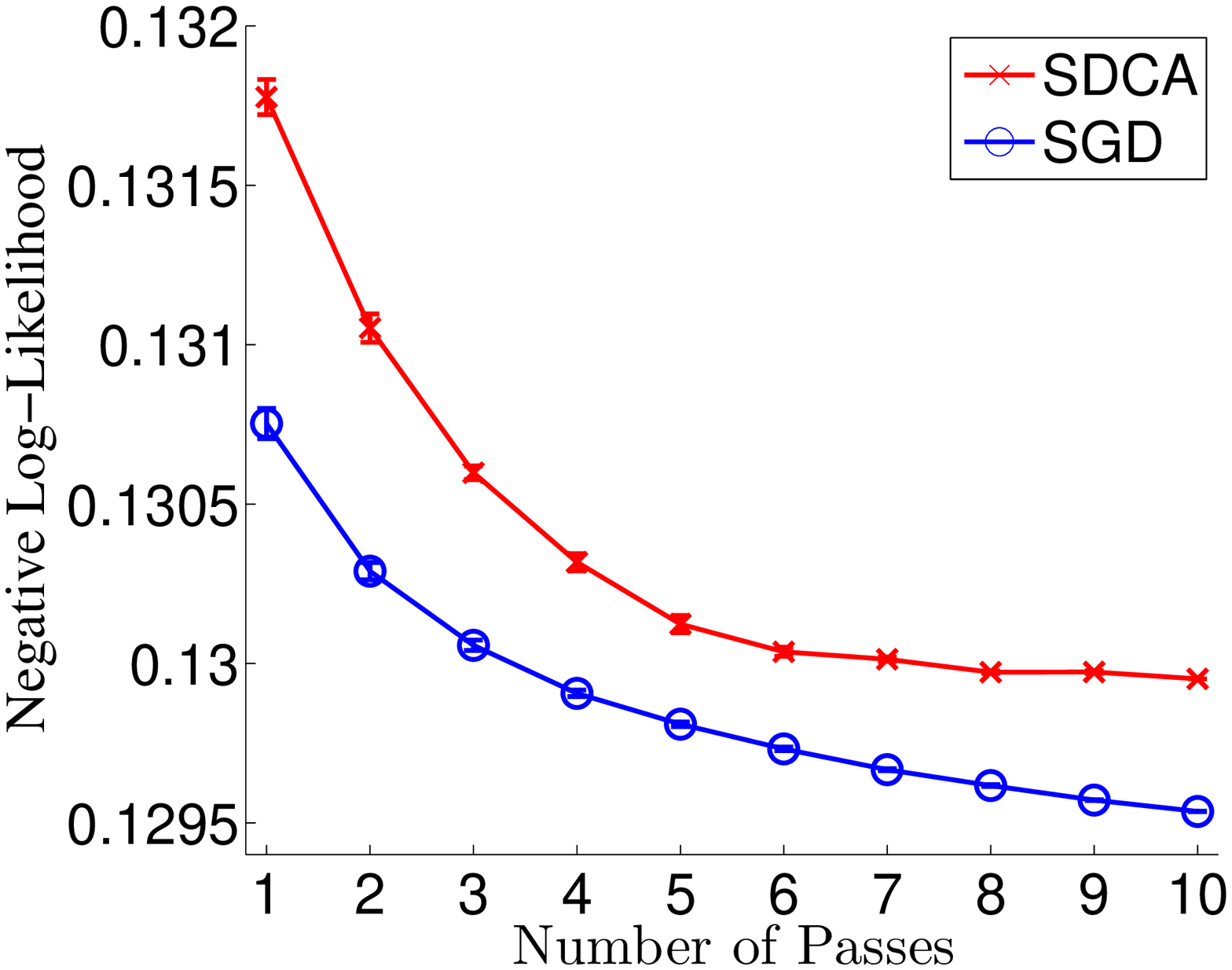} \\
    (a) & (b)\\
  \end{tabular}
  \caption{\label{fig:results} The negative log-likelihood of the output of
    the \avgm, \savgm, and stochastic methods on the held-out dataset for the
    click-through prediction task. (a)~Performance of the \avgm\ and
    \savgm\ methods versus the number of splits $\nummac$ of the
    data. (b)~Performance of SDCA and SGD baselines as a function of number of
    passes through the entire dataset.}
  \end{center}
  \vspace{-.5cm}
\end{figure}

For this problem, we cannot evaluate the mean-squared error
$\ltwos{\what{\optvar} - \optvar^*}^2$, as we do not know the true optimal
parameter $\optvar^*$.  Consequently, we evaluate the performance of an
estimate $\what{\optvar}$ using log-loss on a held-out dataset. Specifically,
we perform a five-fold validation experiment, where we shuffle the data and
partition it into five equal-sized subsets. For each of our five experiments,
we hold out one partition to use as the test set, using the remaining data as
the training set for inference. When studying the \avgm\ or
\savgm\ method, we compute the local estimate $\optvar_i$ via a trust-region
Newton-based method~\citep{NocedalWr06} implemented by LIBSVM~\citep{chang2011libsvm}.

The dataset is too large to fit in the memory of most computers: in total,
four splits of the data require 55~gigabytes. Consequently, it is difficult
to provide an oracle training comparison using the full $\totalnumobs$
samples.  Instead, for each experiment, we perform 10 passes of stochastic dual
coordinate ascent (SDCA)~\citep{shalev2012stochastic} and 10 passes of stochastic gradient descent (SGD) through the
dataset to get two rough baselines of the
performance attained by the empirical minimizer for the entire training
dataset. Figure~\ref{fig:results}(b) shows the hold-out set log-loss after
each of the sequential passes through the training data finishes. Note that
although the SDCA enjoys faster convergence rate on the regularized empirical risk~\citep{shalev2012stochastic},
the plot shows that the SGD has better generalization performance.

In Figure~\ref{fig:results}(a), we show the average hold-out set log-loss
(with standard errors) of the estimator $\optavg_1$ provided by the
\avgm\ method versus number of splits of the data $\nummac$, and we also plot
the log-loss of the \savgm\ method using subsampling ratios of $\ratio \in
\{.1, .25\}$. The plot shows that for small $\nummac$, both \avgm\ and
\savgm\ enjoy good performance, comparable to or better than (our proxy for)
the oracle solution using all $\totalnumobs$ samples. As the number of
machines $\nummac$ grows, however, the de-biasing provided by the subsampled
bootstrap method yields substantial improvements over the standard
\avgm\ method. In addition, even with $\nummac = 128$ splits of the dataset,
the \savgm\ method gives better hold-out set performance than performing two
passes of stochastic gradient on the entire dataset of $\nummac$ samples; with
$\nummac = 64$, \savgm\ enjoys performance as strong as looping through the
data four times with stochastic gradient descent.  This is striking, since
doing even one pass through the data with stochastic gradient descent gives
minimax optimal convergence rates~\citep{PolyakJu92, AgarwalBaRaWa12}.  In
ranking applications, rather than measuring negative log-likelihood, one may
wish to use a direct measure of prediction error; to that end,
Figure~\ref{fig:ads-auc} shows plots of the area-under-the-curve (AUC)
measure for the
\avgm\ and \savgm\ methods; AUC is a well-known measure of prediction error
for bipartite ranking~\citep{ManningRaSc08}. Broadly, this plot shows
a similar story to that in Figure~\ref{fig:results}.

\begin{figure}[t]
  \begin{center}
    \psfrag{AVGM}{\avgm}
    \psfrag{BAVGM \(r=0.1\)}{\savgm\ $(r = .1)$}
    \psfrag{BAVGM \(r=0.25\)}{\savgm\ $(r = .25)$}
    \includegraphics[width=.5\columnwidth]{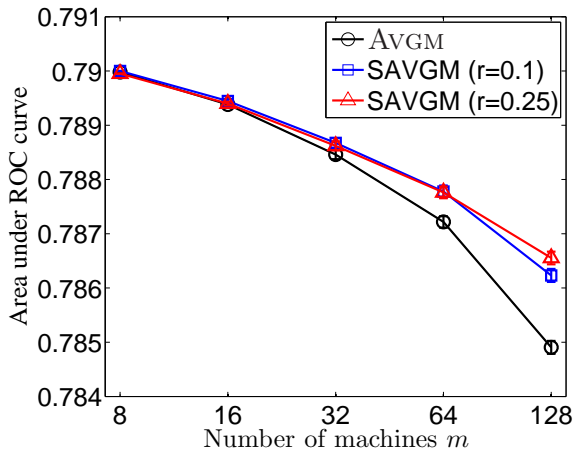}
    \caption{\label{fig:ads-auc} The area-under-the-curve (AUC) measure of
      ranking error for the output of the \avgm\ and \savgm\ methods for the
      click-through prediction task.}
  \end{center}
\end{figure}

\begin{figure}[t]
  \begin{center}
    \psfrag{BAVGM \(m=128\)}{\small{\savgm\ $(\nummac=128)$}}
    \includegraphics[width=.45\columnwidth]{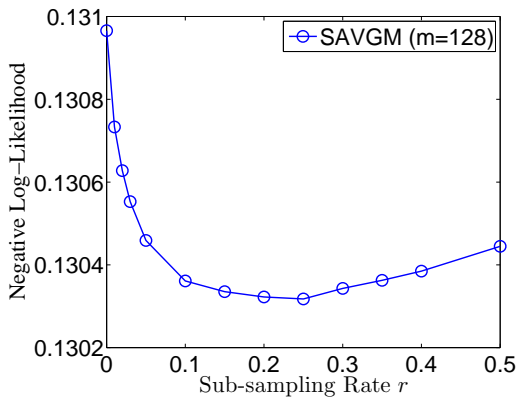}
    \caption{\label{fig:compare_r}
      The log-loss on held-out data for the \savgm\ method applied with
      $\nummac = 128$ parallel splits of the data, plotted versus the
      sub-sampling rate $\ratio$.}
  \end{center}
  \vspace{-.5cm}
\end{figure}

It is instructive and important to understand the sensitivity of the
\savgm\ method to the value of the
resampling parameter $\ratio$. We explore this question
in Figure~\ref{fig:compare_r} using $\nummac = 128$ splits, where we plot
the log-loss of the \savgm\ estimator on the held-out data set versus the
subsampling ratio $\ratio$. We choose $\nummac = 128$ because more data splits
provide more variable performance in $\ratio$. For the {\tt soso.com} ad
prediction data set, the choice $\ratio = .25$ achieves the best performance,
but Figure~\ref{fig:compare_r} suggests that mis-specifying the ratio is not
terribly detrimental. Indeed, while the performance of $\savgm$ degrades
to that of the $\avgm$ method, a wide range of settings
of $\ratio$ give
improved performance, and there does not appear to be a phase transition
to poor performance.


\section{Discussion}
\label{sec:discussion}

Large scale statistical inference problems are challenging, and the
difficulty of solving them will only grow as data becomes more
abundant: the amount of data we collect is growing much faster than
the speed or storage capabilities of our computers. Our \avgm, \savgm,
and \sgdavgm\ methods provide strategies for efficiently solving such
large-scale risk minimization problems, enjoying performance
comparable to an oracle method that is able to access the entire large
dataset. We believe there are several interesting questions that
remain open after this work. First, nonparametric estimation problems,
which often suffer superlinear scaling in the size of the data, may
provide an interesting avenue for further study of decomposition-based
methods.  Our own recent work has addressed aspects of this challenge
in the context of kernel methods for non-parametric
regression~\citep{ZhaDucWai_COLT13}.  More generally, an understanding
of the interplay between statistical efficiency and communication
could provide an avenue for further research, and it may also be
interesting to study the effects of subsampled or bootstrap-based
estimators in other distributed environments.

\subsection*{Acknowledgments}

We thank Joel Tropp for some informative discussions on and references for
matrix concentration and moment inequalities. We also thank Ohad Shamir for
pointing out a mistake in the statements of results related to Theorem 1, and
the editor and reviewers for their helpful comments and feedback.  JCD was
supported by the Department of Defence under the NDSEG Fellowship Program and
by a Facebook PhD fellowship.  This work was partially funded by Office of
Naval Research MURI grant N00014-11-1-0688 to MJW.


\appendix

\section{The necessity of smoothness}
\label{appendix:smoothness-necessary}
\newcommand{\floor}[1]{\left\lfloor{#1}\right\rfloor}

Here we show that some version of the smoothness conditions presented
in Assumption~\ref{assumption:smoothness} are necessary for averaging
methods to attain better mean-squared error than using only the
$\numobs$ samples on a single processor.  Given the loss
function~\eqref{eqn:smoothness-loss}, let $\numobs_0 =
\sum_{i=1}^\numobs \indic{\statrv_i = 0}$ to be the count of 0
samples. Using $\optvar_1$ as shorthand for $\optvar_{1,i}$, we see by
inspection that the empirical minimizer $\optvar_1$ is
\begin{equation*}
  \optvar_1 = 
  \begin{cases}
  \frac{\numobs_0}{\numobs} - \half & \mbox{when} ~ \numobs_0 \le \numobs / 2
  \\
  1 - \frac{\numobs}{2 \numobs_0} & \mbox{otherwise.}
  \end{cases}
\end{equation*}
For simplicity, we may assume that $\numobs$ is odd. In this case, we
obtain that
\begin{align*}
  \E[\optvar_1] & = \frac{1}{4} + \E\left[
    \frac{\numobs_0}{\numobs} \indic{\numobs_0 < \numobs/2} \right] -
  \E\left[\frac{\numobs}{2 \numobs_0} \indic{\numobs_0 > \numobs / 2}\right] \\
  & = \frac{1}{4} + \frac{1}{2^\numobs}
  \sum_{i=0}^{\floor{\numobs/2}}
  \binom{\numobs}{i} \frac{i}{n}
  - \frac{1}{2^\numobs}
  \sum_{i = \ceil{\numobs/2}}^{\numobs}
  \binom{\numobs}{i} \frac{n}{2i}
  = \frac{1}{4} + \frac{1}{2^\numobs}
  \sum_{i=0}^{\floor{\numobs/2}}
  \binom{\numobs}{i} \left[\frac{i}{n}
    - \frac{\numobs}{2(\numobs - i)}\right]
\end{align*}
by the symmetry of the binomial. Adding and subtracting $\half$ from the term
within the braces, noting that $\statprob(\numobs_0 < \numobs/2) = 1/2$, we
have the equality
\begin{equation*}
  \E[\optvar_1]
  = \frac{1}{2^\numobs} \sum_{i=0}^{\floor{\numobs/2}}
  \binom{\numobs}{i} \left[\frac{i}{\numobs}
    - \frac{\numobs}{2(\numobs - i)}
    + \half\right]
  = \frac{1}{2^\numobs}
  \sum_{i=0}^{\floor{\numobs/2}}
  \binom{\numobs}{i} \frac{i (\numobs - 2i)}{2 \numobs(\numobs - i)}.
\end{equation*}
If $Z$ is distributed normally with mean $1/2$ and variance
$1/(4\numobs)$, then an asymptotic expansion of the binomial
distribution yields
\begin{align*}
  \left(\half\right)^\numobs
  \sum_{i=0}^{\floor{\numobs/2}}
  \binom{\numobs}{i} \frac{i (\numobs - 2i)}{2 \numobs(\numobs - i)}
  & = \E\left[\frac{Z(1 - 2 Z)}{2 - 2 Z} \mid 0 \le Z \le \half\right]
  + o(\numobs^{-1/2}) \\
  & \ge \half \E\left[Z - 2Z^2 \mid 0 \le Z \le \half\right]
  + o(\numobs^{-1/2})
  = \Omega(\numobs^{-\half}),
\end{align*}
the final equality following from standard calculations, since $\E[|Z|]
= \Omega(n^{-1/2})$.


\section{Proof of Theorem~\ref{theorem:no-bootstrap}}

Although Theorem~\ref{theorem:no-bootstrap} is in terms of bounds on
$8^{th}$ order moments, we prove a somewhat more general result in
terms of a set of $(\objmoment, \gradmoment, \hessmoment)$ moment
conditions given by
\begin{align*}
  & \E[\ltwo{\nabla f(\optvar; \statrv)}^\objmoment] \leq
  \lipobj^\objmoment, \qquad ~~~ \E[\matrixnorm{\nabla^2 f(\optvar;
      \statrv) - \nabla^2 \popfun(\optvar)}_2^\gradmoment] \leq
  \lipgrad^\gradmoment, \\ & \E[\liphessian(\statrv)^\hessmoment] \le
  \liphessian^\hessmoment \quad \mbox{and} \quad
  \E[(\liphessian(\statrv) - \E[\liphessian(\statrv)])^\hessmoment]
  \le \liphessian^\hessmoment
\end{align*}
for $\optvar \in U$.  (Recall the definition of $U$ prior to
Assumption~\ref{assumption:smoothness}).  Doing so allows sharper
control if higher moment bounds are available. The reader should
recall throughout our arguments that we have assumed
$\min\{\objmoment, \gradmoment, \hessmoment\} \ge 8$. Throughout the
proof, we use $F_1$ and $\optvar_1$ to indicate the local empirical
objective and empirical minimizer of machine $1$ (which have the same
distribution as those of the other processors), and we recall the
notation $\indic{\event}$ for the indicator function of the event
$\event$.

Before beginning the proof of Theorem~\ref{theorem:no-bootstrap}
proper, we begin with a simple inequality that relates the error term
$\optavg - \optvar^*$ to an average of the errors $\optvar_i -
\optvar^*$, each of which we can bound in turn. Specifically, a bit of
algebra gives us that
\begin{align}
  \E[\ltwo{\optavg - \optvar^*}^2] & = \E\bigg[\ltwobigg{\frac{1}{m}
      \sum_{i=1}^m \optvar_i - \optvar^*}^2\bigg] \nonumber \\
 & = \frac{1}{m^2} \sum_{i=1}^m \E[\ltwo{\optvar_i - \optvar^*}^2] +
  \frac{1}{m^2} \sum_{i \neq j} \E[\<\optvar_i - \optvar^*, \optvar_j
    - \optvar^*\>] \nonumber \\
 & \le \frac{1}{m} \E[\ltwo{\optvar_1 - \optvar^*}^2] + \frac{m(m -
    1)}{m^2} \ltwo{\E[\optvar_1 - \optvar^*]}^2 \nonumber \\
  & \le \frac{1}{m} \E[\ltwo{\optvar_1 - \optvar^*}^2] +
  \ltwo{\E[\optvar_1 - \optvar^*]}^2.
  \label{eqn:error-inequality}
\end{align}
Here we used the definition of the averaged vector $\optavg$ and the
fact that for $i \neq j$, the vectors $\optvar_i$ and $\optvar_j$ are
statistically independent, they are functions of independent
samples. The upper bound~\eqref{eqn:error-inequality} illuminates the
path for the remainder of our proof: we bound each of
$\E[\ltwo{\optvar_i - \optvar^*}^2]$ and $\ltwo{\E[\optvar_i -
    \optvar^*]}^2$. Intuitively, since our objective is locally
strongly convex by Assumption~\ref{assumption:strong-convexity}, the
empirical minimizing vector $\optvar_1$ is a nearly unbiased estimator
for $\optvar^*$, which allows us to prove the convergence rates in the
theorem.

We begin by defining three events---which we (later) show hold with
high probability---that guarantee the closeness of $\optvar_1$ and
$\optvar^*$. In rough terms, when these events hold, the function
$F_1$ behaves similarly to the population risk $F_0$ around the point
$\optvar^*$; since $F_0$ is locally strongly convex, the minimizer
$\optvar_1$ of $F_1$ will be close to $\optvar^*$.  Recall that
Assumption~\ref{assumption:smoothness} guarantees the existence of a
ball $U_\rho = \{\optvar \in \R^d : \ltwo{\optvar - \optvar^*} <
\rho\}$ of radius $\rho \in (0, 1)$ such that
\begin{equation*}
  \matrixnorm{\nabla^2 f(\optvar; \statsample)
    - \nabla^2 f(\optvar'; \statsample)}_2
  \le \liphessian(\statsample) \ltwo{\optvar - \optvar'}
\end{equation*}
for all $\optvar, \optvar' \in U_\rho$ and any $\statsample$, where
$\E[\liphessian(\statrv)^\hessmoment] \le \liphessian^\hessmoment$. In
addition, Assumption~\ref{assumption:strong-convexity} guarantees that
$\nabla^2 F_0(\optvar^*) \succeq \strongparam I$. Now, choosing the
potentially smaller radius $\delta_\rho = \min\{\rho, \rho
\strongparam / 4 \liphessian\}$, we can define the three ``good''
events
\begin{align}
  \event_0 & \defeq \bigg\{ \frac{1}{\numobs} \sum_{i=1}^\numobs
  \liphessian(\statrv_i) \le 2 \liphessian \bigg\}, \nonumber
  \\ \event_1 & \defeq \left\{\matrixnorm{\nabla^2 F_1(\optvar^*) -
    \nabla^2 F_0(\optvar^*)}_2 \le \frac{\rho \strongparam}{2}
  \right\}, \quad \mbox{and}
  \label{eqn:good-events} \\
  \event_2 & \defeq \left\{\ltwo{\nabla F_1(\optvar^*)} \le
  \frac{(1 - \rho) \strongparam \delta_\rho}{2}
  \right\}. \nonumber
\end{align}
We then have the following lemma:
\begin{lemma}
  \label{lemma:good-events}
  Under the events $\event_0$, $\event_1$, and $\event_2$ previously
  defined~\eqref{eqn:good-events}, we have
  \begin{equation*}
    \ltwo{\optvar_1 - \optvar^*} \le \frac{2 \ltwo{\nabla
        F_1(\optvar^*)}}{ (1 - \rho) \strongparam}, \quad \mbox{and}
    \quad \nabla^2 F_1(\optvar) \succeq (1 - \rho) \strongparam I_{d
      \times d}.
  \end{equation*}
\end{lemma}
\noindent
The proof of Lemma~\ref{lemma:good-events} relies on some standard
optimization guarantees relating gradients to minimizers of functions
(e.g.~\cite{BoydVa04}, Chapter 9), although some care is required
since smoothness and strong convexity hold only locally in our
problem. As the argument is somewhat technical, we defer it to
Appendix~\ref{appendix:proof-of-lemma-good-events}.

Our approach from here is to give bounds on $\E[\ltwo{\optvar_1 -
    \optvar^*}^2]$ and $\ltwo{\E[\optvar_1 - \optvar^*]}^2$ by careful
Taylor expansions, which allows us to bound $\E[\ltwo{\optavg_1 -
    \optvar^*}^2]$ via our initial
expansion~\eqref{eqn:error-inequality}.  We begin by noting that
whenever the events $\event_0$, $\event_1$, and $\event_2$ hold, then
$\nabla F_1(\optvar_1) = 0$, and moreover, by a Taylor series
expansion of $\nabla F_1$ between $\optvar^*$ and $\optvar_1$, we have
\begin{align*}
  0 = \nabla F_1(\optvar_1)
  & = \nabla F_1(\optvar^*) + \nabla^2 F_1(\optvar')(\optvar_1 - \optvar^*)
\end{align*}
where $\optvar' = \interp \optvar^* + (1 - \interp) \optvar_1$ for
some $\interp \in [0, 1]$.  By adding and subtracting terms,
we have
\begin{multline}
 \label{eqn:taylor-error-expansion}
0  = \nabla F_1(\optvar^*) + (\nabla^2 F_1(\optvar') - \nabla^2
F_1(\optvar^*)) (\optvar_1 - \optvar^*) \\
+ (\nabla^2 F_1(\optvar^*) -
\nabla^2 F_0(\optvar^*))(\optvar_1 - \optvar^*) + \nabla^2
F_0(\optvar^*)(\optvar_1 - \optvar^*).
\end{multline}
Since $\nabla^2 F_0(\optvar^*) \succeq \strongparam I$, we can define
the inverse Hessian matrix $\invhessian \defeq [\nabla^2
  F_0(\optvar^*)]^{-1}$, and setting $\optvarerr \defeq \optvar_1 -
\optvar^*$, we multiply both sides of the Taylor
expansion~\eqref{eqn:taylor-error-expansion} by $\invhessian$ to
obtain the relation
\begin{equation}
  \optvarerr = -\invhessian \nabla F_1(\optvar^*) +
  \invhessian(\nabla^2 F_1(\optvar^*) - \nabla^2 F_1(\optvar'))
  \optvarerr + \invhessian(\nabla^2 F_0(\optvar^*) - \nabla^2
  F_1(\optvar^*)) \optvarerr.
  \label{eqn:error-expansion}
\end{equation}
Thus, if we define the matrices $P = \nabla^2 F_0(\optvar^*) -
\nabla^2 F_1(\optvar^*)$ and
$Q = \nabla^2 F_1(\optvar^*) - \nabla^2
F_1(\optvar')$, equality~\eqref{eqn:error-expansion} can be re-written
as
\begin{equation}
  \label{eqn:simple-error-expansion}
  \optvar_1 - \optvar^* =
  -\invhessian \nabla F_1(\optvar^*)
  + \invhessian(P + Q)(\optvar_1 - \optvar^*).
\end{equation}
Note that equation~\eqref{eqn:simple-error-expansion} holds when the
conditions of Lemma~\ref{lemma:good-events} hold, and otherwise we may
simply assert only that $\ltwo{\optvar_1 - \optvar^*} \le \radius$.
Roughly, we expect the final two terms in the error
expansion~\eqref{eqn:simple-error-expansion} to be of smaller order
than the first term, since we hope that $\optvar_1 - \optvar^*
\rightarrow 0$ and additionally that the Hessian differences decrease
to zero at a sufficiently fast rate. We now formalize this intuition.

Inspecting the Taylor expansion~\eqref{eqn:simple-error-expansion}, we
see that there are several terms of a form similar to $(\nabla^2
F_0(\optvar^*) - \nabla^2 F_1(\optvar^*))(\optvar_1 - \optvar^*)$;
using the smoothness Assumption~\ref{assumption:smoothness}, we can
convert these terms into higher order terms involving only $\optvar_1
- \optvar^*$. Thus, to effectively control the
expansions~\eqref{eqn:error-expansion}
and~\eqref{eqn:simple-error-expansion}, we must show that higher order
terms of the form $\E[\ltwo{\optvar_1 - \optvar^*}^k]$, for $k \geq
2$, decrease quickly enough in $\numobs$.

\paragraph{Control of $\E[\ltwo{\optvar_1 - \optvar^*}^k]$:}
Recalling the events~\eqref{eqn:good-events}, we define $\event \defeq
\event_0 \cap \event_1 \cap \event_2$ and then observe that
\begin{align*}
  \E[\ltwo{\optvar_1 - \optvar^*}^k]
  & = \E[\indic{\event} \ltwo{\optvar_1 - \optvar^*}^k]
  + \E[\indic{\event^c} \ltwo{\optvar_1 - \optvar^*}^k] \\
  & \le \frac{2^k \E[\indic{\event} \ltwo{\nabla F_1(\optvar^*)}^k]}{
    (1 - \rho)^k \strongparam^k}
  + \P(\event^c) \radius^k \\
  & \le \frac{2^k \E[\ltwo{\nabla F_1(\optvar^*)}^k]}{(1 - \rho)^k
    \strongparam^k} + \P(\event^c) \radius^k,
\end{align*}
where we have used the bound $\ltwo{\optvar - \optvar^*} \le \radius$
for all $\optvar \in \optvarspace$, from
Assumption~\ref{assumption:parameter-space}.  Our goal is to prove
that $\E[\ltwo{\nabla F_1(\optvar^*)}^k] = \order(\numobs^{-k/2})$ and
that $\P(\event^c) = \order(\numobs^{-k/2})$.  We move forward with a
two lemmas that lay the groundwork for proving these two facts:
\begin{lemma}
  \label{lemma:moment-consequences}
  Under Assumption~\ref{assumption:smoothness}, there exist constants $C$
  and $C'$ (dependent only on the moments $\objmoment$ and $\gradmoment$
  respectively) such that
  \begin{subequations}
    \begin{align}
      \E[\ltwo{\nabla F_1(\optvar^*)}^\objmoment] & \le
      C \frac{\lipobj^\objmoment}{n^{\objmoment / 2}},
      ~~~ \mbox{and}
      \label{eqn:grad-bound} \\
      \E[\matrixnorm{\nabla^2 F_1(\optvar^*)
          - \nabla^2 F_0(\optvar^*)}_2^\gradmoment]
      & \le C' \frac{\log^{\gradmoment / 2}(2d) \lipgrad^\gradmoment}{
        n^{\gradmoment/2}}.
      \label{eqn:hessian-bound}
    \end{align}
  \end{subequations}
\end{lemma}
\noindent

\noindent See Appendix~\ref{appendix:moment-bounds} for the proof of
this claim.

As an immediate consequence of Lemma~\ref{lemma:moment-consequences}, we
see that the events $\event_1$ and $\event_2$ defined
by~\eqref{eqn:good-events} occur with reasonably high probability. Indeed,
recalling that $\event = \event_0 \cap \event_1 \cap \event_2$,
Boole's law and the union bound imply
\begin{align}
  \lefteqn{\P(\event^c)  = \P(\event_0^c \cup \event_1^c \cup \event_2^c)}
\nonumber \\
& \le \P(\event_0^c) +\P(\event_1^c) + \P(\event_2^c) \nonumber \\
& \le \frac{ \E[|\frac{1}{\numobs}
    \sum_{i=1}^\numobs\liphessian(\statrv_i) -
    \E[\liphessian(\statrv)]|^\hessmoment]}{\liphessian^\hessmoment} +
\frac{2^\gradmoment \E[\matrixnorm{\nabla^2 F_1(\optvar^*) - \nabla^2
      F_0(\optvar^*)}_2^\gradmoment]}{ \rho^\gradmoment
  \strongparam^\gradmoment} + \frac{2^\objmoment \E[\ltwo{\nabla
      F_1(\optvar^*)}^\objmoment]}{ (1 - \rho)^\objmoment
  \strongparam^\objmoment \delta_\rho^\objmoment} \nonumber \\ & \le
C_2 \frac{1}{\numobs^{\hessmoment/2}} + C_1
\frac{\log^{\gradmoment/2}(2d) \lipgrad^\gradmoment}{
  \numobs^{\gradmoment/2}} + C_0
\frac{\lipobj^\objmoment}{\numobs^{\objmoment/2}}
  \label{eqn:no-bad-events}
\end{align}
for some universal constants $C_0, C_1, C_2$, where in the
second-to-last line we have invoked the moment bound in
Assumption~\ref{assumption:smoothness}.  Consequently, we find that
\begin{align*}
\P(\event^c) \radius^k = \order(\radius^k (\numobs^{-\gradmoment/2} +
\numobs^{-\hessmoment/2} + \numobs^{-\objmoment/2}) \quad \mbox{for
  any $k \in \N$.}
\end{align*}
In summary, we have proved the following lemma:
\begin{lemma}
  \label{lemma:higher-order-moments}
  Let Assumptions~\ref{assumption:strong-convexity}
  and~\ref{assumption:smoothness} hold.  For any $k \in \N$
  with $k \le \min\{\objmoment, \gradmoment, \hessmoment\}$, we have
  \begin{equation*}
    \E[\ltwo{\optvar_1 - \optvar^*}^k]
    = \order\left(\numobs^{-k/2} \cdot \frac{
      \lipobj^{k}}{(1 - \rho)^k
    \strongparam^k}  + \numobs^{-\objmoment/2}
    + \numobs^{-\gradmoment/2} + \numobs^{-\hessmoment/2} \right)
    = \order\left(\numobs^{-k/2}\right),
  \end{equation*}
  where the order statements hold as $\numobs \rightarrow +\infty$.
\end{lemma}

\noindent Now recall the matrix $Q = \nabla^2 F_1(\optvar^*) -
\nabla^2 F_1(\optvar')$ defined following
equation~\eqref{eqn:error-expansion}.  The following result controls
the moments of its operator norm: \\
\begin{lemma}
  \label{lemma:higher-order-matrix-moments}
For $k \le \min\{\hessmoment, \gradmoment, \objmoment\} / 2$, we have
$\E[\matrixnorm{Q}_2^k] = \order(\numobs^{-k/2})$.
\end{lemma}
\begin{proof}
  We begin by using Jensen's inequality and
  Assumption~\ref{assumption:smoothness} to see that
  \begin{equation*}
    \matrixnorm{Q}^k
    \le \frac{1}{\numobs} \sum_{i=1}^\numobs
    \matrixnorm{\nabla^2 f(\optvar'; \statrv_i) - \nabla^2
      f(\optvar^*; \statrv_i)}^k
    \le
    \frac{1}{\numobs} \sum_{i=1}^\numobs \liphessian(\statrv_i)^k
    \ltwo{\optvar' - \optvar^*}^k.
  \end{equation*}
  Now we apply the Cauchy-Schwarz inequality and
  Lemma~\ref{lemma:higher-order-moments}, thereby obtaining
  \begin{equation*}
    \E[\matrixnorm{Q}_2^k]
    \le \E\left[
      \bigg(\frac{1}{\numobs} \sum_{i=1}^\numobs \liphessian(X_i)^k\bigg)^2
      \right]^\half
    \E\left[\ltwo{\optvar_1 - \optvar^*}^{2k}\right]^\half
    = \order\left( \liphessian^k \frac{\lipobj^k}{(1 - \rho)^k
      \strongparam^k} \numobs^{-k/2}\right),
  \end{equation*}
  where we have used Assumption~\ref{assumption:smoothness} again.
\end{proof}

Lemma~\ref{lemma:higher-order-moments} allows us to control the first term
from our initial bound~\eqref{eqn:error-inequality} almost
immediately. Indeed, using our last Taylor
expansion~\eqref{eqn:simple-error-expansion}
and the definition of
the event $\event = \event_0 \cap \event_1 \cap \event_2$, we have
\begin{align*}
  \E[\ltwo{\optvar_1 - \optvar^*}^2]
  & = \E\left[\indic{\event}\ltwo{-\invhessian \nabla F_1(\optvar^*)
    + \invhessian(P + Q)(\optvar_1 - \optvar^*)}^2\right]
  + \E[\indic{\event^c} \ltwo{\optvar_1 - \optvar^*}^2] \\
  & \le 2 \E\left[\ltwo{\invhessian \nabla F_1(\optvar^*)}^2\right]
  + 2\E\left[\ltwo{\invhessian(P + Q)(\optvar_1 - \optvar^*)}^2\right]
  + \P(\event^c) \radius^2,
\end{align*}
where we have applied the inequality $(a + b)^2 \le 2 a^2 + 2 b^2$. Again
using this same inequality, then applying Cauchy-Schwarz and
Lemmas~\ref{lemma:higher-order-moments}
and~\ref{lemma:higher-order-matrix-moments}, we see that
\begin{align*}
  \lefteqn{\E\left[\ltwo{\invhessian(P + Q)(\optvar_1 - \optvar^*)}^2\right]
    \le
    2\matrixnorm{\invhessian}_2^2
    \left(\E[\matrixnorm{P}_2^2 \ltwo{\optvar_1 - \optvar^*}^2]
    + \E[\matrixnorm{Q}_2^2 \ltwo{\optvar_1 - \optvar^*}^2]\right)} \\
  & \qquad\qquad\qquad\qquad\quad ~ \le 2 \matrixnorm{\invhessian}_2^2
  \left(\sqrt{\E[\matrixnorm{P}_2^4]\E[\ltwo{\optvar_1 - \optvar^*}^4]}
  + \sqrt{\E[\matrixnorm{Q}_2^4]\E[\ltwo{\optvar_1 - \optvar^*}^4]}\right) \\
  & \qquad\qquad\qquad\qquad\quad ~ = \order(\numobs^{-2}),
\end{align*}
where we have used the fact that $\min\{\objmoment, \gradmoment, \hessmoment\}
\ge 8$ to apply Lemma~\ref{lemma:higher-order-matrix-moments}.
Combining these results, we obtain the upper bound
\begin{equation}
  \label{eqn:expected-norm-bound}
  \E[\ltwo{\optvar_1 - \optvar^*}^2]
  \le 2 \E\left[\ltwo{\invhessian \nabla F_1(\optvar^*)}^2\right]
  + \order(\numobs^{-2}),
\end{equation}
which completes the first part of our proof of
Theorem~\ref{theorem:no-bootstrap}.

\paragraph{Control of $\ltwo{\E[\optvar_1 - \optvar^*]}^2$:}
It remains to consider the $\ltwo{\E[\optvar_1 - \optvar^*]}^2$ term
from our initial error inequality~\eqref{eqn:error-inequality}. When
the events~\eqref{eqn:good-events} occur, we know that all derivatives
exist, so we may recursively apply our
expansion~\eqref{eqn:simple-error-expansion} of $\optvar_1 -
\optvar^*$ to find that
\begin{align}
  \optvar_1 - \optvar^*
  & = -\invhessian\nabla F_1(\optvar^*)
  + \invhessian(P + Q)(\optvar_1 - \optvar^*) \nonumber \\
  & = \underbrace{-\invhessian\nabla F_1(\optvar^*)
  + \invhessian(P + Q)\left[-\invhessian \nabla F_1(\optvar^*)
  + \invhessian(P + Q)(\optvar_1 - \optvar^*)\right]}_{\eqdef v}
  \label{eqn:recursive-equality}
\end{align}
where we have introduced $v$ as shorthand for the vector on the right hand side.
Thus, with a bit of algebraic manipulation we
obtain the relation
\begin{equation}
  \optvar_1 - \optvar^*
  = \indic{\event} v + \indic{\event^c} (\optvar_1 - \optvar^*)
  = v + \indic{\event^c}(\optvar_1 - \optvar^*)
  - \indic{\event^c}v
  = v + \indic{\event^c}(\optvar_1 - \optvar^* - v).
  \label{eqn:indicator-splitting-for-expectation}
\end{equation}
Now note that $\E[\nabla F_1(\optvar^*)]=0$ thus
\begin{align*}
  \E[v] & = \E\left[-\invhessian \nabla F_1(\optvar^*)
      + \invhessian(P + Q)[-\invhessian \nabla F_1(\optvar^*)
        + \invhessian(P + Q)(\optvar_1 - \optvar^*)]\right] \\
  & = \E\left[\invhessian(P + Q)\invhessian\left[(P + Q)(\optvar_1 - \optvar^*)
    - \nabla F_1(\optvar^*)\right]\right].
\end{align*}
Thus, by re-substituting the appropriate quantities
in~\eqref{eqn:indicator-splitting-for-expectation} and applying the
triangle inequality, we have
\begin{align}
  \lefteqn{\ltwo{\E[\optvar_1 - \optvar^*]}} \nonumber \\
  & \le \ltwo{\E[\invhessian(P + Q)\invhessian\left((P + Q)(\optvar_1 - \optvar^*)
      -\nabla F_1(\optvar^*)\right)]}
  + \ltwo{\E[\indic{\event^c}(\optvar_1 - \optvar^* - v)]} \nonumber \\
  & \le \ltwo{\E[\invhessian(P + Q)\invhessian\left((P + Q)(\optvar_1 - \optvar^*)
      -\nabla F_1(\optvar^*)\right)]}
  + \E[\indic{\event^c} \ltwo{\optvar_1 - \optvar^*}] \nonumber \\
  & \quad ~
  + \E\left[\indic{\event^c} \ltwo{-\invhessian\nabla F_1(\optvar^*)
  + \invhessian(P + Q)\invhessian\left[-\nabla F_1(\optvar^*)
  + (P + Q)(\optvar_1 - \optvar^*)\right]}\right].
  \label{eqn:unbiasing-norms}
\end{align}

Since $\ltwo{\optvar_1 - \optvar^*} \le \radius$ by assumption, we have
\begin{equation*}
  \E[\indic{\event^c}\ltwo{\optvar_1 - \optvar^*}]
  \le \P(\event^c) \radius
  \stackrel{(i)}{=} \order(\radius \numobs^{-k/2})
\end{equation*}
for any $k \le \min\{\hessmoment, \gradmoment, \objmoment\}$, where
step (i) follows from
the inequality~\eqref{eqn:no-bad-events}.
H\"older's inequality also yields that
\begin{align*}
  \lefteqn{\E\left[\indic{\event^c}
      \ltwo{\invhessian(P + Q) \invhessian \nabla F_1(\optvar^*)}\right]
    \le \E\left[\indic{\event^c}
      \matrixnorm{\invhessian(P + Q)}_2 \ltwo{\invhessian
        \nabla F_1(\optvar^*)}\right]} \\
  & \qquad\qquad\qquad\qquad\qquad \qquad\quad \le \sqrt{\P(\event^c)}
  \E\left[\matrixnorm{\invhessian(P + Q)}_2^4\right]^{1/4}
  \E\left[\ltwo{\invhessian \nabla F_1(\optvar^*)}^4\right]^{1/4}.
\end{align*}
Recalling Lemmas~\ref{lemma:moment-consequences}
and~\ref{lemma:higher-order-matrix-moments}, we have
$\E[\matrixnorm{\invhessian(P + Q)}_2^4] = \order(\log^2(d) \numobs^{-2})$,
and we similarly have $\E[\ltwo{\invhessian \nabla F_1(\optvar^*)}^4]
= \order(n^{-2})$. Lastly, we have $\P(\event^c) = \order(n^{-k/2})$ for
$k \le \min\{\objmoment, \gradmoment, \hessmoment\}$, whence we find that
for any such $k$,
\begin{equation*}
  \E\left[\indic{\event^c}
    \ltwo{\invhessian(P + Q) \invhessian \nabla F_1(\optvar^*)}\right]
  = \order\left(\sqrt{\log(d)} \numobs^{-k/4 - 1}\right).
\end{equation*}
We can similarly
apply Lemma~\ref{lemma:higher-order-moments} to the last
remaining term in the inequality~\eqref{eqn:unbiasing-norms} to obtain that
for any $k \le \min\{\hessmoment, \gradmoment, \objmoment\}$,
\begin{align*}
  \E\left[\indic{\event^c} \ltwo{-\invhessian\nabla F_1(\optvar^*)
      + \invhessian(P + Q)\left[-\invhessian \nabla F_1(\optvar^*)
        + \invhessian(P + Q)(\optvar_1 - \optvar^*)\right]}\right]
  ~~~~~ \\
  = \order(\numobs^{-k/2} + \numobs^{-k/4 - 1}).
\end{align*}
Applying these two bounds, we find that
\begin{equation}
  \label{eqn:intermediate-norm-expected-bound}
  \ltwo{\E[\optvar_1 - \optvar^*]}
  \le \ltwo{\E\left[\invhessian(P + Q)\invhessian\left(
      (P + Q)(\optvar_1 - \optvar^*) - \nabla F_1(\optvar^*)\right)\right]}
  + \order(\numobs^{-k})
\end{equation}
for any $k$ such that $k \le \min\{\objmoment, \gradmoment, \hessmoment\}/2$
and $k \le \min\{\objmoment, \gradmoment, \hessmoment\}/4 + 1$.

In the remainder of the proof, we show that part of the
bound~\eqref{eqn:intermediate-norm-expected-bound} still consists only of
higher-order terms, leaving us with an expression not involving $\optvar_1 -
\optvar^*$. To that end, note that
\begin{equation*}
  \E\left[
    \ltwo{\invhessian(P + Q)\invhessian(P + Q)(\optvar_1 - \optvar^*)}^2\right]
  = \order(\numobs^{-3})
\end{equation*}
by three applications of H\"older's inequality, the fact that $\ltwo{Ax} \le
\matrixnorm{A}_2 \ltwo{x}$, and Lemmas~\ref{lemma:moment-consequences},~\ref{lemma:higher-order-moments}
and~\ref{lemma:higher-order-matrix-moments}.
Coupled with our bound~\eqref{eqn:intermediate-norm-expected-bound}, we
use the fact that $(a + b)^2 \le 2 a^2 + 2 b^2$
to obtain
\begin{equation}
  \label{eqn:norm-outside-expectation-split}
  \ltwo{\E[\optvar_1 - \optvar^*]}^2
  \le 2 \ltwo{\E[\invhessian(P + Q) \invhessian \nabla F_1(\optvar^*)]}^2
  + \order(\numobs^{-3}).
\end{equation}
We focus on bounding the remaining
expectation. We have
the following series of inequalities:
\begin{align*}
  \ltwo{\E[\invhessian(P + Q) \invhessian \nabla F_1(\optvar^*)]}
  & \stackrel{(i)}{\le} \E\left[\matrixnorm{\invhessian(P + Q)}_2
    \ltwo{\invhessian \nabla F_1(\optvar^*)}\right] \\
  & \stackrel{(ii)}{\le} \left(\E\left[\matrixnorm{\invhessian(P + Q)}_2^2\right]
  \E\left[\ltwo{\invhessian \nabla F_1(\optvar^*)}^2\right]\right)^{\half} \\
  & \stackrel{(iii)}{\le} \left(2\E\left[\matrixnorm{\invhessian P}_2^2
    + \matrixnorm{\invhessian Q}_2^2\right]
  \E\left[\ltwo{\invhessian \nabla F_1(\optvar^*)}^2\right]\right)^{\half}.
\end{align*}
Here step~(i) follows from Jensen's inequality and the fact that
$\ltwo{Ax} \le \matrixnorm{A}_2 \ltwo{x}$; step~(ii) uses the
Cauchy-Schwarz inequality; and step~(iii) follows from the fact that
$(a + b)^2 \le 2a^2 + 2b^2$.  We have already bounded the first two
terms in the product in our proofs; in particular,
Lemma~\ref{lemma:moment-consequences} guarantees that
$\E[\matrixnorm{P}_2^2] \le C \lipgrad \log d / \numobs$, while
\begin{equation*}
  \E[\matrixnorm{Q}_2^2]
  \le \E\bigg[\frac{1}{\numobs}
    \sum_{i=1}^n \liphessian(\statrv_i)^4\bigg]^\half\
  \E[\ltwo{\optvar_1 - \optvar^*}^4]^\half
  \le C
  \frac{\liphessian^2 \lipobj^2}{(1 - \rho)^2 \strongparam^2}
  \cdot \numobs^{-1}
\end{equation*}
for some numerical constant $C$ (recall
Lemma~\ref{lemma:higher-order-matrix-moments}).  Summarizing our bounds on
$\matrixnorm{P}_2$ and $\matrixnorm{Q}_2$, we have
\begin{align}
  \lefteqn{\ltwo{
      \E\left[\invhessian(P + Q) \invhessian \nabla F_1(\optvar^*)\right]}^2}
  \nonumber \\
  & \le
  2 \matrixnorm{\invhessian}_2^2
  \left(\frac{2 \lipgrad^2(\log d + 1)}{\numobs}
  + 2 C \frac{\liphessian^2 \lipobj^2}{(1 - \rho)^2 \strongparam^2 \numobs}
  + \order(\numobs^{-2})\right)
  \E\left[\ltwo{\invhessian \nabla F_1(\optvar^*)}^2\right].
  \label{eqn:matrices-bounded}
\end{align}
From Assumption~\ref{assumption:smoothness} we know that $\E[\ltwo{\nabla
    F_1(\optvar^*)}^2] \le \lipobj^2 / n$ and $\matrixnorm{\invhessian}_2 \le
1 / \strongparam$, and hence we can further simplify the
bound~\eqref{eqn:matrices-bounded} to obtain
\begin{align*}
  \ltwo{\E[\optvar_1 - \optvar^*]}^2
  &\le
  \frac{C}{\strongparam^2}
  \left(\frac{\lipgrad^2 \log d
  + \liphessian^2 \lipobj^2 / \strongparam^2 (1 - \rho)^2}{\numobs}
  \right)
  \E\left[\ltwo{\invhessian \nabla F_1(\optvar^*)}^2\right]
  + \order(\numobs^{-3})\\
  &=
  \frac{C}{\strongparam^2}
  \left(\frac{\lipgrad^2 \log d
  + \liphessian^2 \lipobj^2 / \strongparam^2 (1 - \rho)^2}{\numobs^2}
  \right)
  \E\left[\ltwo{\invhessian \nabla f(\optvar^*;X)}^2\right]
  + \order(\numobs^{-3})
\end{align*}
for some numerical constant $C$, where we have applied our earlier
inequality~\eqref{eqn:norm-outside-expectation-split}.
Noting that we may (without loss of
generality) take $\rho < \half$, then applying this inequality with the
bound~\eqref{eqn:expected-norm-bound} on $\E[\ltwo{\optvar_1 - \optvar^*}^2]$
we previously proved to our decomposition~\eqref{eqn:error-inequality}
completes the proof.



\section{Proof of Theorem~\ref{theorem:bootstrap}}

\newcommand{\remainder}{\ensuremath{\mathcal{R}}}
\newcommand{\optvector}{\optvar}

Our proof of Theorem~\ref{theorem:bootstrap} begins with a simple inequality
that mimics our first inequality~\eqref{eqn:error-inequality} in the proof of
Theorem~\ref{theorem:no-bootstrap}. Recall the definitions of the averaged
vector $\optavg_1$ and subsampled averaged vector $\optavg_2$. Let
$\optvector_1$ denote the minimizer of the (an arbitrary) empirical risk
$F_1$, and $\optvector_2$ denote the minimizer of the resampled empirical risk
$F_2$ (from the same samples as $\optvector_1$). Then we have
\begin{equation}
  \E\left[\ltwo{\frac{\optavg_1 - \ratio \optavg_2}{1 - \ratio}
      - \optvar^*}^2\right]
  \le \ltwo{\E\left[\frac{\optvar_1 - \ratio \optvar_2}{1 - \ratio}
      - \optvar^*
      \right]}^2
  + \frac{1}{\nummac}
  \E\left[\ltwo{\frac{\optvector_1 - \ratio \optvar_2}{1 - \ratio}
      - \optvar^*}^2\right].
  \label{eqn:bootstrap-decomposition}
\end{equation}
Thus, parallel to our proof of Theorem~\ref{theorem:no-bootstrap}, it
suffices to bound the two terms in the
decomposition~\eqref{eqn:bootstrap-decomposition}
separately. Specifically, we prove the following two lemmas.
\begin{lemma}
  \label{lemma:target-norm-e}
  Under the conditions of Theorem~\ref{theorem:bootstrap},
  \begin{equation}
    \label{eqn:target-norm-e}
    \ltwo{\E\left[\frac{\optvector_1 - \ratio\optvector_2}{1 - \ratio}
        -\optvector^* \right]}^2
    \le \order(1)
    \frac{1}{\ratio(1 - \ratio)^2}
    \left(\frac{\lipthird^2 \lipobj^6}{\strongparam^6}
    + \frac{\lipobj^4 \liphessian^2}{\strongparam^4} d \log d\right)
    \frac{1}{\numobs^3}.
  \end{equation}
\end{lemma}
\begin{lemma}
  \label{lemma:target-e-norm}
  Under the conditions of Theorem~\ref{theorem:bootstrap},
  \begin{equation}
    \label{eqn:target-e-norm}
    \E\left[\ltwo{\optvar_1 - \optvar^* - \ratio(\optvar_2
        - \optvar^*)}^2\right]
      \le (2+3\ratio)\E\left[
        \ltwo{\nabla^2 F_0(\optvar^*)^{-1}\nabla F_1(\optvar^*)}^2\right]
      + \order(\numobs^{-2})
  \end{equation}
\end{lemma}
\noindent
In conjunction, Lemmas~\ref{lemma:target-norm-e} and~\ref{lemma:target-e-norm}
coupled with the decomposition~\eqref{eqn:bootstrap-decomposition}
yield the desired claim. Indeed, applying each
of the lemmas to the decomposition~\eqref{eqn:bootstrap-decomposition}, we see
that
\begin{align*}
  \E\left[\ltwo{\frac{\optavg_1 - \ratio\optavg_2}{1 - \ratio}
      -\optvar^*}^2\right]
  & \le \frac{2+3 \ratio}{(1 - \ratio)^2 \nummac}
  \E\left[\ltwo{\nabla^2 F_0(\optvar^*)^{-1}\nabla F_1(\optvar^*)}^2\right] \\
  &\qquad ~ +
  \order\left(\frac{1}{(1 - \ratio)^2}\nummac^{-1}\numobs^{-2}\right)
  + \order\left(\frac{1}{\ratio (1 - \ratio)^2} \numobs^{-3}\right),
\end{align*}
which is the statement of Theorem~\ref{theorem:bootstrap}.

The remainder of our argument is devoted to establishing
Lemmas~\ref{lemma:target-norm-e} and~\ref{lemma:target-e-norm}.
Before providing their proofs (in
Appendices~\ref{sec:proof-target-norm-e}
and~\ref{sec:proof-target-e-norm} respectively), we require some
further set-up and auxiliary results.  Throughout the rest of the
proof, we use the notation
\begin{equation*}
  Y = Y' + \remainder_k
\end{equation*}
for some random variables $Y$ and $Y'$ to mean that there exists a random
variable $Z$ such that $Y = Y' + Z$ and $\E[\ltwo{Z}^2] =
\order(\numobs^{-k})$.\footnote{ Formally, in our proof this will mean that
  there exist random vectors $Y$, $Y'$, and $Z$ that are measurable with
  respect to the $\sigma$-field $\sigma(\statrv_1, \ldots, \statrv_\numobs)$,
  where $Y = Y' + Z$ and $\E[\ltwo{Z}^2] = \order(\numobs^{-k})$.}
The symbol $\remainder_k$ may
indicate different random variables throughout a proof and is notational
shorthand for a moment-based big-O notation. We also remark that
if we have $\E[\ltwo{Z}^2] = \order(a^k \numobs^{-k})$, we
have $Z = a^{k/2} \remainder_k$, since $(a^{k/2})^2 = a^k$.
For shorthand, we also say
that $\E[Z] = \order(h(\numobs))$ if $\ltwo{\E[Z]} =
\order(h(\numobs))$, which implies that if $Z = \remainder_k$ then
$\E[Z] = \order(\numobs^{-k/2})$, since
\begin{equation*}
  \ltwo{\E[Z]} \le \sqrt{\E[\ltwo{Z}^2]} = \order(\numobs^{-k/2}).
\end{equation*}

\subsection{Optimization Error Expansion}

In this section, we derive a sharper asymptotic expansion of the optimization
errors $\optvar_1 - \optvar^*$.
Recall our definition of the Kronecker product $\otimes$, where for vectors
$u, v$ we have $u \otimes v = u v^\top$. With this notation, we have the
following expansion of $\optvar_1 - \optvar^*$.
In these lemmas, $\remainder_3$
denotes a vector $Z$ for which $\E[\ltwo{Z}^2] \le c n^{-3}$ for a numerical
constant $c$.
\begin{lemma}
  \label{lemma:first-level-third-order-expansion}
  Under the conditions of Theorem~\ref{theorem:bootstrap}, we have
  \begin{align}
    \optvar_1 - \optvar^* & = -\invhessian \nabla F_1(\optvar^*)
    + \invhessian(\nabla^2 F_1(\optvar^*) - \hessian)
    \invhessian \nabla F_1(\optvar^*)
    \label{eqn:third-order-expansion} \\
    & \quad ~ - \invhessian \nabla^3 F_0(\optvar^*)\left(
    (\invhessian \nabla F_1(\optvar^*)) \otimes
    (\invhessian \nabla F_1(\optvar^*)) \right)\nonumber \\
    & \quad ~ + \left(\lipthird^2 \lipobj^6 / \strongparam^6
    + \lipobj^4 \liphessian^2 d \log(d) / \strongparam^4\right)
    \remainder_3.
    \nonumber
  \end{align}
\end{lemma}
\noindent
We prove Lemma~\ref{lemma:first-level-third-order-expansion} in
Appendix~\ref{sec:proof-third-order-expansion}.  The lemma requires
careful moment control over the expansion $\optvar_1 - \optvar^*$, leading
to some technical difficulty, but is similar in spirit to the results leading
to Theorem~\ref{theorem:no-bootstrap}.

An immediately analogous result to
Lemma~\ref{lemma:first-level-third-order-expansion}
follows for our sub-sampled estimators. Since we use
$\ceil{\ratio \numobs}$ samples to compute $\optvar_2$, the second
level estimator, we find
\begin{lemma}
  \label{lemma:second-level-third-order-expansion}
  Under the conditions of Theorem~\ref{theorem:bootstrap},
  we have
  \begin{align*}
    \optvar_2 - \optvar^* & = -\invhessian \nabla F_2(\optvar^*)
    + \invhessian(\nabla^2 F_2(\optvar^*) - \hessian)
    \invhessian \nabla F_2(\optvar^*) \\
    & ~~ - \invhessian \nabla^3 F_0(\optvar^*)\left(
    (\invhessian \nabla F_2(\optvar^*)) \otimes
    (\invhessian \nabla F_2(\optvar^*)) \right)\\
    & ~~ + \ratio^{-\frac{3}{2}} \left(\lipthird^2 \lipobj^6 / \strongparam^6
    + \lipobj^4 \liphessian^2 d \log(d) / \strongparam^4\right)
    \remainder_3.
  \end{align*}
\end{lemma}

\subsection{Bias Correction}

Now that we have given Taylor expansions that describe the behavior of
$\optvar_1 - \optvar^*$ and $\optvar_2 - \optvar^*$, we can prove
Lemmas~\ref{lemma:target-norm-e} and~\ref{lemma:target-e-norm} (though, as
noted earlier, we defer the proof of Lemma~\ref{lemma:target-e-norm} to
Appendix~\ref{sec:proof-target-e-norm}). The key insight is that expectations
of terms involving $\nabla F_2(\optvar^*)$ are nearly the same as expectations
of terms involving $\nabla F_1(\optvar^*)$, except that some corrections for
the sampling ratio $\ratio$ are necessary.

We begin by noting that
\begin{equation}
  \label{eqn:separate-combination-to-two-opt-errors}
  \frac{\optvar_1 -\ratio\optvar_2}{1 - \ratio} - \optvar^*
  = \frac{\optvar_1 - \optvar^*}{1 - \ratio}
  - \ratio\frac{\optvar_2-\optvar^*}{1 - \ratio}.
\end{equation}
In Lemmas~\ref{lemma:first-level-third-order-expansion}
and~\ref{lemma:second-level-third-order-expansion}, we
derived expansions for each of the right hand side terms, and since
\begin{equation*}
  \E[\invhessian\nabla F_1(\optvar^*)] = 0
  ~~~ \mbox{and} ~~~
  \E[\invhessian\nabla F_2(\optvar^*)] = 0,
\end{equation*}
Lemmas~\ref{lemma:first-level-third-order-expansion}
and~\ref{lemma:second-level-third-order-expansion} coupled with the
rewritten
correction~\eqref{eqn:separate-combination-to-two-opt-errors} yield
\begin{align}
  \E[\optvar_1 - \optvar^* - \ratio(\optvar_2-\optvar^*)]
  & = -\ratio\E[\invhessian(\nabla^2 F_2(\optvar^*) - \hessian)
    \invhessian \nabla F_2(\optvar^*)] \nonumber \\
  & \quad ~ + \E[\invhessian(\nabla^2 F_1(\optvar^*) - \hessian)
    \invhessian \nabla F_1(\optvar^*)] \nonumber \\
  & \quad ~ + \ratio\E[\invhessian \nabla^3 F_0(\optvar^*)\left(
    (\invhessian \nabla F_2(\optvar^*)) \otimes
    (\invhessian \nabla F_2(\optvar^*)) \right)] \nonumber \\
  & \quad ~ - \E[\invhessian \nabla^3 F_0(\optvar^*)\left(
    (\invhessian \nabla F_1(\optvar^*)) \otimes
    (\invhessian \nabla F_1(\optvar^*)) \right)] \nonumber \\
  & \quad ~ + \order(1) \ratio^{-1/2}\left(
  \lipthird^2 \lipobj^6 / \strongparam^6
  + \lipobj^4 \liphessian^2 d \log(d) / \strongparam^4\right)
  \numobs^{-3/2}.
  \label{eqn:mean-expansion-final}
\end{align}
Here the remainder terms follow because of the $\ratio^{-3/2} \remainder_3$
term on $\optvar_2 - \optvar^*$. 


\subsection{Proof of Lemma~\ref{lemma:target-norm-e}}
\label{sec:proof-target-norm-e}

To prove the claim in the lemma, it suffices to show that
  \begin{equation}
    \label{eqn:second-order-cancellation-one}
    \ratio\E[\invhessian(\nabla^2 F_2(\optvar^*) - \hessian) \invhessian
      \nabla F_2(\optvar^*)]
    =
    \E[\invhessian(\nabla^2
      F_1(\optvar^*) - \hessian) \invhessian \nabla F_1(\optvar^*)]
  \end{equation}
  and
  \begin{align}
    \lefteqn{\ratio\E[\invhessian \nabla^3 F_0(\optvar^*)\left(
        (\invhessian \nabla F_2(\optvar^*)) \otimes (\invhessian \nabla
        F_2(\optvar^*)) \right)]} \nonumber \\
    & \quad =
    \E[\invhessian \nabla^3 F_0(\optvar^*)\left( (\invhessian
      \nabla F_1(\optvar^*)) \otimes (\invhessian \nabla F_1(\optvar^*))
      \right)]
    \label{eqn:second-order-cancellation-two}
  \end{align}
  Indeed, these two claims combined with the
  expansion~\eqref{eqn:mean-expansion-final} yield the
  bound~\eqref{eqn:target-norm-e} in Lemma~\ref{lemma:target-norm-e}
  immediately.

  We first consider the difference~\eqref{eqn:second-order-cancellation-one}.
  To make things notationally simpler, we define functions $A
  : \statsamplespace \rightarrow \R^{d \times d}$ and $B : \statsamplespace
  \rightarrow \R^d$ via $A(\statsample) \defeq \invhessian (\nabla^2
  f(\optvar^*; \statsample) - \hessian)$ and $B(\statsample) \defeq \invhessian
  \nabla f(\optvar^*; \statsample)$. If we let $\samples_1 = \{\statrv_1,
  \ldots, \statrv_\numobs\}$ be the original samples and
  $\samples_2 = \{Y_1, \ldots, Y_{\ratio \numobs}\}$ be the subsampled
  dataset, we must show
  \begin{equation*}
    \ratio\E\bigg[\frac{1}{(\ratio \numobs)^2}
      \sum_{i, j}^{\ratio \numobs} A(Y_i) B(Y_j)
      \bigg]
    = \E\bigg[\frac{1}{\numobs^2}
      \sum_{i,j}^\numobs A(\statrv_i) B(\statrv_j)\bigg].
  \end{equation*}
  Since the $Y_i$ are sampled without replacement (i.e.\
  from $\statprob$ directly), and $\E[A(\statrv_i)] =
  0$ and $\E[B(\statrv_i)] = 0$, we find that $\E[A(Y_i) B(Y_j)] = 0$ for
  $i \neq j$, and thus
  \begin{equation*}
    \sum_{i, j}^{\ratio \numobs} \E[A(Y_i) B(Y_j)]
    = \sum_{i=1}^{\ratio \numobs} \E[A(Y_i) B(Y_i)]
    = \ratio \numobs \E[A(Y_1) B(Y_1)].
  \end{equation*}
  In particular, we see that the
  equality~\eqref{eqn:second-order-cancellation-one} holds:
  \begin{align*}
    \frac{\ratio}{(\ratio \numobs)^2}
    \sum_{i,j}^{\ratio \numobs} \E[A(Y_i) B(Y_j)]
    =
    \frac{\ratio}{\ratio \numobs} \E[A(Y_1) B(Y_1)]
    & = \frac{1}{n} \E[A(\statrv_1) B(\statrv_1)] \\
    & = \frac{1}{\numobs^2}
    \sum_{i, j}^\numobs \E[A(\statrv_i) B(\statrv_j)].
  \end{align*}
 The statement~\eqref{eqn:second-order-cancellation-two} follows from
 analogous arguments.


\subsection{Proof of Lemma~\ref{lemma:target-e-norm}}
\label{sec:proof-target-e-norm}

The proof of Lemma~\ref{lemma:target-e-norm} follows from that of
Lemmas~\ref{lemma:first-level-third-order-expansion}
and~\ref{lemma:second-level-third-order-expansion}. We first claim
that
\begin{equation}
  \optvar_1-\optvar^* = -\invhessian \nabla F_1(\optvar^*) + \remainder_2
  ~~~ \mbox{and} ~~~ \optvar_2-\optvar^* = -\invhessian \nabla
  F_2(\optvar^*) + \ratio^{-1}\remainder_2.
  \label{eqn:two-remainders}
\end{equation}
The proofs of both claims similar, so we focus on proving the second
statement.  Using the inequality $(a+b+c)^2\leq 3(a^2+b^2+c^2)$ and
Lemma~\ref{lemma:second-level-third-order-expansion}, we see
that
\begin{align}
  \E\left[\ltwo{\optvar_2-\optvar^* + \invhessian \nabla F_2(\optvar^*)}^2
    \right]
  & \le 3\E\left[\ltwo{\invhessian(\nabla^2 F_2(\optvar^*) - \hessian)
      \invhessian \nabla F_2(\optvar^*)}^2\right] \nonumber\\
  & ~~ + 3\E\left[\ltwo{\invhessian \nabla^3 F_0(\optvar^*)\left(
      (\invhessian \nabla F_2(\optvar^*)) \otimes
      (\invhessian \nabla F_2(\optvar^*)) \right)}^2\right] \nonumber\\
  & ~~ + 3 \ratio^{-3} \order(\numobs^{-3}).
  \label{eqn:bootstrap-e-norm-expansion-bound}
\end{align}
We now bound the first two terms in
inequality~\eqref{eqn:bootstrap-e-norm-expansion-bound}.  Applying
the Cauchy-Schwarz inequality and Lemma~\ref{lemma:moment-consequences}, the
first term can be upper bounded as
\begin{align*}
  \lefteqn{\E\left[\ltwo{\invhessian(\nabla^2 F_2(\optvar^*) - \hessian)
        \invhessian \nabla F_2(\optvar^*)}^2\right]} \\ 
  & \le \left(\E\left[\matrixnorm{\invhessian (\nabla^2 F_2(\optvar^*) -
      \hessian)}_2^4\right]
  \E\left[\ltwo{\invhessian \nabla F_2(\optvar^*)}^4\right]\right)^{1/2} \\
  & = \left(\ratio^{-2})\order(\log^2(d) \numobs^{-2})
  \cdot \ratio^{-2} \order(\numobs^{-2})\right)^{1/2} \; =
  \; \ratio^{-2} \order(n^{-2}),
\end{align*}
where the order notation subsumes the logarithmic factor in the
dimension.  Since $\nabla^3 F_0(\optvar^*) : \R^{d^2} \rightarrow
\R^d$ is linear, the second term in the
inequality~\eqref{eqn:bootstrap-e-norm-expansion-bound} may be bounded
completely analogously as it involves the outer product $\invhessian
\nabla F_2(\optvar^*) \otimes \invhessian \nabla F_2(\optvar^*)$.
Recalling the
bound~\eqref{eqn:bootstrap-e-norm-expansion-bound}, we have thus shown that
\begin{equation*}
  \E\left[\ltwo{\optvar_2-\optvar^* + \invhessian \nabla
      F_2(\optvar^*)}^2 \right]
  = \ratio^{-2} \order(\numobs^{-2}),
\end{equation*}
or $\optvar_2 - \optvar^* = -\invhessian \nabla F_2(\optvar^*) +
\ratio^{-1} \remainder_2$. The proof of the first equality in
equation~\eqref{eqn:two-remainders} is entirely analogous.

We now apply the equalities~\eqref{eqn:two-remainders} to obtain the
result of the lemma. We have
\begin{equation*}
  \E \left[\ltwo{\optvar_1 - \optvar^* - \ratio(\optvar_2 -
      \optvar^*)}^2\right] = \E \left [\ltwo{-\invhessian \nabla
      F_1(\optvar^*) + \ratio \invhessian \nabla F_2(\optvar^*) +
      \remainder_2}^2 \right].
\end{equation*}
Using the inequality
$(a + b)^2 \le (1 + \eta) a^2 + (1 + 1 / \eta) b^2$ for any $\eta \ge 0$,
we have
\begin{align*}
  (a + b + c)^2
  & \le (1 + \eta) a^2 + (1 + 1/\eta) (b + c)^2 \\
  & \le (1 + \eta) a^2 + (1 + 1 / \eta)(1 + \alpha) b^2
  + (1 + 1/\eta) (1 + 1/\alpha) c^2
\end{align*}
for any $\eta, \alpha \ge 0$. Taking $\eta = 1$ and $\alpha = 1/2$, we obtain
$(a+b+c)^2 \leq 2a^2+ 3b^2 + 6c^2$, so applying the triangle inequality,
we have
\begin{align}
  \label{eqn:bootstrap-e-norm-bound}
  \E \left[\ltwo{\optvar_1 - \optvar^* - \ratio(\optvar_2 -
      \optvar^*)}^2 \right]
  & = \E\left[\ltwo{-\invhessian \nabla F_1(\optvar^*)
      + \ratio \invhessian \nabla F_2(\optvar^*)
      + \remainder_2}^2 \right] \\
  & \le 2 \E \left[\ltwo{\invhessian \nabla
      F_1(\optvar^*)}^2 \right]
  + 3 \ratio^2 \E \left[ \ltwo{\invhessian \nabla F_2(\optvar^*)}^2
    \right] + \order(\numobs^{-2}).
  \nonumber
\end{align}
Since $F_2$ is a sub-sampled version of $F_1$, algebraic manipulations yield
\begin{equation}
  \E\left[\ltwo{\invhessian \nabla F_2(\optvar^*)}^2 \right]
  = \frac{\numobs}{\ratio \numobs}
  \E\left[\ltwo{\invhessian\nabla F_1(\optvar^*)}^2\right]
  = \frac{1}{\ratio} \E \left [ \ltwo{\invhessian \nabla
      F_1(\optvar^*)}^2 \right].
  \label{eqn:bootstrap-variance-to-true-variance}
\end{equation}
Combining equations~\eqref{eqn:bootstrap-e-norm-bound}
and~\eqref{eqn:bootstrap-variance-to-true-variance}, we obtain the
desired bound~\eqref{eqn:target-e-norm}.


\section{Proof of Theorem~\ref{theorem:sgd}}
\label{sec:proof-sgd}

We begin by recalling that if $\optvar^\numobs$ denotes the output
of performing stochastic gradient on one machine, then
from the inequality~\eqref{eqn:error-inequality} we have the upper bound
\begin{equation*}
  \E[\ltwobig{\optavg^\numobs -\optvar^*}^2]
  \leq \frac{1}{\nummac}
  \E[\ltwo{\optvar^\numobs - \optvar^*}^2] +
  \ltwo{\E[\optvar^\numobs - \optvar^*]}^2.
\end{equation*}
To prove the error bound \eqref{eqn:sgd-bound}, it thus suffices to
prove the inequalities
\begin{subequations}
  \begin{align}
    \E[\ltwo{\optvar^\numobs - \optvar^*}^2]
    & \le \frac{\alpha \lipobj^2}{\lambda^2 \numobs},
    ~~~\mbox{and}
    \label{eqn:sgd-bound-part1} \\
    \ltwo{\E[\optvar^\numobs - \optvar^*]}^2
    &\leq \frac{\beta^2}{\numobs^{3/2}}.
    \label{eqn:sgd-bound-part2}
  \end{align}
\end{subequations}
 Before proving the theorem, we introduce some notation and a few
 preliminary results.  Let $\gradient_t = \nabla f(\optvar^t;
 \statrv_t)$ be the gradient of the $t^{th}$ sample in stochastic
 gradient descent, where we consider running SGD on a single
 machine. We also let
\begin{equation*}
  \project(v) \defeq \argmin_{\optvar \in \optvarspace}
  \left\{ \ltwo{\optvar - v}^2 \right\}
\end{equation*}
denote the projection of the point $v$ onto the domain $\optvarspace$.

We now state a known result, which gives sharp rates on the
convergence of the iterates $\{\optvar^t\}$ in stochastic gradient
descent.
\begin{lemma}[\citeauthor{RakhlinShSr12}, \citeyear{RakhlinShSr12}]
  \label{lemma:sgd-distance-bound}
  Assume that $\E[\ltwo{\gradient_t}^2] \leq \lipobj^2$ for all $t$. Choosing
  $\eta_t=\frac{c}{\lambda t}$ for some $c\geq 1$, for any $t \in \N$ we have
  \begin{equation*}
    \E\left[\ltwo{\optvar^t-\optvar^*}^2\right]
    \leq \frac{\alpha \lipobj^2}{\lambda^2t}
    ~~~ \mbox{where} ~~~
    \alpha = 4c^2.
  \end{equation*}
\end{lemma}

\vspace*{.1in}

With these ingredients, we can now turn to the proof of
Theorem~\ref{theorem:sgd}.  Lemma~\ref{lemma:sgd-distance-bound} gives
the inequality~\eqref{eqn:sgd-bound-part1}, so it remains to prove
that $\optavg^\numobs$ has the smaller
bound~\eqref{eqn:sgd-bound-part2} on its bias. To that end, recall the
neighborhood $U_\rho \subset \optvarspace$ in
Assumption~\ref{assumption:smoothness-sgd}, and note that
\begin{align*}
  \optvar^{t + 1} - \optvar^*
  & = \project(\optvar^t - \eta_t \gradient_t - \optvar^*) \\
  & = \optvar^t - \eta_t \gradient_t - \optvar^*
  + \indic{\optvar^{t+1} \not\in U_\rho}
  \left(\project(\optvar^t - \eta_t \gradient_t)
  - (\optvar^t - \eta_t \gradient_t)\right)
\end{align*}
since when $\optvar \in U_\rho$, we have $\project(\optvar) = \optvar$.
Consequently, an application of the triangle inequality gives
\begin{equation*}
  \ltwo{\E[\optvar^{t+1}-\optvar^*]}
  \leq \ltwo{\E[\optvar^{t}-\eta_t\gradient_t-\optvar^*]}
  + \E[\ltwo{(\Pi(\optvar^{t}-\eta_t\gradient_t)
      -(\optvar^{t}-\eta_t\gradient_t))1(\optvar^{t+1}\notin U_\rho)}].
\end{equation*}
By the definition of the projection and the fact that $\optvar^t \in
\optvarspace$, we additionally have
\begin{equation*}
  \ltwo{\project(\optvar^{t}-\eta_t\gradient_t)
    -(\optvar^{t}-\eta_t\gradient_t)}
  \le \ltwo{\optvar^{t}-(\optvar^{t}-\eta_t\gradient_t))}
  \le \eta_t \ltwo{g_t}.
\end{equation*}
Thus, by applying H\"older's inequality (with the
conjugate choices $(p, q) = (4, \frac{4}{3})$) and
Assumption~\ref{assumption:smoothness-sgd}, we have
\begin{align}
  \ltwo{\E[\optvar^{t+1}-\optvar^*]}
  & \le \ltwo{\E[\optvar^{t}-\eta_t\gradient_t-\optvar^*]}
  + \eta_t \E[\ltwo{\gradient_t} \indic{\optvar^{t + 1} \not\in U_\rho}]
  \nonumber \\
  & \le \ltwo{\E[\optvar^{t}-\eta_t\gradient_t-\optvar^*]}
  + \eta_t \sqrt[4]{\E[\ltwo{\gradient_t}^4]}
  \left(\E[\indic{\optvar^t \not\in U_\rho}^{4/3}]\right)^{3/4}
  \nonumber \\
  & \le \ltwo{\E[\optvar^{t}-\eta_t\gradient_t-\optvar^*]}
  + \eta_t \lipobj
  \left(\P(\optvar^t \not\in U_\rho)\right)^{3/4}
  \nonumber \\
  & \le \ltwo{\E[\optvar^{t}-\eta_t\gradient_t-\optvar^*]}
  + \eta_t \lipobj
  \left(\frac{\E\ltwo{\optvar^{t+1}-\optvar^*}^2}{\rho^2}\right)^{3/4},
  \label{eqn:sgd-tplusone-truncation}
\end{align}
the inequality~\eqref{eqn:sgd-tplusone-truncation} following from an
application of Markov's inequality. By applying
Lemma~\ref{lemma:sgd-distance-bound}, we finally obtain
\begin{align}
  \ltwo{\E[\optvar^{t + 1} - \optvar^*]}
  & \le \ltwo{\E[\optvar^t - \eta_t \gradient_t - \optvar^*]}
  + \eta_t \lipobj \left(\frac{\alpha \lipobj^2}{\strongparam^2
    \rho^2 t}\right)^{3/4}
  \nonumber \\
  & = \ltwo{\E[\optvar^t - \eta_t \gradient_t - \optvar^*]}
  + \frac{c \alpha^{3/4} \lipobj^{5/2}}{
    \strongparam^{5/2} \rho^{3/2}}
  \cdot \frac{1}{t^{7/4}}.
  \label{eqn:sgd-t-plus-one-mean-square}
\end{align}

Now we turn to controlling the rate at which $\optvar^t - \eta_t
\gradient_t$ goes to zero. Let $f_t(\cdot) = f(\cdot; \statrv_t)$ be
shorthand for the loss evaluated on the $t^{th}$ data point. By
defining
\begin{align*}
r_t = g_t - \nabla f_t(\optvar^*) - \nabla^2
  f_t(\optvar^*)(\optvar^t - \optvar^*),
\end{align*} 
a bit of algebra yields
\begin{equation*}
  \gradient_t = \nabla f_t(\optvar^*)
  + \nabla^2 f_t(\optvar^*)(\optvar^t-\optvar^*) + r_t.
\end{equation*}
Since $\optvar^t$ belongs to the $\sigma$-field of $\statrv_1, \ldots,
\statrv_{t-1}$, the Hessian $\nabla^2 f_t(\optvar^*)$ is (conditionally)
independent of $\optvar^t$ and
\begin{equation}
  \label{eqn:sgd-truncate-grad}
  \E[\gradient_t] =
  \nabla^2 F_0(\optvar^*)\E[\optvar^t-\optvar^*]
  + \E[r_t \indic{\optvar^t\in U_\rho}] +
  \E[r_t\indic{\optvar^t\notin U_\rho}].
\end{equation}
If $\optvar^t\in U_\rho$, then Taylor's theorem implies that $r_t$ is the
Lagrange remainder
\begin{align*}
  r_t = (\nabla^2 f_t(\optvar') - \nabla^2 f_t(\optvar^*))(\optvar'-\optvar^*),
\end{align*}
where $\optvar'= \interp\optvar^t +(1-\interp)\optvar^*$ for
some $\interp \in [0, 1]$.
Applying Assumption~\ref{assumption:smoothness-sgd} and
H\"older's inequality, we find that
since $\optvar^t$ is conditionally independent of $\statrv_t$,
\begin{align*}
  \E\left[\ltwo{r_t \indic{\optvar^t \in U_\rho}}\right]
  & \le \E\left[
    \matrixnorm{\nabla^2 f(\optvar'; \statrv_t)
      - \nabla^2 f(\optvar^*; \statrv_t)}
    \ltwo{\optvar^t - \optvar^*}
    \indic{\optvar^t \in U_\rho}\right] \\
  & \le \E\left[\liphessian(\statrv_t) \ltwo{\optvar^t - \optvar^*}^2\right]
  = \E[\liphessian(\statrv_t)] \E[\ltwo{\optvar^t - \optvar^*}^2] \\
  & \le \liphessian \E\left[\ltwo{\optvar^t - \optvar^*}^2\right]
  \le \frac{\alpha \liphessian \lipobj^2}{\strongparam^2 t}.
\end{align*}
On the other hand, when $\optvar^t \not \in U_\rho$, we have
the following sequence of inequalities:
\begin{align*}
  \lefteqn{\E\left[\ltwo{r_t \indic{\optvar^t \not\in U_\rho}}\right]
    \stackrel{(i)}{\le} \sqrt[4]{\E[\ltwo{r_t}^4]}
    \left(\P(\optvar^t \not \in U_\rho)\right)^{3/4}} \\
  & \qquad\qquad\quad ~ \stackrel{(ii)}{\le} \sqrt[4]{
    3^3\left(\E[\ltwo{\gradient_t}^4]
    + \E[\ltwo{\nabla f_t(\optvar^*)}^4]
    + \E[\ltwo{\nabla^2 f_t(\optvar^*)(\optvar^t - \optvar^*)}^4]\right)}
  \left(\P(\optvar^t \not \in U_\rho)\right)^{3/4} \\
  & \qquad\qquad\quad ~\le 3^{3/4}
  \sqrt[4]{\lipobj^4 + \lipobj^4 + \lipgrad^4 \radius^4}
  \left(\P(\optvar^t \not \in U_\rho)\right)^{3/4} \\
  & \qquad\qquad\quad \stackrel{(iii)}{\le} 3(\lipobj + \lipgrad \radius)
  \left(\frac{\alpha \lipobj^2}{\strongparam^2 \rho^2 t}\right)^{3/4}.
\end{align*}
Here step~(i) follows from H\"older's inequality (again applied with
the conjugates $(p, q) = (4, \frac{4}{3})$); step~(ii) follows from
Jensen's inequality, since $(a + b + c)^4 \le 3^3(a^4 + b^4 + c^4)$;
and step~(iii) follows from Markov's inequality, as in the
bounds~\eqref{eqn:sgd-tplusone-truncation}
and~\eqref{eqn:sgd-t-plus-one-mean-square}.  Combining our two bounds
on $r_t$, we find that
\begin{equation}
  \E[\ltwo{r_t}]
  \le \frac{\alpha \liphessian \lipobj^2}{\strongparam^2 t}
  + \frac{3 \alpha^{3/4} \lipobj^{3/2} (\lipobj + \lipgrad \radius)}{
    \strongparam^{3/2} \rho^{3/2}} \cdot \frac{1}{t^{3/4}}.
  \label{eqn:sgd-rt-bound}
\end{equation}

By combining the expansion~\eqref{eqn:sgd-truncate-grad} with the
bound~\eqref{eqn:sgd-rt-bound}, we find that
\begin{align*}
  \lefteqn{\ltwo{\E[\optvar^t - \eta_t \gradient_t - \optvar^*]}
    = \ltwo{\E[(I - \eta_t \nabla^2 F_0(\optvar^*))(\optvar^t - \optvar^*)
        + \eta_t r_t]}} \\
  & \qquad \qquad
  \le \ltwo{\E[(I - \eta_t \nabla^2 F_0(\optvar^*))(\optvar^t - \optvar^*)]}
  + \frac{c \alpha \liphessian \lipobj^2}{\strongparam^3 t^2}
  + \frac{3 c \alpha^{3/4} \lipobj^{3/2} (\lipobj + \lipgrad \radius)}{
    \strongparam^{5/2} \rho^{3/2}} \cdot \frac{1}{t^{7/4}}.
\end{align*}
Using the earlier
bound~\eqref{eqn:sgd-t-plus-one-mean-square}, this inequality
then yields
\begin{equation*}
  \ltwo{\E[\optvar^{t + 1} - \optvar^*]}
  \le \matrixnorm{I - \eta_t \nabla^t F_0(\optvar^*)}_2
  \ltwo{\E[\optvar^t - \optvar^*]}
  + \frac{c \alpha^{3/4} \lipobj^{3/2}}{\strongparam^{5/2} t^{7/4}}
  \!\left(\frac{\alpha^{1/4}\liphessian \lipobj^{1/2}}{\strongparam^{1/2}
    t^{1/4}}
  + \frac{4 \lipobj + \lipgrad \radius}{\rho^{3/2}}\right)\!.
\end{equation*}

We now complete the proof via an inductive argument using our immediately
preceding bounds. Our reasoning follows a similar induction given
by~\citet{RakhlinShSr12}.
First, note that by strong convexity and
our condition that $\matrixnorm{\nabla^2 F_0(\optvar^*)} \le \lipgrad$, we
have
\begin{equation*}
  \matrixnorm{I - \eta_t \nabla^2 F_0(\optvar^*)}
  = 1 - \eta_t \lambda_{\min}(\nabla^2 F_0(\optvar^*)
  \le 1 - \eta_t \strongparam
\end{equation*}
whenever $1 - \eta_t \lipgrad \ge 0$. Define $\tau_0 = \ceil{c \lipgrad /
  \strongparam}$; then for $t \ge t_0$ we obtain
\begin{equation}
  \label{eqn:sgd-induction}
  \ltwo{\E[\optvar^{t+1}-\optvar^*]}
  \le (1-c/t)\ltwo{\E[\optvar^t-\optvar^*]}
  + \frac{1}{t^{7/4}} \cdot
  \frac{c \alpha^{3/4} \lipobj^{3/2}}{\strongparam^{5/2}}
  \left(\frac{\alpha^{1/4}\liphessian \lipobj^{1/2}}{\strongparam^{1/2}
    t^{1/4}}
  + \frac{4 \lipobj + \lipgrad \radius}{\rho^{3/2}}\right).
\end{equation}
For shorthand, we define two intermediate variables
\begin{align*}
  a_{t}= \ltwo{\E(\optvar^{t}-\optvar^*)}
  ~~~\mbox{and}~~~
  b = \frac{c \alpha^{3/4} \lipobj^{3/2}}{\strongparam^{5/2}}
  \left(\frac{\alpha^{1/4}\liphessian \lipobj^{1/2}}{\strongparam^{1/2}}
  + \frac{4 \lipobj + \lipgrad \radius}{\rho^{3/2}}\right).
\end{align*}
Inequality \eqref{eqn:sgd-induction} then implies the inductive
relation $a_{t+1}\leq (1-c/t)a_t + b/t^{7/4}$. Now we show that
by defining $\beta = \max\{\tau_0 \radius, b / (c - 1)\}$, we have
$a_t \le \beta / t^{3/4}$. Indeed, it is clear that $a_1 \le \tau_0 \radius$.
Using the inductive hypothesis, we then have
\begin{equation*}
  a_{t+1} \le
  \frac{(1 - c/t) \beta}{t^{3/4}}
  + \frac{b}{t^{7/4}}
  = \frac{\beta(t - 1)}{t^{7/4}} - \frac{\beta(c - 1) - b}{t^2}
  \le \frac{\beta(t - 1)}{t^{7/4}}
  \le \frac{\beta}{(t + 1)^{3/4}}.
\end{equation*}
This completes the proof of the inequality~\eqref{eqn:sgd-bound-part2}.
\qed

\paragraph{Remark}
If we assume $k$th moment bounds instead of $4$th,
i.e.\ $\E[\matrixnorm{\nabla^2 f(\optvar^*; \statrv)}_2^k] \le \lipgrad^k$ and
$\E[\ltwo{\gradient_t}^k] \le \lipobj^k$, we find the following analogue of
the bound~\eqref{eqn:sgd-induction}:
\begin{align*}
  \ltwo{\E[\optvar^{t + 1} - \optvar^*]}
  & \le (1 - c/t) \ltwo{\E[\optvar^t - \optvar^*]} \\
  & \qquad ~ + \frac{1}{t^{\frac{2k - 1}{k}}} \cdot
  \frac{c \alpha^{\frac{k - 1}{k}} \lipobj^{\frac{2k - 2}{k}}}{
    \strongparam^{\frac{3k - 2}{k}}}
  \left[\frac{\left(54^{1/k} + 1\right) \lipobj + 54^{1/k} \lipgrad \radius}{
      \rho^{\frac{2k - 2}{k}}}
    + \frac{\alpha^{1/k} \liphessian \lipobj^{2/k}}{\strongparam^{2/k} t^{1/k}}
    \right].
\end{align*}
In this case, if we define
\begin{equation*}
  b = \frac{c \alpha^{\frac{k - 1}{k}} \lipobj^{\frac{2k - 2}{k}}}{
    \strongparam^{\frac{3k - 2}{k}}}
  \left[\frac{\left(54^{1/k} + 1\right) \lipobj + 54^{1/k} \lipgrad \radius}{
      \rho^{\frac{2k - 2}{k}}}
    + \frac{\alpha^{1/k} \liphessian \lipobj^{2/k}}{\strongparam^{2/k}}
    \right]
  ~~~ \mbox{and} ~~~
  \beta = \max\left\{\tau_0 \radius, \frac{b}{c - 1}\right\},
\end{equation*}
we have the same result except we obtain the bound
$\ltwo{\E[\optvar^\numobs - \optvar^*]}^2 \le \beta^2 /
\numobs^{\frac{2k - 2}{k}}$.


\section{Proof of Lemma~\ref{lemma:good-events}}
\label{appendix:proof-of-lemma-good-events}

We first prove that under the conditions given in the lemma statement,
the function $F_1$ is $(1 - \rho) \strongparam$-strongly convex over
the ball $U \defeq \left\{\optvar \in \R^d : \ltwo{\optvar
- \optvar^*} <
\delta_\rho \right\}$ around $\optvar^*$. Indeed, fix $\optvarb \in U$,
then use the triangle inequality to conclude that
\begin{align*}
  \matrixnorm{\nabla^2 F_1(\optvarb) - \nabla^2 F_0(\optvar^*)}_2
  & \le \matrixnorm{\nabla^2 F_1(\optvarb) - \nabla^2
  F_1(\optvar^*)}_2 + \matrixnorm{\nabla^2 F_1(\optvar^*) - \nabla^2
  F_0(\optvar^*)}_2 \\ & \le \liphessian \ltwo{\optvarb - \optvar^*}
  + \frac{\rho \strongparam}{2}.
\end{align*}
Here we used Assumption~\ref{assumption:smoothness} on the first term
and the fact that the event $\event_1$ holds on the second.  By our
choice of $\delta_\rho \le \rho \strongparam / 4 \liphessian$, this
final term is bounded by $\strongparam \rho$. In particular, we have
\begin{equation*}
  \nabla^2 F_0(\optvar^*) \succeq \strongparam I ~~~ \mbox{so}
  ~~~ \nabla^2 F_1(\optvarb) \succeq \strongparam I
  - \rho \strongparam I = (1 - \rho) \strongparam I,
\end{equation*}
which proves that $F_1$ is $(1 - \rho) \strongparam$-strongly convex
on the ball $U$.

In order to prove the conclusion of the lemma, we argue
that since $F_1$ is (locally) strongly convex, if the function $F_1$ has small
gradient at the point $\optvar^*$, it must be the case that the
minimizer $\optvar_1$ of $F_1$ is near $\optvar^*$. Then we
can employ reasoning similar to standard analyses of optimality for
globally strongly convex functions (e.g.~\cite{BoydVa04}, Chapter 9).
By definition of (the local) strong convexity on the set $U$, for any
$\optvar' \in \optvarspace$, we have
\begin{equation*}
  F_1(\optvar') \ge F_1(\optvar^*) + \<\nabla F_1(\optvar^*),
  \optvar' - \optvar^*\>
  + \frac{(1 - \rho) \strongparam}{2}
  \min\left\{\ltwo{\optvar^* - \optvar'}^2,
  \delta_\rho^2 \right\}.
\end{equation*}
Rewriting this inequality, we find that
\begin{align*}
  \min\left\{\ltwo{\optvar^* - \optvar'}^2,
  \delta_\rho^2 \right\}
  & \le \frac{2}{(1 - \rho) \strongparam}\left[F_1(\optvar') - F_1(\optvar^*)
    + \<\nabla F_1(\optvar^*), \optvar' - \optvar^*\>\right] \\
  & \le \frac{2}{(1 - \rho) \strongparam}\left[F_1(\optvar') - F_1(\optvar^*)
    + \ltwo{\nabla F_1(\optvar^*)} \ltwo{\optvar' - \optvar^*}\right].
\end{align*}
Dividing each side by $\ltwo{\optvar' - \optvar^*}$, then noting that we may
set $\optvar' = \interp \optvar_1 + (1 - \interp) \optvar^*$ for any $\interp
\in [0, 1]$, we have
\begin{equation*}
  \min\left\{ \interp \ltwo{\optvar_1 - \optvar^*},
  \frac{\delta_\rho^2}{\interp \ltwo{\optvar_1 - \optvar^*}}
  \right\}
  \le \frac{2
    \left[F_1(\interp \optvar_1 + (1 - \interp) \optvar^*)
      - F_1(\optvar^*)\right]}{
    \interp (1 - \rho) \strongparam \ltwo{\optvar_1 - \optvar^*}}
  + \frac{2 \ltwo{\nabla F_1(\optvar^*)}}{(1 - \rho) \strongparam}.
\end{equation*}
Of course, $F_1(\optvar_1) < F_1(\optvar^*)$ by assumption, so
we find that for any $\interp \in (0, 1)$ we have the strict inequality
\begin{equation*}
  \min\left\{ \interp \ltwo{\optvar_1 - \optvar^*},
  \frac{\delta_\rho^2}{\interp \ltwo{\optvar_1 - \optvar^*}}
  \right\}
  < \frac{2 \ltwo{\nabla F_1(\optvar^*)}}{(1 - \rho) \strongparam}
  \le \delta_\rho,
\end{equation*}
the last inequality following from the definition of $\event_2$.  Since this
holds for any $\interp \in (0, 1)$, if $\ltwo{\optvar_1 - \optvar^*} >
\delta_\rho$, we may set $\interp = \delta_\rho / \ltwo{\optvar_1 -
  \optvar^*}$, which would yield a contradiction. Thus, we have
$\ltwo{\optvar_1 - \optvar^*} \le \delta_\rho$, and by our earlier
inequalities,
\begin{equation*}
  \ltwo{\optvar_1 - \optvar^*}^2
  \le \frac{2}{(1 - \rho) \strongparam}
  \left[F_1(\optvar_1) - F_1(\optvar^*) + \ltwo{\nabla F_1(\optvar^*)}
    \ltwo{\optvar_1 - \optvar^*}\right]
  \le \frac{2 \ltwo{\nabla F_1(\optvar^*)}}{(1 - \rho) \strongparam}
  \ltwo{\optvar_1 - \optvar^*}.
\end{equation*}
Dividing by $\ltwo{\optvar_1 - \optvar^*}$ completes the proof.
\qed


\section{Moment bounds}
\label{appendix:moment-bounds}

In this appendix, we state two useful moment bounds, showing how they
combine to provide a proof of Lemma~\ref{lemma:moment-consequences}.
The two lemmas are a vector and a non-commutative matrix variant of
the classical Rosenthal inequalities.  We begin with the case of
independent random vectors: \\
\begin{lemma}[\citet{DeAcosta81}, Theorem 2.1]
  \label{lemma:vector-rosenthal}
  Let $k \ge 2$ and $X_i$ be a sequence of independent random vectors in a
  separable Banach space with norm $\norm{\cdot}$ and $\E[\norm{X_i}^k] <
  \infty$. There exists a finite constant $C_k$ such that
  \begin{equation*}
    \E\bigg[\bigg|\normb{\sum_{i=1}^n X_i} -
      \E\bigg[\normb{\sum_{i=1}^n X_i}\bigg]\bigg|^k\bigg] \le C_k
    \left[\left(\sum_{i=1}^n \E[\norm{X_i}^2]\right)^{k/2} +
      \sum_{i=1}^n \E[\norm{X_i}^k]\right].
  \end{equation*}
\end{lemma}

We say that a random matrix $X$ is symmetrically distributed if $X$
and $-X$ have the same distribution.  For such matrices, we have:
\begin{lemma}[\citet{ChenGiTr12}, Theorem A.1(2)]
  \label{lemma:matrix-rosenthal}
Let $X_i \in \R^{d \times d}$ be independent and symmetrically
distributed Hermitian matrices. Then
  \begin{equation*}
    \E\bigg[\matrixnormb{\sum_{i=1}^n X_i}^k\bigg]^{1/k}
    \le \sqrt{2 e \log d} \;
    \matrixnormb{\bigg(\sum_{i=1}^n \E\left[X_i^2\right]\bigg)^{1/2}}
    + 2 e \log d \left(\E[\max_i\matrixnorm{X_i}^k]\right)^{1/k}.
  \end{equation*}
\end{lemma}

Equipped with these two auxiliary results, we turn to our proof
Lemma~\ref{lemma:moment-consequences}.  To prove the first
bound~\eqref{eqn:grad-bound}, let $2 \le k \le \objmoment$ and note
that by Jensen's inequality, we have
\begin{equation*}
  \E[\ltwo{\nabla F_1(\optvar^*)}^k] \le 2^{k-1} \E\left[\big|
    \ltwo{\nabla F_1(\optvar^*)} - \E[\ltwo{\nabla
        F_1(\optvar^*)}]\big|^k \right] + 2^{k-1}
  \E\left[\ltwo{\nabla F_1(\optvar^*)}\right]^k.
\end{equation*}
Again applying Jensen's inequality, $\E[\ltwo{\nabla f(\optvar^*;
    \statrv)}^2] \le \lipobj^2$. Thus by recalling the definition
$\nabla F_1(\optvar^*) = \frac{1}{\numobs} \sum_{i=1}^\numobs \nabla
f(\optvar^*; \statrv_i)$ and applying the inequality
\begin{align*}
\E[\ltwo{\nabla F_1(\optvar^*)}] \le \E[\ltwo{\nabla
    F_1(\optvar^*)}^2]^{1/2} \le \numobs^{-1/2} \lipobj,
\end{align*}
we see that Lemma~\ref{lemma:vector-rosenthal} implies
$\E\left[\ltwo{\nabla F_1(\optvar^*)}^k\right]$ is upper bounded
by
\begin{align*}
  2^{k-1} C_k \left[\left(\frac{1}{\numobs^2} \sum_{i=1}^\numobs
    \E[\ltwo{\nabla f(\optvar; \statrv_i)}^2]\right)^{k/2} +
    \frac{1}{\numobs^k} \sum_{i = 1}^\numobs \E[\ltwo{\nabla
        f(\optvar^*; \statrv_i)}^k] \right] + 2^{k-1}\E[\ltwo{\nabla
      F_1(\optvar^*)}]^k \\ 
\le 2^{k-1} \frac{C_k}{\numobs^{k/2}}
  \left[\left(\frac{1}{\numobs} \sum_{i=1}^n \E[\ltwo{\nabla
        f(\optvar^*; \statrv_i)}^2] \right)^{k/2} + \frac{1}{\numobs^{k/2}}
    \sum_{i=1}^n \E[\ltwo{\nabla f(\optvar^*; \statrv_i)}^k] \right] +
  \frac{2^{k-1} \lipobj^k}{ \numobs^{k/2}}.
\end{align*}
Applying Jensen's inequality yields
\begin{equation*}
  \left(\frac{1}{\numobs} \sum_{i=1}^\numobs \E[\ltwo{\nabla
      f(\optvar^*; \statrv_i)}^2] \right)^{k/2} \le \frac{1}{\numobs}
  \sum_{i=1}^\numobs \E[\ltwo{\nabla f(\optvar^*; \statrv_i)}^2]^{k/2}
  \le \lipobj^k,
\end{equation*}
completes the proof of the inequality~\eqref{eqn:grad-bound}.

The proof of the bound~\eqref{eqn:hessian-bound} requires a very
slightly more delicate argument involving symmetrization step.
Define matrices $Z_i = \frac{1}{n} \left(\nabla^2 f(\optvar^*;
\statrv_i) - \nabla^2 F_0(\optvar^*)\right)$.  If $\varepsilon_i \in
\{\pm 1\}$ are i.i.d.\ Rademacher variables independent of $Z_i$, then
for any integer $k$ in the interval $[2, \hessmoment]$, a standard
symmetrization argument~\citep[e.g.][Lemma 6.3]{LedouxTa91} implies
that
\begin{equation}
  \E\bigg[\matrixnormb{
      \sum_{i=1}^\numobs Z_i}^k\bigg]^{1/k}
  \le 2 \E\bigg[
    \matrixnormb{\sum_{i=1}^\numobs \varepsilon_i Z_i}^k\bigg]^{1/k}.
  \label{eqn:symmetrization}
\end{equation}

Now we may apply Lemma~\ref{lemma:matrix-rosenthal}, since the
matrices $\varepsilon_i Z_i$ are Hermitian and symmetrically
distributed; by expanding the
definition of the $Z_i$, we find that
\begin{align*}
  \E\left[\matrixnorm{\nabla^2 F_1(\optvar^*)
        - \nabla^2 F_0(\optvar^*)}^k\right]^{1/k}
  & \le 5 \sqrt{\log d}
  \matrixnormb{\bigg(\frac{1}{\numobs^2} \sum_{i=1}^\numobs
    \E[(\nabla^2 f(\optvar; \statrv_i) - \nabla^2 F_0(\optvar^*))^2]
    \bigg)^{1/2}} \\
  & ~~ + 4 e \log d \left(\numobs^{-k} \E[\max_i
    \matrixnorm{\nabla^2 f(\optvar^*; \statrv_i)
      - \nabla^2 F_0(\optvar^*)}^k]\right)^{1/k}.
\end{align*}

Since the $\statrv_i$ are i.i.d., we have
\begin{align*}
  \matrixnormb{\bigg(\frac{1}{\numobs^2} \sum_{i=1}^\numobs
    \E[(\nabla^2 f(\optvar; \statrv_i) - \nabla^2 F_0(\optvar^*))^2]
    \bigg)^{1/2}}
  & = \matrixnorm{\numobs^{-1/2} \E\left[\left(
      \nabla^2 f(\optvar^*; \statrv) - \nabla^2 F_0(\optvar^*)\right)^2
      \right]^{1/2}} \\
  & \le \numobs^{-1/2} \E\left[\matrixnorm{\nabla^2 f(\optvar^*; \statrv)
      - \nabla^2 F_0(\optvar^*)}^2\right]^{1/2}
\end{align*}
by Jensen's inequality, since $\matrixnorm{A^{1/2}} = \matrixnorm{A}^{1/2}$
for semidefinite $A$.
Finally, noting that
\begin{equation*}
  \frac{1}{\numobs^{k}}
  \E\left[\max_i \matrixnorm{\nabla^2 f(\optvar^*; \statrv_i) - \nabla^2 F_0(
      \optvar^*)}^k\right]
  \le \frac{\numobs}{\numobs^{k}} \E\left[
    \matrixnorm{\nabla^2 f(\optvar^*; \statrv) - \nabla^2 F_0(
      \optvar^*)}^k\right]
  \le \numobs^{1-k} \lipgrad^k
\end{equation*}
completes the proof of the second bound~\eqref{eqn:hessian-bound}.


\section{Proof of Lemma~\ref{lemma:first-level-third-order-expansion}}
\label{sec:proof-third-order-expansion}

\newcommand{\graddiff}{\mathsf{G}}
\newcommand{\littlegraddiff}{\mathsf{g}}

The proof follows from a slightly more careful application of the Taylor
expansion~\eqref{eqn:taylor-error-expansion}.  The starting point in our
proof is to recall the success events~\eqref{eqn:good-events} and the joint
event $\event \defeq \event_0 \cap \event_1 \cap \event_2$. We begin by
arguing that we may focus on the case where $\event$ holds. Let $C$ denote
the right hand side of the equality~\eqref{eqn:third-order-expansion}
except for the remainder $\remainder_3$ term.  By
Assumption~\ref{assumption:smoothness}, we follow
the bound~\eqref{eqn:no-bad-events} (with $\min\{\objmoment, \gradmoment, \hessmoment\} \ge 8$)
to find that
\begin{equation*}
  \E\left[\indic{\event^c} \ltwo{\optvar_1 - \optvar^*}^2\right]
  = \order\left(\radius^2 \numobs^{-4}\right),
\end{equation*}
so we can focus on the case where the joint event
$\event = \event_0 \cap \event_1 \cap \event_2$ does occur.

Defining $\optvarerr = \optvar_1 - \optvar^*$ for notational
convenience, on $\event$ we have that for some $\interp \in [0, 1]$, with
$\optvar' = (1 - \interp) \optvar_1 + \interp \optvar^*$,
\begin{align*}
  0 & = \nabla F_1(\optvar^*) +
  \nabla^2 F_1(\optvar^*)\optvarerr
  + \nabla^3 F_1(\optvar')(\optvarerr \otimes \optvarerr) \\
  & = \nabla F_1(\optvar^*) +
  \nabla^2 F_0(\optvar^*)\optvarerr
  + \nabla^3 F_0(\optvar^*)(\optvarerr \otimes \optvarerr) \\
  & \qquad ~ + (\nabla^2 F_1(\optvar^*) - \nabla^2 F_0(\optvar^*)) \optvarerr
  + (\nabla^3 F_1(\optvar') - \nabla^3 F_0(\optvar^*))(\optvarerr
  \otimes \optvarerr).
\end{align*}
Now, we recall the definition $\hessian = \nabla^2 F_0(\optvar^*)$, the
Hessian of the risk at the optimal point, and solve for the error
$\optvarerr$ to see that
\begin{align}
  \optvarerr & =
  -\invhessian \nabla F_1(\optvar^*)
  - \invhessian(\nabla^2 F_1(\optvar^*) - \hessian) \optvarerr
  - \invhessian \nabla^3 F_1(\optvar^*)(\optvarerr \otimes \optvarerr)
  \nonumber \\
  & \qquad ~
  + \invhessian (\nabla^3 F_0(\optvar^*) - \nabla^3 F_1(\optvar'))
  (\optvarerr \otimes \optvarerr)
  \label{eqn:third-order-taylor-expansion}
\end{align}
on the event $\event$.
As we did in the proof of Theorem~\ref{theorem:no-bootstrap}, specifically
in deriving the recursive equality~\eqref{eqn:recursive-equality}, we may
apply the expansion~\eqref{eqn:simple-error-expansion} of $\optvarerr =
\optvar_1 - \optvar^*$ to obtain a clean asymptotic expansion of
$\optvarerr$ using~\eqref{eqn:third-order-taylor-expansion}. Recall the
definition $P = \nabla^2 F_0(\optvar^*) - \nabla^2 F_1(\optvar^*)$ for
shorthand here (as in the expansion~\eqref{eqn:simple-error-expansion},
though we no longer require $Q$).

First, we claim that
\begin{equation}
  \indic{\event} (\nabla^3 F_0(\optvar^*) - \nabla^3 F_1(\optvar'))
  (\optvarerr \otimes \optvarerr)
  = \left(\lipthird^2 \lipobj^6 / \strongparam^6
  + \lipobj^4 \liphessian^2 d \log(d) / \strongparam^4\right)\remainder_3.
  \label{eqn:remainder-third-derivative}
\end{equation}
To prove the above expression, we add and subtract $\nabla^3 F_1(\optvar^*)$ (and drop $\indic{\event}$ for simplicity).
We must control
\begin{equation*}
  (\nabla^3 F_0(\optvar^*) - \nabla^3 F_1(\optvar^*))
  (\optvarerr \otimes \optvarerr)
  + (\nabla^3 F_1(\optvar^*) - \nabla^3 F_1(\optvar'))
  (\optvarerr \otimes \optvarerr).
\end{equation*}
To begin, recall that $\matrixnorm{u \otimes v}_2 = \matrixnorm{u v^\top}_2
= \ltwo{u} \ltwo{v}$.  By Assumption~\ref{assumption:strong-smoothness}, on
the event $\event$ we have that $\nabla^3 F_1$ is $(1/\numobs) \sum_{i=1}^n
\lipthird(\statrv_i)$-Lipschitz, so defining $\lipthird_\numobs =
(1/\numobs) \sum_{i=1}^n \lipthird(\statrv_i)$, we have
\begin{align*}
  \E\left[\indic{\event}\ltwo{\left(\nabla^3 F_1(\optvector^*) -
      \nabla^3 F_1(\optvector')\right)
      (\optvarerr \otimes \optvarerr)}^2\right]
  & \le \E\left[
    \lipthird_\numobs^2
    \ltwo{\optvector^* - \optvector'}^2
    \ltwo{\optvarerr}^4\right] \\
  & \le \E\left[\lipthird_\numobs^8\right]^{1/4}
  \E\left[\ltwo{\optvector_1 - \optvector^*}^8\right]^{3/4}
  \le \order(1) \lipthird^2 \frac{\lipobj^6}{\strongparam^6 \numobs^3}
\end{align*}
by H\"older's inequality and Lemma~\ref{lemma:higher-order-moments}. The
remaining term we must control is the derivative difference
$\E[\ltwos{(\nabla^3 F_1(\optvar^*) - \nabla^3 F_0(\optvar^*))( \optvarerr
    \otimes \optvarerr)}^2]$. Define the random vector-valued function
$\graddiff = \nabla (F_1 - F_0)$, and let $\graddiff_j$ denote its $j$th
coordinate. Then by definition we have
\begin{align*}
  (\nabla^3 F_1(\optvar^*) - \nabla^3 F_0(\optvar^*))(
  \optvarerr \otimes \optvarerr) =
  \left[\optvarerr^\top (\nabla^2 \graddiff_1(\optvar^*)) \optvarerr ~~
    \cdots ~~
    \optvarerr^\top (\nabla^2 \graddiff_d(\optvar^*)) \optvarerr\right]^\top
  \in \R^d.
\end{align*}
Therefore, by the Cauchy-Schwarz inequality and the fact that
$x^\top A x \le \matrixnorm{A}_2 \ltwo{x}^2$,
\begin{align*}
  \E\left[\ltwo{(\nabla^3 F_1(\optvar^*) - \nabla^3 F_0(\optvar^*))(
      \optvarerr \otimes \optvarerr)}^2\right]
  & = \sum_{j = 1}^d \E\left[\left(\optvarerr^\top
    (\nabla^2 \graddiff_j(\optvar^*))
    \optvarerr\right)^2\right] \\
  & \leq \sum_{j=1}^d
  \left(\E\left[\ltwo{\optvarerr}^8\right]
  \E\left[\matrixnorm{\nabla^2 \graddiff_j(\optvar^*)}_2^4\right]
  \right)^{1/2}.
\end{align*}
Applying Lemma~\ref{lemma:higher-order-moments} yields that
$\E[\ltwo{\optvarerr}^8] = \order(\lipobj^8 / (\strongparam^2 n)^4)$.
Introducing the shorthand notation \mbox{$\littlegraddiff(\cdot;
  \statsample) \defeq \nabla f(\cdot; \statsample) - \nabla
  \popfun(\cdot)$,} we can write
\begin{equation*}
  \nabla^2 \graddiff_j(\optvar^*)
  = \frac{1}{\numobs }\sum_{i=1}^\numobs \nabla^2
  \littlegraddiff_j(\optvar^*; \statrv_i)
\end{equation*}
For every coordinate $j$, the random matrices $\nabla^2
\littlegraddiff_j(\optvar^*;X_i)$ $(i=1,\ldots,n)$ are i.i.d.\ and
mean zero. By Assumption~\ref{assumption:smoothness}, we have
$\matrixnorm{\nabla^2 \littlegraddiff_j(\optvar^*; \statrv_i)}_2 \le
2\liphessian(\statrv_i)$, whence we have
\mbox{$\E[\matrixnorm{\nabla^2 \littlegraddiff_j(\optvar^*;
      \statrv_i)}_2^8] \le 2^8 \liphessian^8$.}  Applying
Lemma~\ref{lemma:matrix-rosenthal}, we obtain
\begin{equation*}
  \E\left[\matrixnorm{\nabla^2 \graddiff_j(\optvar^*)}_2^4\right]
  \le \order(1) \liphessian^4 \numobs^{-2} \log^2(d),
\end{equation*}
and hence
\begin{equation*}
  \E\left[\ltwo{(\nabla^3 F_1(\optvar^*) - \nabla^3 F_0(\optvar^*))(
      \optvarerr \otimes \optvarerr)}^2\right]
  \le \order(1)
  \frac{\lipobj^4 \liphessian^2}{\strongparam^4} d \log(d) \numobs^{-3},
\end{equation*}
which implies the desired
result~\eqref{eqn:remainder-third-derivative}. From now on, terms of
the form $\remainder_3$ will have no larger constants than those in
the equality~\eqref{eqn:remainder-third-derivative}, so we ignore
them.

Now we claim that
\begin{equation}
  \label{eqn:outer-product-third-derivative-remainder}
  \indic{\event} \nabla^3 F_1(\optvar^*)(\optvarerr \otimes \optvarerr)
  = \nabla^3 F_1(\optvar^*)((\invhessian \nabla F_1(\optvar^*))
  \otimes (\invhessian \nabla F_1(\optvar^*)))
  + \remainder_3.
\end{equation}
Indeed, applying the expansion~\eqref{eqn:simple-error-expansion}
to the difference $\optvarerr = \optvar_1 - \optvar^*$, we have
on $\event$ that
\begin{align*}
  \optvarerr \otimes \optvarerr
  & = (\invhessian \nabla F_1(\optvar^*)) \otimes
  (\invhessian \nabla F_1(\optvar^*))
  + (\invhessian P \optvarerr) \otimes (\invhessian P \optvarerr)
  \\
  & \qquad ~
  - (\invhessian P \optvarerr) \otimes
  (\invhessian \nabla F_1(\optvar^*))
  - (\invhessian \nabla F_1(\optvar^*)) \otimes (\invhessian
  P \optvarerr).
\end{align*}
We can bound each of the second three outer products in the equality above
similarly; we focus on the last for simplicity. Applying the Cauchy-Schwarz
inequality, we have
\begin{align*}
  \E\left[\matrixnorm{(\invhessian \nabla F_1(\optvar^*)) \otimes
      (\invhessian P \optvarerr)}_2^2\right]
  & \le \left(\E\left[\ltwo{\invhessian \nabla F_1(\optvar^*)}^4\right]
  \E\left[\ltwo{\invhessian P (\optvar_1 - \optvar^*)}^4\right]
  \right)^{\half}.
\end{align*}
From Lemmas~\ref{lemma:higher-order-moments}
and~\ref{lemma:higher-order-matrix-moments}, we obtain that
\begin{equation*}
  \E\left[\ltwo{\invhessian \nabla F_1(\optvar^*)}^4\right]
  = \order(n^{-2})
  ~~~ \mbox{and} ~~~
  \E\left[\ltwo{\invhessian P (\optvar_1 - \optvar^*)}^4\right]
  = \order(n^{-4})
\end{equation*}
after an additional application of Cauchy-Schwarz for the second
expectation.
This shows that
\begin{equation*}
  (\invhessian \nabla F_1(\optvar^*)) \otimes (\invhessian P
  \optvarerr) = \remainder_3,
\end{equation*}
and a similar proof applies to the other three terms in the outer product
$\optvarerr \otimes \optvarerr$.  Using the linearity of $\nabla^3
F_1(\optvar^*)$, we see that to prove the
equality~\eqref{eqn:outer-product-third-derivative-remainder},
all that is required is that
\begin{equation}
  \indic{\event^c} \nabla^3 F_1(\optvar^*)
  \left((\invhessian \nabla F_1(\optvar^*)) \otimes
  (\invhessian \nabla F_1(\optvar^*))\right)
  = \remainder_3.
  \label{eqn:deal-with-off-event}
\end{equation}
For this, we apply H\"older's inequality several times. Indeed, we have
\begin{align*}
  \lefteqn{\E\left[\ltwo{\indic{\event^c} \nabla^3 F_1(\optvar^*)
        \left((\invhessian \nabla F_1(\optvar^*)) \otimes
        (\invhessian \nabla F_1(\optvar^*))\right)}^2\right]} \\
  & \le
  \E[\indic{\event^c}]^{1/4}
  \E\left[\ltwo{\nabla^3 F_1(\optvar^*)
      \left((\invhessian \nabla F_1(\optvar^*)) \otimes
      (\invhessian \nabla F_1(\optvar^*))\right)}^{8/3}\right]^{3/4} \\
  & \le \E[\indic{\event^c}]^{1/4}
  \E\left[\matrixnorm{\nabla^3 F_1(\optvar^*)}^{8/3}
    \ltwo{\invhessian \nabla F_1(\optvar^*)}^{16/3}\right]^{3/4} \\
  & \le \E[\indic{\event^c}]^{1/4}
  \E\left[\matrixnorm{\nabla^3 F_1(\optvar^*)}^8\right]^{1/4}
  \E\left[\ltwo{\invhessian \nabla F_1(\optvar^*)}^8\right]^{2/4}
  = \order(\numobs^{-1} \cdot \liphessian^2 \cdot \numobs^{-2}).
\end{align*}
For the final asymptotic bound, we used
equation~\eqref{eqn:no-bad-events} to bound $\E[\indic{\event^c}]$,
used the fact (from Assumption~\ref{assumption:smoothness}) that
$\E[\liphessian(\statrv)^8] \le \liphessian^8$ to bound the term
involving $\nabla^3 F_1(\optvar^*)$, and applied
Lemma~\ref{lemma:moment-consequences} to control
$\E[\ltwos{\invhessian \nabla F_1(\optvar^*)}^8]$.  Thus the
equality~\eqref{eqn:deal-with-off-event} holds, and this completes the
proof of the
equality~\eqref{eqn:outer-product-third-derivative-remainder}.

For the final step in the lemma, we claim that
\begin{equation}
  \label{eqn:recursive-third-remainder}
  -\indic{\event} \invhessian(\nabla^2 F_1(\optvar^*) - \hessian) \optvarerr
  = \invhessian(\nabla^2 F_1(\optvar^*) - \hessian)
  \invhessian \nabla F_1(\optvar^*) + \remainder_3.
\end{equation}
To prove~\eqref{eqn:recursive-third-remainder} requires an argument
completely parallel to that for our
claim~\eqref{eqn:outer-product-third-derivative-remainder}. As before,
we use the expansion~\eqref{eqn:simple-error-expansion} of the difference
$\optvarerr$ to obtain that on $\event$,
\begin{align*}
  \lefteqn{-\invhessian(\nabla^2 F_1(\optvar^*) - \hessian) \optvarerr} \\
  & = \invhessian(\nabla^2 F_1(\optvar^*) - \hessian) \invhessian
  \nabla F_1(\optvar^*)
  - \invhessian(\nabla^2 F_1(\optvar^*) - \hessian)\invhessian P
  \optvarerr.
\end{align*}
Now apply Lemmas~\ref{lemma:higher-order-moments}
and~\ref{lemma:higher-order-matrix-moments} to the final term after a few
applications of H\"older's inequality.  To finish the
equality~\eqref{eqn:recursive-third-remainder}, we argue that
$\indic{\event^c} \invhessian (\nabla^2 F_1(\optvar^*) - \hessian)
\invhessian \nabla F_1(\optvar^*) = \remainder_3$, which follows exactly the
line of reasoning used to prove the
remainder~\eqref{eqn:deal-with-off-event}.

Applying equalities~\eqref{eqn:remainder-third-derivative},
\eqref{eqn:outer-product-third-derivative-remainder},
and~\eqref{eqn:recursive-third-remainder} to our earlier
expansion~\eqref{eqn:third-order-taylor-expansion} yields that
\begin{align*}
  \optvarerr & =
  \indic{\event}\big[
    -\invhessian \nabla F_1(\optvar^*)
    - \invhessian(\nabla^2 F_1(\optvar^*) - \hessian) \optvarerr
    - \invhessian \nabla^3 F_1(\optvar^*)(\optvarerr \otimes \optvarerr)
    \\
    & \qquad\quad ~
    + \invhessian (\nabla^3 F_0(\optvar^*) - \nabla^3 F_1(\optvar'))
    (\optvarerr \otimes \optvarerr)\big]
  + \indic{\event^c} \optvarerr
  \\
  & = -\invhessian \nabla F_1(\optvar^*)
  + \invhessian(\nabla^2 F_1(\optvar^*) - \hessian)
  \invhessian \nabla F_1(\optvar^*) \\
  & \qquad ~ - \invhessian \nabla^3 F_1(\optvar^*)\left(
  (\invhessian \nabla F_1(\optvar^*)) \otimes
  (\invhessian \nabla F_1(\optvar^*)) \right)
  + \remainder_3
  + \indic{\event^c} \optvarerr.
\end{align*}
Finally, the bound~\eqref{eqn:no-bad-events} implies that
$\E[\indic{\event^c}\ltwo{\optvarerr}^2] \le \P(\event^c) \radius^2 =
\order(n^{-4})$, which yields the claim.


\ifdefined\jmlr
\else
\bibliographystyle{abbrvnat}
\fi
\bibliography{bib}

\end{document}